\documentclass[10pt,journal]{IEEEtran}
\IEEEoverridecommandlockouts
\usepackage{amsmath}
\usepackage[ruled,vlined,linesnumbered]{algorithm2e}
\usepackage{cite}
\usepackage{etoolbox}
\usepackage{soul}

\ifx\pdfoutput\undefined
\usepackage{graphicx}
\else
\usepackage[pdftex]{graphicx}
\usepackage{epstopdf}
\fi

\usepackage{mathrsfs}
\usepackage{amsfonts}
\usepackage{amssymb}
\usepackage{array}
\usepackage{lscape}
\usepackage{subfigure}
\usepackage{wrapfig}

\usepackage{color}

\pretocmd{\chapter}{\renewcommand\thesection{\Alph{section}}}{}{}

\allowdisplaybreaks[1]

\graphicspath{{thesisFigures/}}
\providecommand{\citet}[1]{\citeauthor{#1}\,[\citeyear{#1}]}
\providecommand{\citep}[1]{\cite{#1}}

\newcommand{\axx}[1]{{\color{black}{#1}}}

\newcommand{\ph}[1]{\textbf{#1:}}

\newtheorem{proposition}{Proposition}
\newtheorem{proof}{Proof}

\usepackage{mathptmx}      

\newcommand{\flag}{Nocomparison} 

\newcommand{\ifComparisonThenElse}[2]
{\ifthenelse{\equal{\flag}{comparison}}{{\color{red}#1}}{#2}}

\renewcommand{\baselinestretch}{0.97}
\renewcommand{\IEEEbibitemsep}{0pt plus 2pt}
\makeatletter
\IEEEtriggercmd{\reset@font\normalfont\footnotesize}
\makeatother
\IEEEtriggeratref{1}
\usepackage{enumitem}

\newcommand{\gapBeforeSection}{-12pt}
\newcommand{\gapBeforeCaption}{-12pt}
\newcommand{\gapBeforeFigure}{-7pt}

\title{\LARGE SLAP: Simultaneous Localization and Planning Under Uncertainty via Dynamic Replanning in Belief Space}

\author{Ali-akbar Agha-mohammadi, Saurav Agarwal, Sung-Kyun Kim, Suman Chakravorty, and Nancy M. Amato
\thanks{Agha is with NASA-JPL, Caltech. Agarwal and Chakravorty are with the Dept. of Aerospace Eng. and Amato is with the Dept. of Computer Science \& Eng. at Texas A\&M University. Kim is with the Robotics Institute at Carnegie Mellon University, {Emails: \tt\small aliagha@jpl.nasa.gov, } {\{\tt\small sauravag,schakrav,amato\}@tamu.edu, kimsk@cs.cmu.edu}}
}

\begin{document}\fontsize{9.883}{11.8596}\rm 

\maketitle

\begin{abstract}
Simultaneous localization and Planning (SLAP) is a crucial ability for an autonomous robot operating under uncertainty. In its most general form, SLAP induces a continuous POMDP (partially-observable Markov decision process), which needs to be repeatedly solved online. This paper addresses this problem and proposes a dynamic replanning scheme in belief space. The underlying POMDP, which is continuous in state, action, and observation space, is approximated offline via sampling-based methods, but operates in a replanning loop online to admit local improvements to the coarse offline policy. This construct enables the proposed method to combat changing environments and large localization errors, even when the change alters the homotopy class of the optimal trajectory. It further outperforms the state-of-the-art FIRM (Feedback-based Information RoadMap) method by eliminating unnecessary stabilization steps. Applying belief space planning to physical systems brings with it a plethora of challenges. A key focus of this paper is to implement the proposed planner on a physical robot and show the SLAP solution performance under uncertainty, in changing environments and in the presence of large disturbances,	 such as a kidnapped robot situation.

\begin{IEEEkeywords}
Motion planning, belief space, robust, POMDP, uncertainty, mobile robots, rollout 
\end{IEEEkeywords}

\end{abstract}

\section{Introduction} \label{sec:intro}
\axx{Simultaneous Localization and Planning (SLAP) refers to the problem of (re)planning dynamically every time the localization module updates the probability distribution on the robot's state. For autonomous navigation, solving SLAP and enabling a robot to perform online (re)planning under uncertainty is a key step towards reliable operation in changing real-world environments with uncertainties.} For example, consider a low-cost mobile robot operating in an office-like environment with a changing obstacle map (e.g., office doors switch state between open and closed), and responding to changing goals (tasks) assigned online based on user requests. Such changes in the obstacle map or in the goal location often call for \textit{global replanning} to accommodate changes in the homotopy class of the optimal solution. What makes the problem more challenging is that such replanning has to happen online and fast in partially observable environments with motion and sensing uncertainties.

\begin{figure}[ht]
	\centering
	\subfigure[\axx{Belief tree: forward construction} ] {\includegraphics[width=0.76\linewidth]{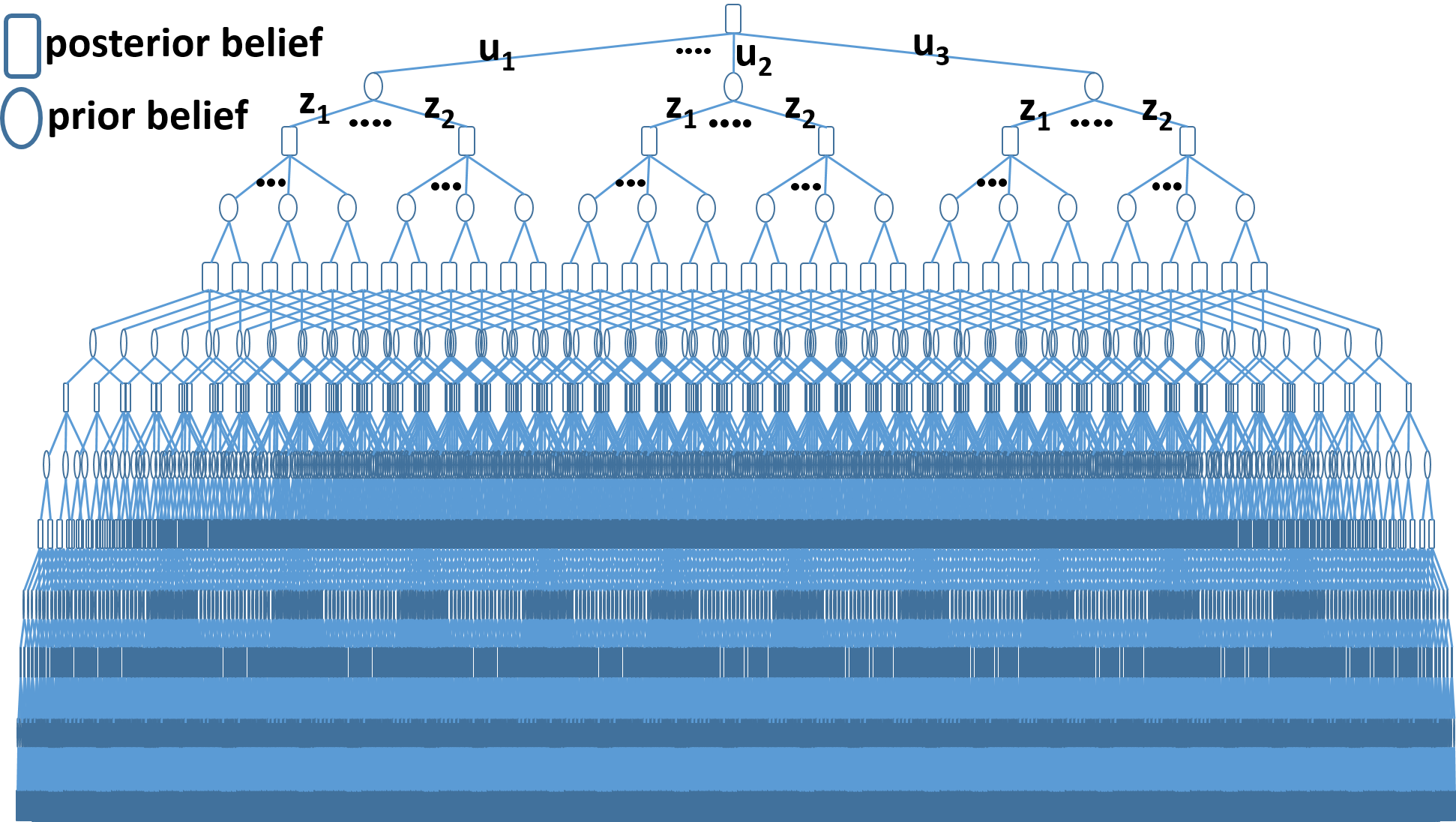}
	\label{fig:curseOfHistory}}
	\subfigure[\axx{Belief graph: backward construction} ] {\includegraphics[width=0.67\linewidth]{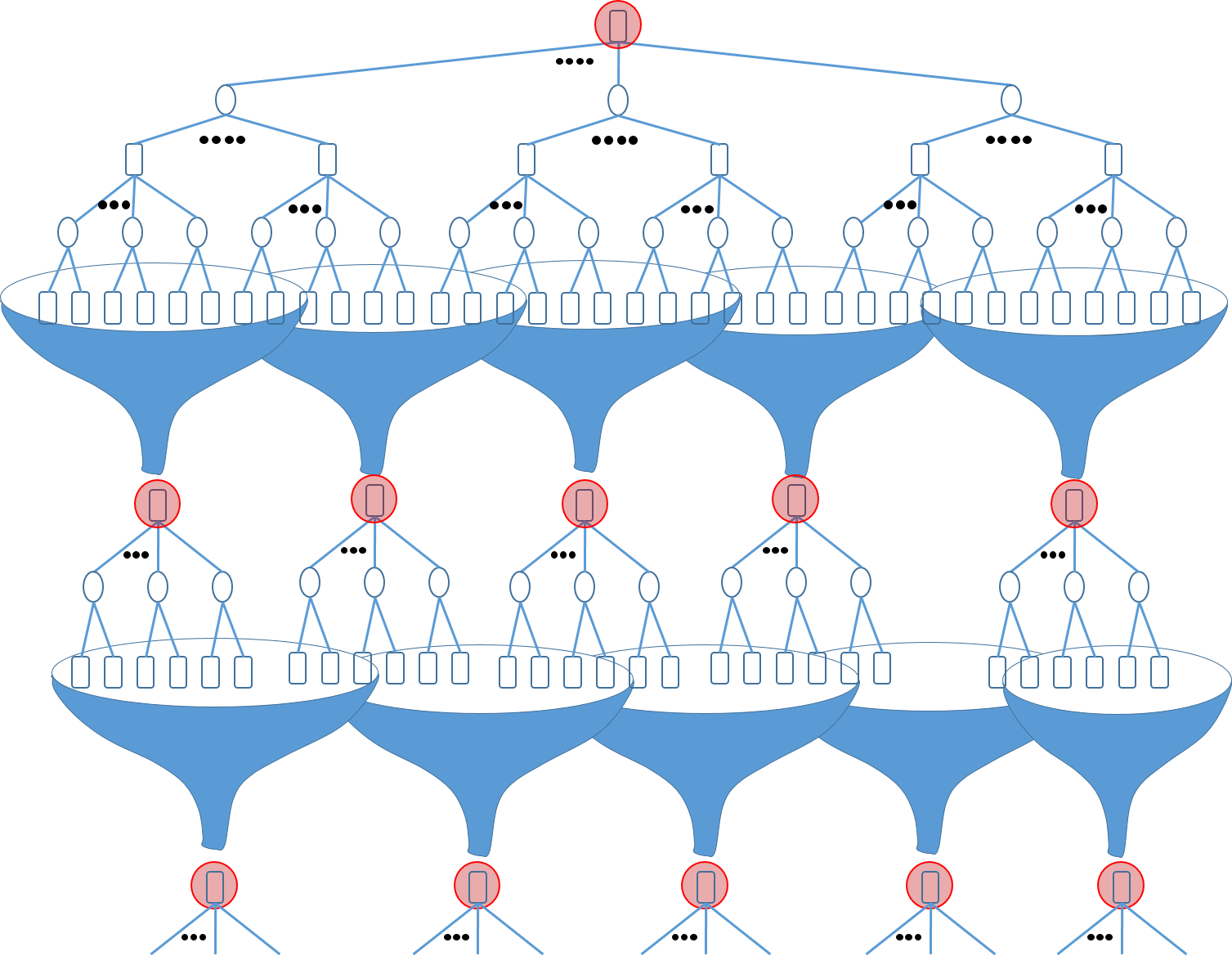} \label{fig:tree-with-funnels}}
	\caption{\axx{(a) This figure depicts a typical search tree in belief space, corresponding to a very small problem with 3 actions $ \mathbb{U}=\{u_{1},u_{2},u_{3}\} $ and two observations $ \mathbb{Z}=\{z_{1},z_{2}\} $. Each posterior belief (probability distribution) branches into $ |\mathbb{U}| $ number of priors and each prior belief branches into $ |\mathbb{Z}| $ posteriors, and thus the tree grows exponentially in the the search depth (referred to as the \textit{curse of history}). (b) This figure depicts the idea of using funnels (local feedback controllers) in belief space that can break this exponential growth by funneling a large set of posteriors into a pre-computed beliefs (in red circles). Thus a graph is formed in belief space with funnels as edges and the pre-computed beliefs as nodes. This graph grows linearly with the search depth.}}
\end{figure}

In general, decision making and control under uncertainty are ubiquitous challenges in many robotic applications. For an autonomous robot to operate reliably, it needs to perceive sensory measurements, infer its situation (state) in the environment, plan, and take actions accordingly. In partially-observable environments, where the state of the system cannot be determined exactly (due to imperfect and noisy measurements), a filtering module (e.g., Kalman filter) can provide an estimate of the state, i.e., a probability distribution function (pdf) over all possible system states. This pdf describing the localization uncertainty is referred to as the \textit{belief} or \textit{information-state}. \axx{At each time-step, actions are chosen based on the robot's belief}. To find the optimal policy that maps beliefs to actions, we cast the problem in its principled form as a Partially Observable Markov Decision Process (POMDP) problem \cite{Smallwood73,Kaelbling98}.

There are a number of challenges in dealing with POMDPs:
\begin{itemize}
	\item \textit{curse of dimensionality} \cite{Pineau03}, which refers to the high dimensions of the belief space. If the underlying robotic system evolves in a discrete grid world with $ n $ cells, the corresponding belief space is an $ n $-dimensional continuous space.
	\item \textit{curse of history} \cite{Pineau03,Ali13-IJRR}, which refers to the exponential growth of the number of future outcomes in the search depth (see Fig. \ref{fig:curseOfHistory}).
\end{itemize}
Methods such as \cite{Pineau03, Ali13-IJRR, Smith05-HSVI, Spaan05, Kurniawati08-SARSOP, kurniawati11-Migs, Grady2013AMA, littlefield2015importance, bai2015intention, patil2015scaling,silver2010monte} alleviate these issues and take POMDPs to more challenging and realistic domains. In this paper, we consider a class of POMDPs that commonly arise in modeling the SLAP problem. \axx{The settings are similar to the ones used in KF-based localization literature \cite{Thrun2005,dissanayake2001solution}, such as} \textit{(i)} the system model is given as differentiable nonlinear equations, \textit{(ii)} the state/action/observation spaces are continuous, and \textit{(iii)} belief is unimodal and well-approximated by a Gaussian.

In addition to the above-mentioned challenges, when dealing with physical systems, POMDPs need to cope with discrepancies between real models and the models used for computation, e.g., discrepancies due to changes in the environment map or due to intermittent sensor/actuator failures. Dynamic replanning under uncertainty is a plausible solution to compensate for deviations caused by such discrepancies.

To enable an online replanning scheme for SLAP, we rely on multi-query methods in belief space and specifically the Feedback-based Information Roadmap (FIRM) method, discussed below. The main body of POMDP literature (sampling-based methods in particular) propose single-query solvers, i.e., the computed solution depends on the system's initial belief \cite{Prentice09,Berg11-IJRR, kurniawati11-Migs}. Therefore, in replanning (from a new initial belief) almost all the computations need to be reproduced, which limits their usage in solving SLAP where dynamic replanning is required. However, multi-query methods such as FIRM provide a construction mechanism, independent of the initial belief of the system (Fig. \ref{fig:tree-with-funnels} and \ref{fig:funnel-FIRM}), making them suitable methods for SLAP. The original FIRM framework provides a reliable methodology for solving the problem of motion planning under uncertainty by reducing the intractable dynamic programming (DP) to a tractable DP over the nodes of the FIRM graph. In this paper, we extend our previous work on FIRM by proposing a dynamic replanning scheme in belief space that enables online replanning for real world applications in mobile robotics. This extension leads to intelligent robot behaviors that provably takes the solution closer to the optimal solution and guarantees that success probability only increases. 

In addition to the proposed algorithms, an emphasis of this paper is on the implementation of the proposed SLAP solution on a physical robot. We investigate the performance of the proposed method and demonstrate its ability to cope with model discrepancies, such as changes in the environment and large deviations which can globally change the plan by changing the homotopy class of the optimal solution. We believe these results lay the ground work for advancing the theoretical POMDP framework towards practical SLAP applications and achieving long-term robotic autonomy. 

\vspace{-8pt}
\subsection{Related Work}\label{subsec:RelatedWork}
Online replanning in belief space is an important capability to solve the SLAP problem for two main reasons. First, belief dynamics are usually more random than state dynamics because the belief is directly affected by the measurement noise. Therefore, a noisy measurement or spurious data association can cause large changes in the belief. Second, in practical applications, discrepancies between real and computational models or changes in the environment can cause the belief to occasionally show off-nominal behaviors. Online replanning while localizing equips the system with the ability to recover from such situations.

\ph{Active localization}
Solving the planning problem alongside localization and mapping has been the topic of several recent works (e.g, \cite{hcarrillo-icra14-optim-active-slam}, \cite{hcarrillo-icra12-active-slam}, \cite{lcarlone-2014-jirs}, \cite{kim-eustice-pdn}). The method in \cite{GBS-IJRR-2015} presents an approach to uncertainty-constrained simultaneous planning, localization, and mapping for an unknown environment. In \cite{lcarlone-icra14}, the authors propose an approach to actively explore unknown maps while imposing a hard constraint on the localization uncertainty at the end of the planning horizon. Our work assumes the environment map is known, formulates the problem in its most general form (POMDP), and focuses on online replanning in the presence of obstacles which may possibly change over time.

\ph{General-purpose POMDP solvers}
There is a strong body of literature on general purpose POMDP solvers (e.g., \cite{kurniawati2012global}, \cite{Bai10-contstatePOMDP}, \cite{chaudhari-ACC13}, \cite{smith2004heuristic}, \cite{ong2010planning}), divided into offline and online solvers. Offline solvers \cite{shani2013survey} compute a policy over the belief space (e.g., \cite{Pineau03,Spaan05,Smith05-HSVI,Kurniawati08-SARSOP}) and online solvers \cite{Ross_2008online_survey} aim to compute the best action for the current belief by creating a forward search tree rooted in the current belief. In recent years, general-purpose online solvers have become faster and more efficient (e.g., AEMS \cite{Ross_2007_AEMS}, DESPOT \cite{somani2013despot}, ABT \cite{kurniawati-isrr13}, and POMCP \cite{silver2010monte}). However, direct application of these methods to SLAP-like problems is a challenge due to \textit{(i)} expensive simulation steps, \textit{(ii)} continuous, high-dimensional spaces, and \textit{(iii)} difficult tree pruning steps. We discuss these three challenges in the following paragraphs.

Forward search methods rely on simulating the POMDP model forward in time and creating a tree of possible scenarios in future. At each simulation step $ (x',z,c)\sim\mathcal{G}(x,u) $, the simulator $ \mathcal{G} $, simulates taking action $ u $ at state $ x $ and computes the next state $ x' $, observation $ z $, and the cost $c$ and constraints of this transition. When dealing with Games (e.g., Go) or traditional POMDP problems (e.g., Rock Sample \cite{Ross_2008online_survey}), the forward simulation step and cost computation is computationally very inexpensive. However, in SLAP-like problems, computing costs are typically orders of magnitude more expensive (e.g., computing collisions between the robot and obstacles).

The second challenge in applying tree-based methods to SLAP is that the tree-based methods require the simulator to revisit the same belief many times to learn its value. However, in SLAP-like problems with continuous state/action/observation spaces, the chances of visiting the same belief is almost zero. Even in the discretized version, the number of simulation steps $ (x',z,c)\sim\mathcal{G}(x,u) $ along the tree is of the order $ n_{coll} = O(n_{b}(|\mathbb{U}||\mathbb{Z}|)^{d}) $. For the problem in this paper, typical values are in the order of: $ n_{b}>10 $ for the number of particles used to represent the belief for accurate collision checking. $ d = 10^4 $ steps for path length (search tree depth), $|\mathbb{U}| = 50^2$ and $ |\mathbb{Z}|= 100^{10}$ for discretization of the 2-dimensional and 10-dimensional control and observation spaces. Thus, the chances of revisiting the same belief in the discretized version of the problem are quite low.

Finally, unlike many traditional POMDP domains where the domain structure (e.g., game rules) prunes a lot of tree branches, pruning the search tree is much more challenging (if at all possible) in SLAP-like problems, where there is a terminal belief to achieve, no discount factor, and strong dependence between future costs and the past costs.

\begin{figure}[t!]
	\vspace{-10pt}
	\centering
	\includegraphics[width=2.8in]{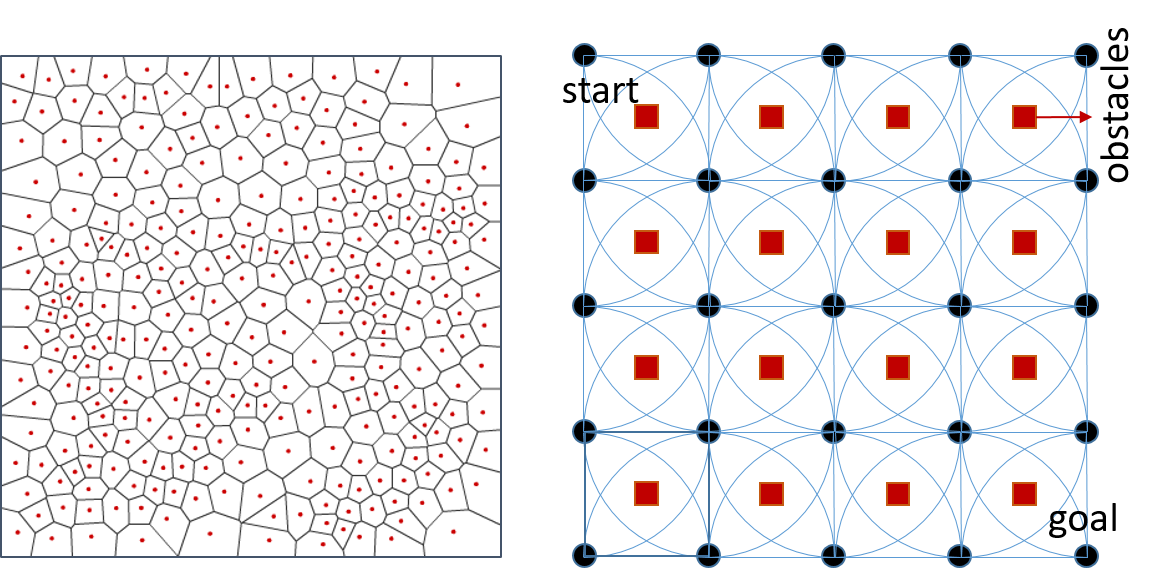}
	\caption{\axx{Example environments with many homotopy classes. Left figure is a Voronoi graph around point obstacles. Right figure is a simple lattice around obstacles. The number of paths from start to goal is in the order of $ g^d $, where $ g $ is the branching factor of the graph, and $ d $ is the search depth. In the right figure, even if we assume the robot can only go right and down directions in the lattice, the number of paths would be in the order of $ 4^8 $.}}
	\label{fig:homotoypy}
\end{figure}

To handle these challenges, different methods limit the scope of POMDPs to smaller classes, such as POMDPs with differentiable models and Gaussian noise. The following include several example of such methods.

\ph{Local optimization-based methods}
In these methods, the optimization variable is typically an open-loop sequence of actions and the optimization converges to the local optimum around the initial solution. The challenge with these methods (e.g., \cite{Platt10,Berg12AAAI}) is that they require an initial trajectory. However, finding a good initial solution could be as difficult as the original problem depending on the environment and observation model. For example, Fig. \ref{fig:homotoypy} shows environments with thousands of homotopy classes and local minima. In contrast, the proposed method in this paper does not require an initial solution, and is not sensitive to local minima. Further, typically the planning horizon in the local optimization-based methods is short since the computational complexity of the optimization grows (often super-linearly) in the planning horizon. Local optimization-based methods can be used in a Receding Horizon Control (RHC) scheme as follows: At every step a sequence of optimal actions is computed within a limited horizon of $ T $ steps. Then, only the first action is executed and the rest are discarded. The executed action takes the system to a new point, from which a new sequence of optimal actions is recomputed within horizon $ T $. This process is repeated until the system reaches the goal region. The RHC framework was originally designed for deterministic systems and its extension to stochastic systems and belief space planning is still an open problem. A direct approach is to replace the uncertain quantities (such as future observations) with their nominal values (e.g., most likely observations), and then treat the stochastic system as a deterministic one and use it in an RHC framework (e.g., \cite{Erez2010}, \cite{Platt10}, \cite{Chakrav11-IRHC}, \cite{He11JAIR}, \cite{Toit10}). Due to this approximation and the limited optimization horizon, the system may myopically choose good local actions but after a while may find itself in a high-cost region. 
	
\ph{Global Sampling-based methods}
To handle local minima, methods like \cite{Prentice09,Berg11-IJRR,Bry11} extend the traditional deterministic motion planning methods (e.g., PRM and RRT) to belief space. The main challenge is that these belief space planners are single query (the solution is valid for a given initial belief). Thus, when replanning from a new belief, most of the computation needs to be reproduced. In particular, when the planner needs to switch the plan from one homotopy class to another, replanning and finding the right homotopy class becomes challenging (Fig. \ref{fig:homotoypy}). The proposed method in this paper can inherently deal with changes in the homotopy class due to its multi-query graph structure.

\ph{Application of POMDPs to physical robots}
From an experimental point of view, a few recent works have focused on applying belief space planning to real-world robots. \cite{arne-brock-icra2014} implements a belief planner on a mobile manipulator with time of traversal as a cost metric. \cite{kaelbling2012integrated} is an integrated task and motion planner, utilizing symbolic abstraction, whose performance is demonstrated on a PR2 robot tasked with picking and placing household objects. In \cite{alterovitz-iros14}, the authors develop a stochastic motion planner and show its performance on a physical manipulation task where unknown obstacles are placed in the robot's operating space and the task-relevant objects are disturbed by external agents. \cite{bai2015intention} extends the application of POMDP methods to autonomous driving in a crowd by predicting pedestrian intentions. Authors in \cite{Marthi_RSS12_PR2} apply a POMDP-based planner to navigate a PR2 robot in an office-like environment. The paper proposes an elegant way of incorporating environment changes into the planning framework and can cope with changes in the homotopy class. The main difference with our method is that in \cite{Marthi_RSS12_PR2} the authors address the uncertainty in obstacles and assume that the robot's position in the map is perfectly known. Compared to above methods, the work in this paper extends the application of POMDPs to continuous state/action/observation spaces with very long planning horizons. It further demonstrates the real-time replanning capability in belief space with changing environment while incorporating accurate risk and collision probabilities in the planning framework.

\vspace{-10pt}
\subsection{Highlights and Contributions}\label{subsec:contributions}
This paper extends our previous work in \cite{Ali-FIRM-ICRA14}. Compared to  \cite{Ali-FIRM-ICRA14}, we discuss in detail the concept of rollout-based belief space planning, policy execution, and present extensive simulation and experimental results to demonstrate the performance improvements made by using the proposed method. We also present analyses that guarantee a lower execution cost and failure probability as compared to the nominal FIRM policy. The main highlights and contributions of this work are as follows.

\ph{Online belief planner to enable SLAP}
We propose an online planning method in belief space. The method lends itself to the class of rollout-based methods \cite{Bertsekas07} and extends them to the belief space. Compared to belief space RHC methods, this method is not limited to a horizon, does not get stuck in local minima, and does not assume deterministic future observations.

\ph{Online switching between homotopy classes}
In motion planning, a homotopy class of paths \cite{Bhattacharya11} refers to a set of paths that can be deformed into each other by continuous transformation (bending and stretching) without
passing through obstacles (Fig.~\ref{fig:homotoypy}). A unique feature of the presented method is that it is capable of updating the plan globally and online, even when the homotopy class of the optimal solution has changed. This feature allows the proposed planner to work in changing environments and cope with large deviations.

\newcommand{\ScenarioFigSize}{1.8in}
\begin{figure*}[ht!]
	\centering
	\subfigure[A simple scenario with a FIRM roadmap, robot and environment as depicted.]{\includegraphics[width=\ScenarioFigSize]{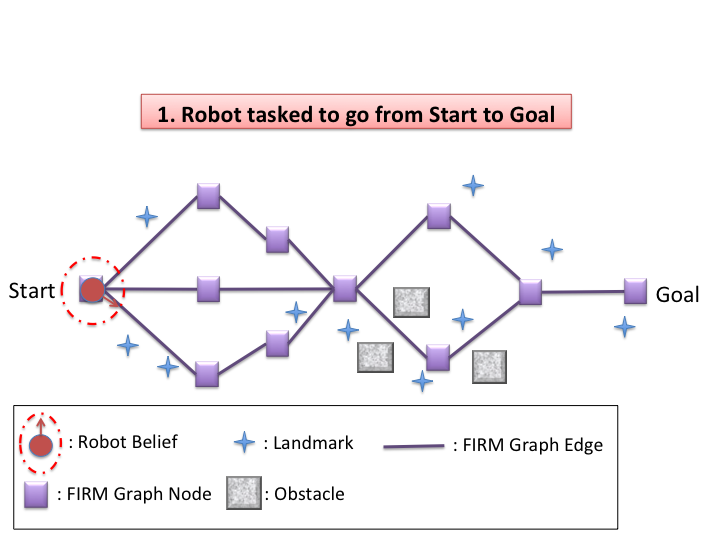}}
	\hspace{0.1in}
	\subfigure[The rollout policy is computed periodically. Four candidate edges (dashed lines), inluding the current edge, are compared.]{\includegraphics[width=\ScenarioFigSize]{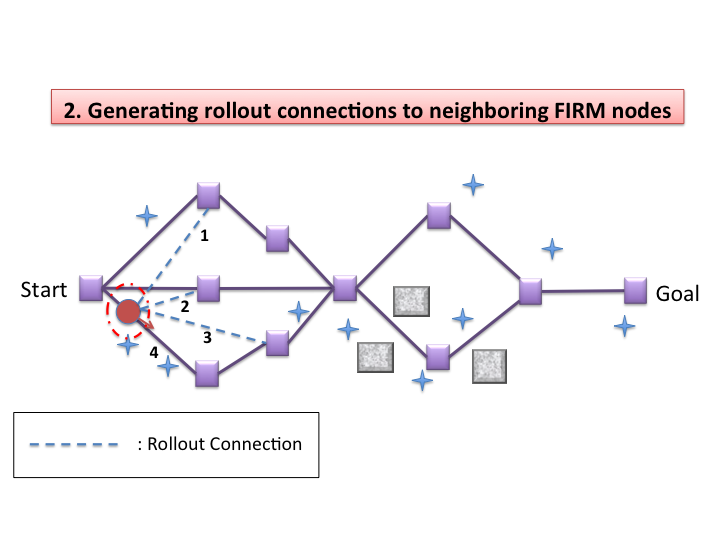}}
	\hspace{0.1in}
	\subfigure[In clutter-free regions, rollout takes a shorter route (edge 3), increasing performance and speed while loosing certainty (i.e., skipping node stabilization).]{\includegraphics[width=\ScenarioFigSize]{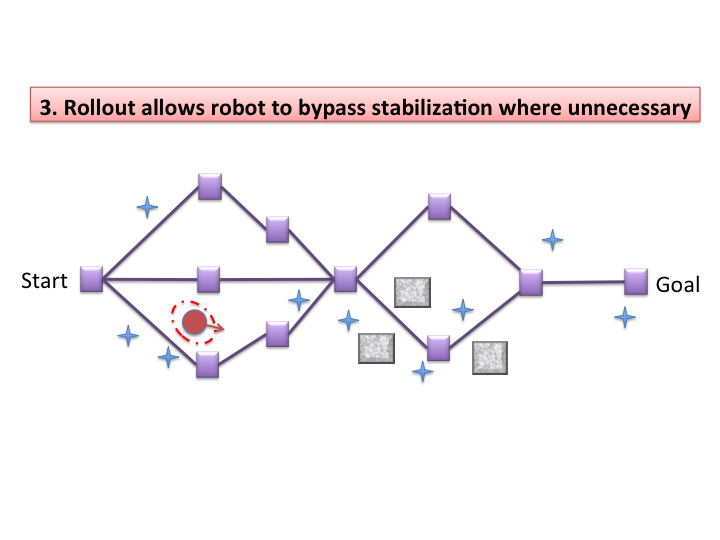}}
	\hspace{0.1in}
	\subfigure[While completing edge 3, the new rollout further cuts down task execution time by taking shorter route through a newly computed rollout edge 2.]{\includegraphics[width=\ScenarioFigSize]{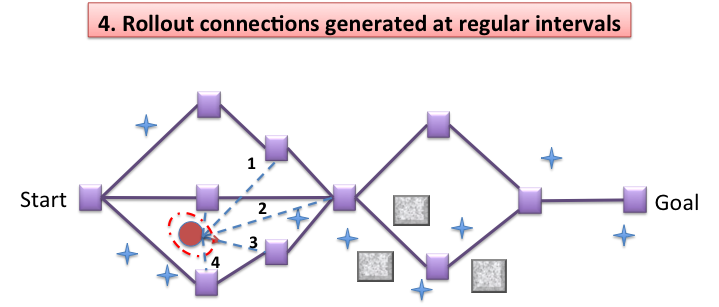}} 
	\hspace{0.1in}
	\subfigure[The robot is approaching the cluttered region. As needed the planner will slow the robot down to trade performance with certainty.]{\includegraphics[width=\ScenarioFigSize]{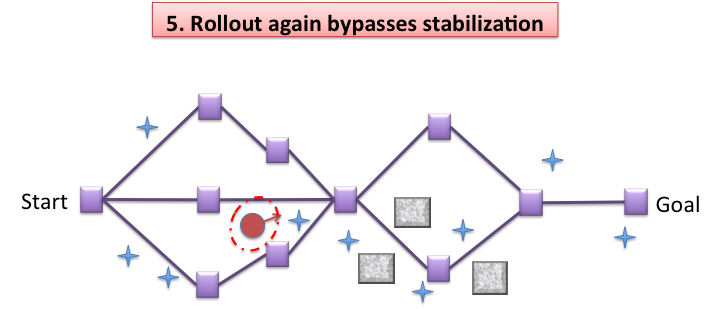}}
	\hspace{0.1in}
	\subfigure[\axx{Stabilization is performed because the reduced localization uncertainty (smaller covariance) leads to higher success probability in this case.}]{\includegraphics[width=\ScenarioFigSize]{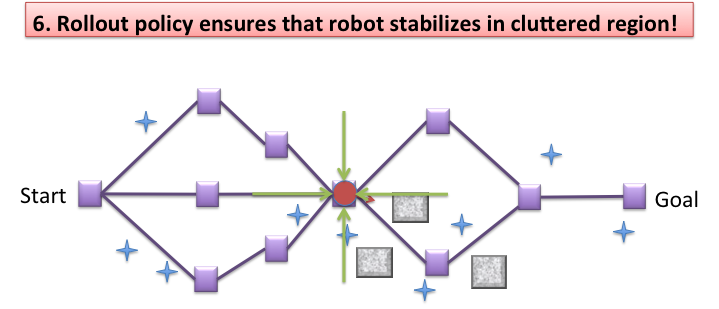}} 
	\hspace{0.1in}
	\subfigure[Stabilization occurs again as robot is still passing through the narrow passage.]{\includegraphics[width=\ScenarioFigSize]{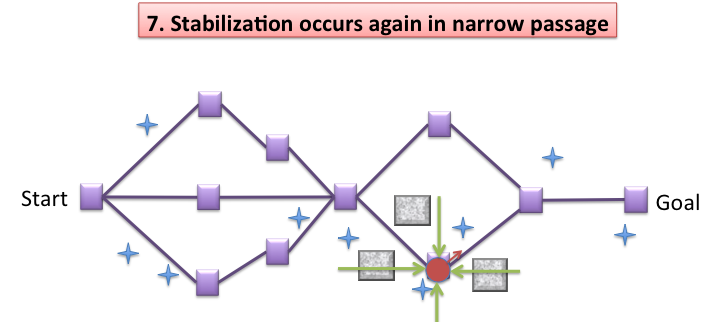}}
	\hspace{0.1in}
	\subfigure[New rollout connections allow bypassing stabilization.]{\includegraphics[width=\ScenarioFigSize]{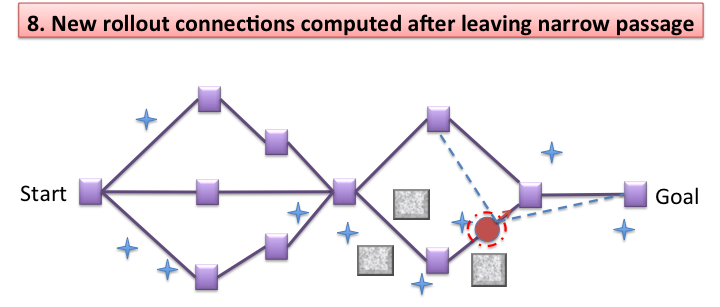}}
	\hspace{0.1in}
	\subfigure[The robot is approaching the goal.]{\includegraphics[width=\ScenarioFigSize]{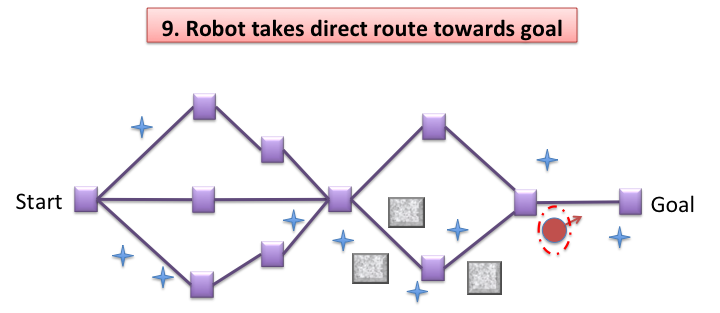}}
	\caption{A representative scenario depicting how rollout-based planner achieves higher performance compared to the standard FIRM algorithm while guaranteeing robustness. The 9 scenes depict different stages of task execution as the robot moves from the start to goal location.}
	\label{fig:rollout-cartoon}
\end{figure*}

\ph{Selective stabilization policy}
The proposed method supercedes a state-of-the-art method, FIRM \cite{Ali13-IJRR}, in performance, success probability, and ability to cope with changing environments. It builds upon a FIRM and inherits the desired features of the FIRM framework such as robustness, scalability, and the feedback nature of the solution. In addition, it significantly reduces the number of belief node stabilization (required in the original FIRM method) by eliminating the unnecessary ones. Thus the proposed method can be viewed as FIRM with a selective stabilization policy. In the original FIRM framework, at every node the system needs to steer its sensory information to reach the belief node (each graph node is a belief, i.e., a particular localization uncertainty). In this work, however, using an online local planning method, we achieve a locally optimal tradeoff between stabilization to a node (i.e., exploring the information space to reach the exact belief node) and moving forward towards the goal (exploiting the gradient of local cost function). As a result of this optimal exploration-exploitation tradeoff, interesting behaviors emerge out of the algorithm without encoding any heuristic. These behaviors (locally) optimally \textit{trade-off information and energy}. For example, consider a case when the objective is to ``reach a goal while minimizing the probability of colliding with obstacles''. In that case, in the open areas where there are no narrow passages, the system bypasses the belief node stabilizations. It speeds up and does not waste any time gathering information because reducing its uncertainty in obstacle-free regions does not have much benefit. However, once it is faced with obstacles and narrow passages, it automatically decides to perform stabilization until the uncertainty is small enough to safely traverse the narrow passage. Fig. \ref{fig:rollout-cartoon}, shows this phenomenon pictorially.

\ph{Performance guarantees}
We provide lower bounds on the performance of the method and show that in stationary environments, the performance and success probability of the proposed method always exceed (or in the worst case are equivalent to) those of the FIRM method.

\ph{Applications to physical systems} Among the set of methods that cope with continuous state/action/observation POMDP, only a very small number of methods have been applied to physical systems due to their computational complexity when dealing with real-world robotics problems. An important contribution of this work is to implement a continuous state/action/observation POMDP solver on a physical robot in a real-world office-like environment.

\section{Problem formulation} \label{sec:planning-under-incertainty}
In motion planning under uncertainty, we are looking for a policy $\pi$ that maps the available system information to an optimal action. Let $x_k\in\mathbb{X}$, $u_k\in\mathbb{U}$, and $w_k$ denote the system's state, control action, and motion noise at the $k$-th time step. Let us denote the state evolution model by $x_{k+1} = f(x_k,u_k,w_k)$. In a partially observable system, the system state is not perfectly known. Rather, the state needs to be inferred from noisy measurements. Let us denote the sensor measurement (or observation) vector by $z_k\in\mathbb{Z}$ at the $k$-th time step and the measurement model by $z_k = h(x_k,v_k)$, where $v_k$ denotes sensing noise. All the data that is available for decision making at the $ k $-th time step is the history of observations and controls: $\mathcal{H}_{k}=\{z_{0:k},u_{0:k-1}\}=\{z_{0},z_{1},\cdots,z_{k},u_{0},\cdots,u_{k-1}\} $. 

System \textit{belief} or \textit{information-state} $ b_k\in\mathbb{B} $ compresses the data history $ \mathcal{H}_{k} $ into a conditional probability distribution over all possible system states $b_{k}:=p(x_{k}|\mathcal{H}_{k}) $. In Bayesian filtering, belief can be computed recursively based on the last action and current observation $b_{k+1}=\tau(b_k,u_k,z_{k+1})$ \cite{Bertsekas07},\cite{Thrun2005}:\vspace{-3pt}
\begin{align}
\nonumber
b_{k+1}&=\alpha{p(z_{k+1}|x_{k+1})\int_{\mathbb{X}}p(x_{k+1}|x_{k},u_{k})b_{k}dx_{k}}=:\tau(b_k,u_k,z_{k+1}),
\end{align}
where, $ \alpha={p(z_{k+1}|\mathcal{H}_{k},u_{k})}^{-1} $ is the normalization constant. Once the belief is formed, a policy $ \pi_{k} $ generates the next action, i.e., $ u_{k}=\pi_{k}(b_{k}) $. The optimal policy $ \pi_{k} $ is the solution of a POMDP, which is intractable over continuous state/action/observation spaces. 

SLAP is the problem of online planning under uncertain robot poses in a known environment with changing obstacles and goal locations. SLAP entails dynamic risk assessment and planning every time the localization module updates the probability distribution on the robot's state, or every time the environment map changes. We refer to the POMDP problem induced by SLAP, as SLAP-POMDP.

\ph{SLAP-POMDP} In SLAP-POMDP, the system state is the robot pose (e.g., location). In SLAP-POMDP, the state, action, and observation spaces are continuous, and motion model $ f $ and sensor model $h$ are locally differentiable nonlinear functions. The risk is critical throughout the entire plan, and hence there is no discount factor in SLAP-POMDP, as opposed to traditional POMDP problems. Instead, there exists a termination set $B^{goal}\subset\mathbb{B}$ such that $c(b,u)=0$ for all $b\in B^{goal}$. Risk typically represents probability of violating constraints $ X^{obst}\!\!,U^{const}\!\!\! $ on state (e.g., collision with obstacles) and actions (e.g., actuator saturation).

\ph{SLAP problem} At \textit{each time-step}, re-solve the SLAP-POMDP from a new belief $ b_0 $ based on the updated constraint set $ X^{obst} $ and updated goal regions $ B^{goal} $.\vspace{-7pt}
\begin{align}\label{eq:SLAP-definition}
\nonumber &~~~~~~~~~~\pi(\cdot) = \arg\min_{\Pi}\mathbb{E}\left[\sum\limits_{k=0}^{\infty}c(b_k,\pi_{k}(b_k))~|~X^{obst},B^{goal}\right]\\
\nonumber &~~s.t.~~b_{k+1}=\tau(b_k,\pi_{k}(b_{k}),z_{k}),~~~z_{k}\sim p(z_{k}|x_{k})\\
\nonumber &~~~~~~~~x_{k+1}=f(x_k,\pi_{k}(b_{k}),w_{k}),~~~w_{k}\sim p(w_{k}|x_{k},\pi_{k}(b_{k}))\\
\nonumber &~~~~~~~~x_{k}\notin X^{obst},~~~u_k\notin U^{const}\\
&~~~~~~~~c(b,\cdot)=0, \forall b\in B^{goal},~~~~c(b,\cdot)>0, \forall b\notin B^{goal}
\end{align}
Following Kalman filter-based robot localization literature (which are widely applied to physical robots), in this paper we restrict our scope to Gaussian beliefs. We also assume there is an upper bound on the magnitude of environment changes at each step. More precisely, when $ X^{obst} $ is updated, there is an upper bound on the number of affected graph edges (see Fig. \ref{fig:tree-with-funnels}).

\section{\axx{FIRM overview}}\label{subsec:FIRM-overview}
In this section, we briefly describe the abstract framework of Feedback-based Information RoadMap (FIRM) followed by a description of its concrete implementation in our experiments. We refer the reader to \cite{Ali13-IJRR,Ali11-FIRM-IROS} for more in-depth descriptions.

FIRM is a framework designed to reduce a class of intractable continuous POMDPs to tractable problems by generating a representative graph (PRM: Probabilistic Roadmap Method) in the belief space. Similar to PRMs \cite{Kavraki96} where the \textit{solution path} is a concatenation of local paths, in FIRM, the \textit{solution policy} is a concatenation of local policies. Every node in a FIRM graph is a small region $B=\{b : \|b-\grave{b}\|\leq \epsilon\}$ around a sampled belief $ \grave{b} $. We denote the $ i $-th node by $ B^{i} $ and the set of nodes by $ \mathbb{V}=\{B^{i} \} $. Each edge in a FIRM graph is a closed-loop local controller whose goal is to steer the belief into the target node of the edge. An edge controller connecting nodes $ i $ and $ j $ is denoted by $ \mu^{ij} $ and the set of edges by $ \mathbb{M}=\{\mu^{ij} \} $. \axx{An analogy for each local controller is a ``funnel in belief space". As depicted in Fig. \ref{fig:funnel-FIRM}, each funnel steers the set of beliefs to a milestone belief. Further, using the slide-funnel composition, we can create sparse graphs of funnels as shown in Fig. \ref{fig:funnel-FIRM}. A basic instantiation of the funnel in belief space is a stationary Linear Quadratic Gaussian (SLQG) controller (see Appendix C in \cite{Ali13-IJRR}). A basic instantiation of the slide in belief space is a Time-Varying Linear Quadratic Gaussian (TV-LQG) controller (see Appendix B in \cite{Ali13-IJRR}).}

\begin{figure}[ht]
	\centering
	\subfigure[\axx{Belief funnel}]{\includegraphics[width=1in]{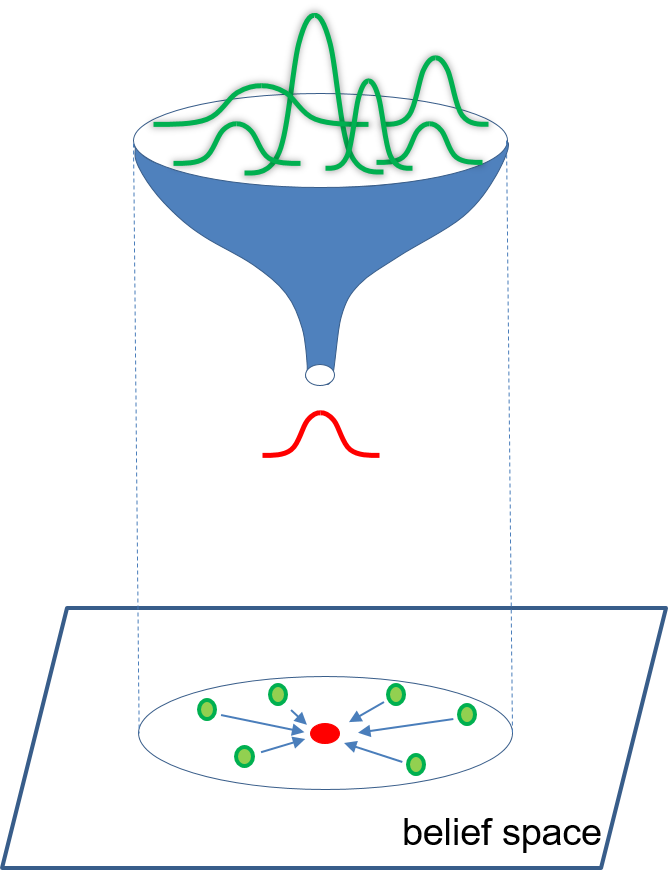}\label{fig:belief-funnel}}
	\subfigure[\axx{Funnel chain}]{\includegraphics[width=0.8in]{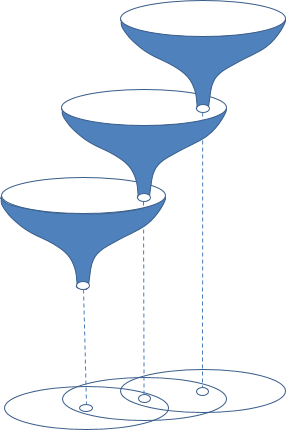}}
	\subfigure[\axx{Funnel graph}]{\includegraphics[width=1.2in]{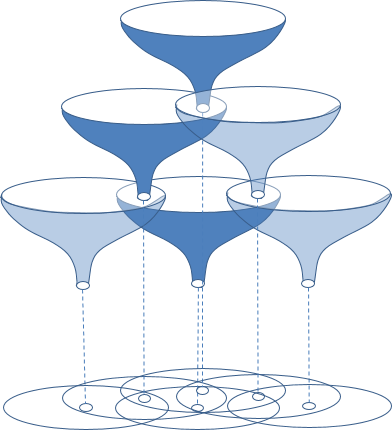}}
	\subfigure[\axx{Funnel-slide-funnel}]{\includegraphics[width=1.2in]{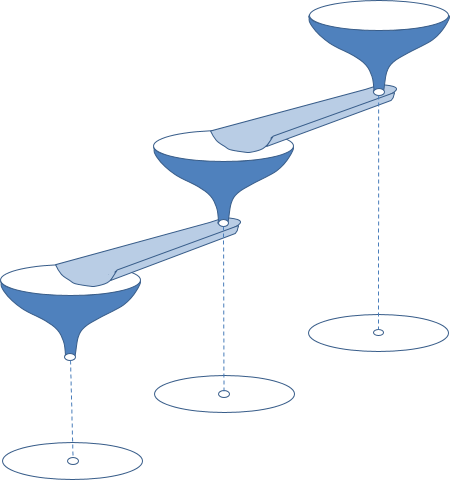}}
	\subfigure[\axx{FIRM: graph of funnel-slide-funnel}]{\includegraphics[width=1.7in]{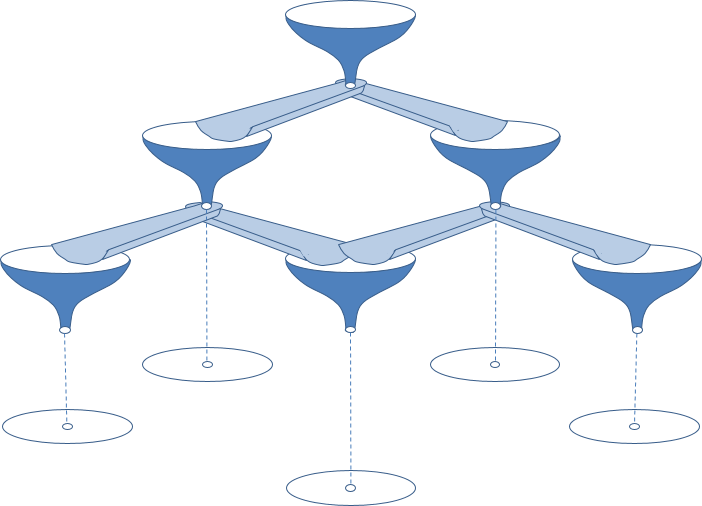}\label{fig:FIRM-slide-funnel-graph}}
	\caption{An extension of sequential composition methods \cite{burridge1999sequential} to belief space. (a) A funnel in belief space that collapses a set of Gaussian distribution to a particular Gaussian distribution, referred to as the graph node or milestone. The 2D projection denotes the belief space, where each point represents a full Gaussian distribution. The projection of the mouth of funnel is analogous to its region of attraction in belief space. (b) A chain of funnels to guide the belief towards a goal. (c) A graph of funnels, where the tip of multiple funnels can fall into the region of attraction of a single funnel. (d) For a sparse set of funnels, one can use tracking controllers (slides) to create the funnel-slide-funnel structure. (e) Graph of funnel-slide-funnel: a FIRM graph.}
	\label{fig:funnel-FIRM}
\end{figure}

Given a graph of these local controllers (Fig. \ref{fig:FIRM-slide-funnel-graph}), we can define policy $ \pi^{g} $ on the graph as a mapping from graph nodes to edges; i.e., $ \pi^{g}:\mathbb{V}\rightarrow\mathbb{M} $. $ \Pi^{g} $ denotes the set of all such graph policies. Having such a graph in belief space, we can form a tractable POMDP on the FIRM graph (so-called FIRM MDP):
\vspace{-2pt}
\begin{align}\label{eq:FIRM-MDP}
\pi^{g^{*}}=\arg\min_{\Pi^{g}} \mathbb{E}\sum_{n=0}^{\infty}C^{g}(B_{n},\pi^{g}(B_{n}))
\vspace*{-3pt}
\end{align}
where, $ B_{n} $ is the $ n $-th visited node, and $ \mu_{n} $ is the edge taken at $ B_{n} $. $ C^{g}(B,\mu):=\sum_{k=0}^{\mathcal{T}}c(b_{k},\mu(b_{k})) $ is the generalized cost of taking local controller $ \mu $ at node $ B $ centered at $ b_{0} $. 
We incorporate the failure (collision) set in planning by adding a hypothetical FIRM node $ B^{0} $ to the list of FIRM nodes. Since the FIRM MDP in Eq.~\eqref{eq:FIRM-MDP} is defined over a finite set of nodes, it can be solved by computing the cost-to-go for all graph nodes through the following dynamic programming problem:
\begin{align}\label{eq:FIRM-DP}
\!\!\!\!J^{g}(B^{i}) \!= \!\min_{\mu} \{C^{g}(B^{i},\mu)+\!\sum\limits_{\gamma=0}^{N}\mathbb{P}^{g}(B^{\gamma}|B^{i},\mu)J^{g}(B^{\gamma})\}
\end{align}
and $ \pi^{g}(B^{i}) =\mu^{*}$, where $ \mu^{*} $ is the argument of above minimization. $ \mathbb{P}^{g}(B^{\gamma}|B^{i},\mu) $ is the probability of reaching $ B^{\gamma} $ from $ B^{i} $ under $ \mu $. The failure and goal cost-to-go (i.e., $ J^{g}(B^{0}) $ and $ J^{g}(B^{goal}) $) are set to a suitably high positive value and zero, respectively.

Collision (failure) probability of FIRM starting from a given node $ B^{i} $ can be computed \cite{Ali12-ProbComp-ICRA} as:
\begin{align}\label{eq:prob-failure}
\mathbb{P}(Failure|B^i,\pi^{g}) = 1 - \Gamma_{i}^{T}(I-Q)^{-1}R_{goal},
\end{align}
where, $ \Gamma_{i} $ is a column vector of zeros with only the $ i $-th element set to one. $ Q $ is a matrix, whose $ (i,j) $-th element is $ Q[i,j]=\mathbb{P}(B^{i}|B^{j},\pi^{g}(B^{j})) $ and $ R_{goal} $ is a column vector, whose $ j $-th entry is $ R_{goal}[j]= \mathbb{P}(B^{goal}|B^{j},\pi^{g}(B^{j}))$. It can be shown that FIRM is an anytime algorithm \cite{Ali12-ProbComp-ICRA}, i.e., in a given static environment, by increasing the number of nodes, the cost (e.g., the failure probability) will go down. As will be discussed in the next section, FIRM's failure probability will be an upper bound for the failure probability of the proposed planner.

\subsection{Concrete FIRM instance in our implementation} \label{subsec:FIRM-elements}
Here we discuss the concrete realization of the FIRM graph constructed for conducting the experiments in this paper.

\ph{One-step-cost} 
Although the objective function can be general, the cost function we use in our experiments includes the localization uncertainty, control effort, and elapsed time.
\begin{align}\label{eq:one-step-cost}
c(b_{k},u_{k})=\zeta_{p}\text{tr}(P_{k})+\zeta_{u}\|u_{k}\|+\zeta_{T},
\end{align}
where, $ \mathrm{tr}(P_{k}) $ is the trace of estimation covariance as a measure of localization uncertainty. The norm of the control signal $ \|u_{k}\| $ denotes the control effort, and $\zeta_{T}$ is present in the cost to penalize each time lapse. Coefficients $ \zeta_{p} $, $ \zeta_{u} $, and $ \zeta_{T} $ are user-defined task-dependent positive scalars to combine these costs to achieve a desirable behavior. In the presence of constraints (such as obstacles in the environment), we assume the task fails if the robot violates these constraints (e.g., collides with obstacles). Therefore, collision and goal belief are terminal states such that $ J^{g}(B^{0})=J^{g}(Fail)$ is set to a suitably high cost-to-go and $ J^{g}(B^{goal})=0 $. Note that typically the one-step-cost in Eq.~\eqref{eq:one-step-cost} is defined in the state space (i.e., cost of taking action $ u $ at state $ x $). While our cost can be written as a state space cost, writing it directly in belief space better demonstrates the “active” localization aspect of the work (in the sense of minimizing the uncertainty in the localization belief) along the plan.

\ph{Steering localization covariance} To construct a FIRM graph, we first need to sample a set of stabilizers (belief steering functions). Each stabilizer is a closed-loop controller, whose role is to drive the localization uncertainty (or belief) to a FIRM node. A stabilizer consists of a filter and a separated controller \cite{Kumar-book-86}. The filter governs the belief evolution and the separated-controller generates control signals based on the available belief at each time step \cite{Kumar-book-86}. To design these steering laws, we first sample a set of points $ \mathcal{V}=\{\mathbf{v}^{i}\} $ in the robot's state space. Then, for with each point, we construct a stabilizer (i.e., funnel) \cite{Ali13-IJRR}. In the vicinity of each node $ \mathbf{v}^{j} $, we rely on the Stationary Kalman Filter (SKF) as the steering filter (which is constructed by linearizing the system about the target point $ \mathbf{v}^{j} $). It can be shown that for an observable system, the covariance under the $ j $-th SKF approaches to covariance $ P^{+^{j}}_{s} $, which can be efficiently computed by solving a corresponding Discrete Algebraic Riccati Equation \cite{Arnold84DARE}.

\ph{Steering localization mean} While steering belief toward node $ B^{j} $, separated-controller $ \mu^{ij} $ is responsible for generating the control signals based on the available belief, i.e., $ u_{k}=\mu^{ij}(b_{k}) $. The iRobot Create (used in our experiments) is a nonholonomic robot and is modeled as a unicycle (see Section \ref{subsec:robotModel}). Thus, to steer the estimation mean toward the target node $ \mathbf{v}^{j} $, one needs to use controllers designed for stabilizing nonholonomic systems (e.g., \cite{Oriolo02-DFL}, \cite{murray1993nonholonomic}, \cite{samson1991feedback}). However, the randomness of the estimation mean (resulting from randomness of observations) calls for a controller that can perform such stabilization under uncertainty. To this end, we implemented different controllers including polar coordinate-based controller \cite{deLuca2001control} and Dynamic Feedback Linearization-based controller \cite{Ali12-DFL-IROS}. Observing the behavior of different controllers, we adopted a variant of the Open-Loop Feedback Control (OLFC) scheme \cite{Bertsekas07} for stabilization purposes. In this variant of OLFC, for a given $ \mathbf{v}^{j} $, we compute an open-loop control sequence from the current estimation mean to $ \mathbf{v}^{j} $. Then, we apply a truncated sequence of the first $ l $ controls ($ l=5 $ in our experiments)\footnote{Only one control (i.e., $ l=1 $) is not enough due to the nonholonomicity of the system.}. This process repeats every $ l $ steps until we reach the vicinity of $ \mathbf{v}^{j} $.

\ph{FIRM graph} Associated with each sample $ \mathbf{v}^{j} $, we can define the belief node $\grave{b}^{j}\equiv(\mathbf{v}^{j},P^{+^{j}}_{s})$. Defining a FIRM node as a ball around this point $ B^{j}=\{b : \|b-\grave{b}^{j}\|\leq \epsilon\} $, we can steer the Gaussian localization uncertainty to this ball with combination of OLFC and SKF. Accordingly, we sample $ N $ FIRM nodes $ \{B^{j}\}_{j=1}^{N} $. 

The SKF/OLFC combination between nodes $ i $ and $ j $ forms the FIRM edge (local controller) and is denoted by $ \mu^{ij} $. We connect each node to its $ k $-nearest neighbors. The set of constructed edges is denoted by $ \mathbb{M}=\{\mu^{ij} \} $. 

Then, we compute and store costs and transition probabilities associated with each edge using offline simulations. Finally, we solve the DP in Eq. \eqref{eq:FIRM-DP} to get the optimal graph cost-to-go values $ J^{g}(B^{i}) $ and policy $ \pi^{g}(B^{i}) $ for all $ i $.

\section{SLAP via Rollout-based Dynamic Replanning in Belief Space}\label{sec:Rollout-policy-for-replanning}
As discussed in Section \ref{sec:planning-under-incertainty}, SLAP in this paper refers to the problem of (re)planning dynamically every time the localization module updates the probability distribution on the robot's state. In this section, we develop an online replanning method in belief space by extending \textit{rollout policy methods} \cite{Bertsekas07} to the stochastic partially-observable setting. In particular, we discuss replanning in the presence of changes in the environment and goal location, large deviations in the robot's location, and discrepancies between real and computational models.  We show that the proposed method increases the performance of FIRM by enabling selective stabilization.

To make the connection with the rollout policy, we re-state the POMDP problem in a more general setting of the time-varying policy.
\vspace{-7pt} 
\begin{align}\label{eq:belief-MDP-timeVarying}
\nonumber &\pi_{0:\infty}(\cdot) = \arg\min_{\Pi_{0:\infty}}\sum\limits_{k=0}^{\infty}\mathbb{E}\left[c(b_k,\pi_{k}(b_k))\right]\\
\nonumber &~s.t.~~~b_{k+1}=\tau(b_k,\pi_{k}(b_{k}),z_{k}),~~~z_{k}\sim p(z_{k}|x_{k})\\
&~~~~~~~~x_{k+1}=f(x_k,\pi_{k}(b_{k}),w_{k}),~~~w_{k}\sim p(w_{k}|x_{k},\pi_{k}(b_{k}))
\end{align}
In Eq. \eqref{eq:belief-MDP-timeVarying}, we seek a sequence of policies $\pi_{0:\infty}=\{\pi_{0}(\cdot),\pi_{1}(\cdot),\pi_{2}(\cdot),\cdots \} $, where $ \pi_{k} $ maps belief $ b_{k} $ to the optimal action $ u_{k} $. $ \Pi_{k} $ is the space of all possible policies at time step $ k $, i.e., $ \pi_{k}\in\Pi_{k} $. In the infinite horizon case, it can be shown that the solution is a stationary policy $ \pi_{s} $, i.e., $\pi_{1}=\pi_{2}=\cdots=\pi_{s} $ and the problem is reduced to Eq. \eqref{eq:SLAP-definition}. However, we keep the time-varying format for reasons that will be made clear below.

As discussed earlier, in the SLAP problem, one needs to re-solve this POMDP ``online'' every time the localization pdf is updated. To handle the computational intractability of the continuous POMDP in Eq. \eqref{eq:belief-MDP-timeVarying}, we re-use computations in an efficient way as will be explained in the next subsection. Here, we first start by discussing the general form of repeated online solutions as an RHC scheme.

\ph{RHC in belief space} Receding horizon control (also referred to as rolling horizon or model predictive control) was originally designed for deterministic systems \cite{garcia1989model} to cope with model discrepancies. For stochastic systems, where the closed-loop (feedback) control law is needed, formulation of the RHC scheme is up for debate \cite{Li02,Hessem03,Shah12,Chakrav11-IRHC}. 
In the most common form of RHC for stochastic systems \cite{Bertsekas07}, the system is approximated with a deterministic one by replacing the uncertain quantities with their typical values (e.g., maximum likelihood value.) In belief space planning, the quantity that injects randomness in belief dynamics is the observation. Thus, one can replace the random observations $ z_{k} $ with their deterministic maximum likelihood value $ z^{ml} $, where $ z_{k}^{ml}:=\arg\max_{z} p(z_{k}|x^{d}_{k}) $ in which $ x^{d} $ is the nominal deterministic value for the state that results from replacing the motion noise $ w $ by zero, i.e., $ x^{d}_{k+1}=f(x^{d}_{k},\pi_{k}(b^{d}_{k}),0) $. The deterministic belief $ b^{d} $ is then used for planning in the receding horizon window. At every time step, the RHC scheme performs a two-stage computation. At the first stage, the RHC scheme for deterministic systems solves an open-loop control problem (i.e., returns a sequence of actions $ u_{0:T} $) over a fixed finite horizon $ T $ as follows:
\vspace{-13pt}
\begin{align}\label{eq:RHC-BeliefSpace}
\nonumber &u_{0:T} = \arg\min_{\mathbb{U}_{0:T}}\sum\limits_{k=0}^{T}c(b^{d}_k,u_{k})\\
\nonumber &~~s.t.~~~~b^{d}_{k+1}=\tau(b^{d}_k,u_{k},z^{ml}_{k+1})\\
\nonumber &~~~~~~~~~~z^{ml}_{k+1}=\arg\max_{z} p(z|x^{d}_{k+1}) \\
&~~~~~~~~~~x^{d}_{k+1}=f(x^{d}_{k},u_{k},0) 
\end{align}
In the second stage, RHC executes only the first action $ u_{0} $ and discards the remaining actions in the sequence $ u_{0:T} $. However, since the actual observation is noisy and is not equal to the $ z^{ml} $, the the belief $ b_{k+1} $ will be different than $ b^{d}_{k+1} $. Subsequently, RHC performs these two stages from the new belief $ b_{k+1} $. In other words, RHC computes an open loop sequence $ u_{0:T} $ from this new belief, and this process continues until the belief reaches the desired belief location. Algorithm \ref{alg:RHC-beliefSpace} recaps this procedure. State-of-the-art methods such as \cite{platt-wafr12-RHC} and \cite{Toit10} utilize RHC in belief space. \cite{Toit10} refers to the method as partially-closed loop RHC (PCLRHC) as it exploits partial information about future observations (i.e., $ z^{ml} $) and does not ignore them.
\begin{algorithm}[h!]
\caption{RHC with most likely observations for partially-observable stochastic systems}\label{alg:RHC-beliefSpace}
\textbf{input}  :  Initial belief $ b_{current}\in\mathbb{X} $, $ B^{goal}\subset\mathbb{B} $\\
{
\While{$ b_{current}\notin B^{goal} $}{
$ u_{0:T} = $ Solve the optimization in Eq.\eqref{eq:RHC-BeliefSpace} starting from $ b^{d}_{0}=b_{current} $;\\
Apply the action $ u_{0} $ to the system;\\
Observe the actual $ z $;\\
Compute the belief $ b_{current} \leftarrow \tau(b_{current},u_{0},z) $;\\
}
}
\end{algorithm}

A known shortcoming of the stated RHC formulation is its limited horizon, which might lead the system to local minima by choosing actions that guide the robot toward ``favorable'' states (with low cost) in the near future followed by a set of ``unfavorable'' states (with a high cost) in the long run. To improve the basic RHC, different variants have been proposed including the ``rollout policy'' \cite{Bertsekas07}. Here, we discuss how they can be extended to and realized in belief space.

\ph{Rollout policy in belief space} Another class of methods that aims to reduce the complexity of the stochastic planning problem in Eq.~\eqref{eq:belief-MDP-timeVarying} is the class of rollout policies \cite{Bertsekas07}, which are more powerful than the described version of RHC in the following sense: First, they 
do not approximate the system with a deterministic one. Second, they avoid local minima using a suboptimal policy that approximates the true cost-to-go beyond the horizon. This policy is referred to as the ``base policy'' and denoted by $ \widetilde{J} $. Formally, at each step of the rollout policy scheme, the following closed-loop optimization is solved:
\begin{align}\label{eq:Rollout-BeliefSpace}
\nonumber &\pi_{0:T}(\cdot) = \arg\min_{\Pi_{0:T}}\mathbb{E}\left[\sum\limits_{k=0}^{T}c(b_k,\pi_{k}(b_k))+\widetilde{J}(b_{T+1})\right]\\
\nonumber &~~s.t.~~b_{k+1}=\tau(b_k,\pi_{k}(b_{k}),z_{k}),~~~z_{k}\sim p(z_{k}|x_{k})\\
&~~~~~~~~x_{k+1}=f(x_k,\pi_{k}(b_{k}),w_{k}),~~~w_{k}\sim p(w_{k}|x_{k},\pi_{k}(b_{k}))
\end{align}

Then, only the first control law $ \pi_{0} $ is used to generate the control signal $ u_{0} $ and the remaining policies are discarded. Similar to the RHC, after applying the first control, a new sequence of policies is computed from the new point. The rollout algorithm is described in Algorithm \ref{alg:Rollout-BeliefSpace}.

\begin{algorithm}[h!]
\caption{Rollout algorithm in belief space}\label{alg:Rollout-BeliefSpace}
\textbf{input}  :  Initial belief $ b_{current}\in\mathbb{B} $, $ B^{goal}\subset\mathbb{B} $\\
{
\While{$ b_{current}\notin B^{goal} $}{
$ \pi_{0:T} = $ Solve optimization in Eq.\eqref{eq:Rollout-BeliefSpace} starting from $ b_{0} = b_{current} $;\\
Apply the action $ u_{0} = \pi(b_{0}) $ to the system;\\
Observe the actual $ z $;\\
Compute the belief $ b_{current} \leftarrow \tau(b_{current},u_{0},z) $;\\
}
}
\end{algorithm}

Although the rollout policy in belief space efficiently reduces the computational cost compared to the original POMDP problem, it is still formidable to solve since the optimization is carried out over the policy space. Moreover, there should be a base policy that provides a reasonable cost-to-go $ \widetilde{J} $. We now proceed to propose a rollout policy in belief space that exploits the FIRM-based cost-to-go.

\subsection{Enabling SLAP via FIRM-based rollout in belief space}\label{subsec:firm-based-rollout}
In this section, we discuss how a rollout policy in belief space (and hence SLAP) can be realized using the FIRM framework. As explained earlier, in FIRM, the system transitions between two nodes\footnote{\axx{In the cartoon in Fig. \ref{fig:funnel-FIRM}, it looks like $ B^{j} $ is the sole destination for $ \mu^{ij} $. However, in dense graphs the belief under $ \mu^{ij} $ might be absorbed by a different funnel before reaching $ B^j $. The summation over $ \gamma $ in the following equations takes this consideration into account.}} $B^i$ and $B^j$ (centered at sampled beliefs $b^i$ and $b^j$) using a \textit{local} controller $\mu^{ij}$. Global-level decision-making only occurs when the system is in a FIRM node (e.g., region $B^i$) and for the rest of the time, local controls are executed (e.g., $\mu^{ij}$). In FIRM-based rollout, we raise this limitation by forcing the system to globally replan at every time step to enable SLAP. Specifically, denoting the system belief at time step $t$ by $b_t$, we rely on the following procedure. At each time step $ t $:

\begin{enumerate}[leftmargin=0cm,itemindent=.5cm,labelwidth=0.4cm,labelsep=0cm,align=left]
\item We connect $b_t$ to all its neighboring FIRM nodes (in radius $R$) using suitable local controllers $\mu^{tj}$. These local controllers are designed in the same way as FIRM edges.

\item We evaluate the transition costs $C(b_t, \mu^{tj})$ and the probability of landing in nodes $ B^{\gamma} $ under the influence of the controller $ \mu^{tj} $, i.e., $\mathbb{P}(B^{\gamma} | b_t, \mu^{tj})$.

\item We evaluate the best edge outgoing from $ b_{t} $ by solving:
\begin{align}\label{eq:rollout-minimization}
\!\!\!\!j^{*} \!= \!\arg\min_{j} \{C(b_t,\mu^{tj})+\!\sum\limits_{\gamma=0}^{N}\mathbb{P}(B^{\gamma}|b_t,\mu^{tj})J^{g}(B^{\gamma})\}
\end{align}
where, $J^{g}(B^{\gamma})$ is the nominal cost-to-go under the FIRM policy from node $B^{\gamma}$ and $J^g(B^0)$ is the failure cost-to-go as discussed in Section \ref{subsec:FIRM-overview}.

\item We choose $\mu^{tj^{*}}$ as the local controller at $ b_t $ if the expected success probability exceeds the current one. In other words, we only switch from the current local controller (i.e., $ \mu^{ij} $), to $ \mu^{tj^{*}} $ if the following condition holds:
\begin{align}\label{eq:check-success-prob}
\mathbb{E}[success|b_t,\mu^{tj^{*}}]>\mathbb{E}[success|b_t,\mu^{tj}]
\end{align}
where expected success probability is
\begin{align}\label{eq:define-expected-success-prob}
\mathbb{E}[success|b_t,\mu^{t\alpha}]=\sum_{\gamma=1}^{N}\mathbb{P}(B^{\gamma}|b_t,\mu^{t\alpha})P^{success}(B^{\gamma})
\end{align}
and $P^{success}(B^{\gamma})=\Gamma_{\gamma}^{T}(I-Q)^{-1}R_{goal}$ is the probability of success for reaching the goal from FIRM node $B^{\gamma}$ under the nominal FIRM policy (see Eq. \eqref{eq:prob-failure}).

\end{enumerate}

Algorithm \ref{alg:Rollout-FIRM} describes planning with the proposed rollout process. We split the computation to offline and online phases. In the offline phase, we carry out the expensive computation of graph edge costs and transition probabilities. Then, we handle pose deviations and the changes in the start/goal location by repeated online replanning, while reusing offline computations.

\begin{algorithm}[h!]
	\caption{\!Rollout algorithm with FIRM as base policy}\label{alg:Rollout-FIRM}
	\textbf{input}  :  Initial belief $ b_{t} $ and goal belief region $ B^{goal} $\\
	Construct a FIRM graph and store nodes $ \mathbb{V}=\{B^{i} \} $, edges $  \mathbb{M}=\{\mu^{ij}\} $, Cost-to-go $ J^{g}(\cdot) $, and Success probabilities $ P^{success}(\cdot)$;$ ~~~ $\tcp{offine phase}
	{
		\While{$ b_t\notin B^{goal} ~~~~~~~~~~~~~~~~$\tcp{online phase}}{
			Find neighboring nodes $\mathbb{V}_{R}=\{B^{i}\}_{i=1}^{r} $ to $b_{t}$;\\
			Set $ B_{t}=\{b_{t}\} $, $ J(B_{t}) = \infty $, and $ S = 0 $;\\
			\ForAll{$B^j \in \mathbb{V}_{R}$}
			{
				$\mu^{tj}$ = Generate controller from $b_t$ to $B^j$ ; \\
				$C(b_t, \mu^{tj}), \mathbb{P}(B^\gamma | b_t, \mu^{tj})$ = Simulate $\mu^{tj}$ to compute transition probability and expected cost; \\
				Compute the expected success $ \mathbb{E}[success|b_t,\mu^{tj}] $;\\
				\If{$ \mathbb{E}[success|b_t,\mu^{tj}]\geq S $}
				{
				Compute the candidate cost-to-go as $J^{cand} = C(b_t,\mu^{tj})+\!\sum_{\gamma=0}^{N}\mathbb{P}(B^{\gamma}|b_t,\mu^{tj})J^{g}(B^{\gamma})$;\\
				\If{$ J^{cand} < J(B_{t})$ \label{line:fromAlg:condition}}
				{$ J(B_{t}) = J^{cand}$ and $ S = \mathbb{E}[success|b_t,\mu^{tj}] $;\label{line:J-update}\\
					$ \mu^{tj^{*}}=\mu^{tj} $;\label{line:mu-update}\\}
			    } 
			}
			Apply the action $ u_{t} = \mu^{tj^{*}}(b_{t}) $ to the system;\\
			Get the actual measurement $ z_{t+1} $;\\
			Compute the next belief $ b_t \leftarrow \tau(b_{t},u_{t},z_{t+1}) $;\\
			\If{user submits a new goal state $ \mathbf{v}^{goal} $}
			{
				$ B^{goal} \leftarrow$ Sample the corresponding FIRM node;\\
				Add $ B^{goal}$ to the FIRM graph; $ \mathbb{V}\leftarrow\mathbb{V}\cup\{B^{goal} \} $;\\
				Connect $ B^{goal} $ to its $ r $ nearest neighbors using edges $ \{\mu^{(i,goal)} \} $. Also, $ \mathbb{M}\leftarrow\mathbb{M}\cup\{\mu^{(i,goal)} \} $;\\
				$ [J^{g}(\cdot),P^{success}(\cdot)] $ = DynamicProgramming($ \mathbb{V},\mathbb{M} $);
			}
		}
	}
\end{algorithm}

In the following, we discuss how Algorithm \ref{alg:Rollout-FIRM} provides a tractable variant of Eq. \eqref{eq:Rollout-BeliefSpace}. Following the concepts and terminology in \cite{Bertsekas07}, here, the FIRM policy plays the role of the base policy; FIRM's cost-to-go values are used to approximate the cost-to-go beyond the rollout horizon. Given a dense FIRM graph, where nodes partition the belief space (i.e., $ \cup_{i} B^{i}=\mathbb{B} $), the belief at the end of rollout ($ b_{T+1} $ in Eq. \eqref{eq:Rollout-BeliefSpace}) will fall into a FIRM node with a known cost-to-go. With a sparse FIRM graph, where nodes do not cover the entire belief space, we design local policies that drive the belief into a FIRM node at the end of horizon. However, since the belief evolution is random, reaching a FIRM node at deterministic time horizon $ T $ may not be guaranteed. Therefore, we propose a new variant of rollout method with a random horizon $ \mathcal{T} $ as follows: 
\begin{align}\label{eq:IRM-rollout}
\nonumber \pi_{0:\infty}(\cdot) 
&= \arg\min_{\widetilde{\Pi}}\mathbb{E}\left[\sum\limits_{k=0}^{\mathcal{T}}c(b_k,\pi_{k}(b_k))+\widetilde{J}(b_{\mathcal{T}+1})\right]\\
\nonumber &~~s.t.~~~~b_{k+1}=\tau(b_k,\pi_{k}(b_{k}),z_{k}),~~~z_{k}\sim p(z_{k}|x_{k})\\
\nonumber &~~~~~~~~~~x_{k+1}=f(x_k,\pi_{k}(b_{k}),w_{k}),~~~w_{k}\sim p(w_{k}|x_{k},\pi_{k}(b_{k}))\\
&~~~~~~~~~~b_{\mathcal{T}+1}\in \cup_{j}B^{j},
\end{align}
where for $ b_{\mathcal{T}+1}\in B^{j} $ we have
\begin{align}\label{eq:base-CostToGo}
\widetilde{J}(b_{\mathcal{T}+1})=J^{g}(B^{j})
\end{align}

$ \widetilde{\Pi} $ is a restricted set of policies under which the belief will reach a FIRM node in finite time. In other words, if $ \pi\in\widetilde{\Pi} $ and $ \pi=\{\pi_{1},\pi_{2},\cdots \} $, then $ \mathbb{P}(b_{\mathcal{T}+1}\in\cup_{j} B^{j}|\pi)  = 1$ for finite $ \mathcal{T} $. Thus, the last constraint in Eq. \eqref{eq:IRM-rollout} is redundant and automatically satisfied for suitably constructed $\widetilde{\Pi}$. Also, the FIRM-based cost-to-go $ J^{g}(\cdot) $ plays the role of the cost-to-go beyond the horizon (see Eq. \eqref{eq:base-CostToGo}). 

Note that based on Algorithm \ref{alg:Rollout-FIRM}, we can provide guarantees on the performance of the proposed method. Before formally stating the results, recall that at each instance of rollout computation, the current belief $b_t$ is added as a virtual node $B^{virtual}_t$ to the FIRM graph to generate the augmented FIRM graph $ G^{a}_t $. A virtual node being defined as a temporary node with no incoming edges. Virtual nodes are removed from the graph as soon as the system departs their vicinity.

\begin{proposition}
For a given static map, the performance and success probability of the FIRM-Rollout policy is lower bounded by the nominal FIRM policy at any belief state during execution of the planner.
\end{proposition}

\begin{proof}
As discussed, to compute the rollout at time $t$, belief $b_t$ is added to the FIRM graph as a virtual node, with no incoming edges. Therefore, the dynamic programming solution remains unchanged. The optimal cost-to-go from the virtual node $B^{virtual}_t$ (centered at $ b_t $) is given by:
\begin{align}
\nonumber J(B^{virtual}_{t}) = \min_{j} \{C(b_t,\mu^{tj})+\!\sum\limits_{\gamma=0}^{N}\mathbb{P}(B^{\gamma}|b_t,\mu^{tj})J^{g}(B^{\gamma})\}
\end{align}
Since the current FIRM edge is one of edges over which the above minimization is carried out, the cost-to-go (performance) with rollout is strictly upper (lower) bounded by the nominal FIRM policy cost (performance). Furthermore, due to the check in Eq. \eqref{eq:check-success-prob}, it can be further assured that the probability of success of the rollout policy is strictly greater than that of the FIRM policy \axx{in static environments}. 

Once the rollout is computed and the target node is chosen, the robot starts executing the controller $\mu^{tj^*}$ and leaves the vicinity of node $B^t$. The virtual node $B^t$ then gets removed from the graph. Further, it should be noted that as the robot moves on the virtual edge (edge from node $B_{virtual}^t$ to $B^{j^{*}}$), the rollout process is repeated which leads the robot to skip the belief stabilization as needed. Consequently, as the robot moves, due to rollout, it chooses actions which are never worse-off than the nominal FIRM policy. We refer the reader to Fig.~\ref{fig:rollout-cartoon} for a visual explanation of the process.
$ \blacksquare $ \end{proof}

\textit{Remark:} If the desired factor was merely the success probability, one can ignore the cost-to-go comparison condition in Algorithm \ref{alg:Rollout-FIRM} and only maximize the success probability.

In addition to improving the performance while not compromising on the safety, the rollout procedure is particularly helpful in handling the changes in the environment map. We discuss this aspect in the Section \ref{sec:SLAPinChange}.

\subsection{Complexity Analysis}\label{sec:complexity}
In this section, we analyze the computational complexity of the offline and online phase of the proposed algorithm.

\ph{Offline phase}
We assume the environment is a hypercube $ [0,w]^{d} $. For constructing the offline policy on a $ k $-regular graph with $ N $ nodes, we have to simulate $ kN $ edges offline. Let us denote the number of particles describing belief by $ n^{off}_b $. Assuming a fixed velocity  1 m/s on edges, and assuming simulations steps occur at every $ \Delta t $ seconds, the number of simulation calls (including collision checks) is $ n_{coll} = \sum_{s = 1}^{kN} n_b^{off}\Delta t^{-1} l_{s} $, where $ l_s $ is the length of the $ s $-th edge. 

Assuming a uniform distribution of the sampled points (in the sense of infinity norm) in the configuration space, the density of points is $ \rho = Nw^{-d} $. Accordingly, the dispersion \cite{Lavalle04Grid,Hsu06_IJRR_sampling} of the sampled points is $ \delta = wN^{-d^{-1}} $. Assuming all edges have equal length (in the $ l^{\infty}$-norm sense), the edge length of the underlying PRM (over which FIRM has been built) is $ l_{s}=\delta=w\sqrt[d]{N}^{-1} $.
\begin{align}
n_{coll} &= (n_b^{off}\Delta t^{-1})wkN^{1-d^{-1}}
\end{align}

\ph{Online phase}
In the online phase, we connect each node to all nodes in the neighborhood of radius $ R $ (in infinity norm). Thus, the size of neighboring area for connection is $ R^{d} $, which encompasses $ R^{d}*\rho $ neighboring points. For $ R = r\delta $, it will encompass $ r^{d} $ points. Thus, we have $ r^{d} $ new edges in the online phase. It can be shown that the length of $ (i+1)^{d}-i^{d} $ of these edges is in the range $i\delta<edgeLength<(i+1)\delta$.

For all edge lengths between $i\delta<l_{s}=edgeLength<(i+1)\delta$, let's approximate $l_{s} $ by $ i^{+}\delta $ where $ i\leq i^{+}\leq i+1 $. Then, the sum of the length of all new edges is:
\begin{align}
\nonumber
L_{s} &= \sum_{s=1}^{r^{d}}l_{s} =  \sum_{i=1}^{r}\sum_{s = (i-1)^{d}+1}^{i^{d}}l_{s} =\delta\sum_{i=1}^{r}((i)^{d}-(i-1)^{d}-1)i^{+}
\end{align}

Let us denote the number of particles describing belief by $ n_b $. The number of simulation calls (including collision checks) is:
\begin{align}
\nonumber
n_{coll} &= n_b\Delta t^{-1}L_{s}= n_b\Delta t^{-1}\sqrt[d]{N^{-1}}w
\sum_{i=1}^{R\sqrt[d]{N}w^{-1}}\!\!\!\!\!((i)^{d}-(i-1)^{d}-1)i^{+}
\end{align}

Upper/lower bounds on the number of collision checks can be obtained by setting $ i^{+} $ to its upper and lower bounds, i.e., $ i+1 $ and $ i $. To gain further insight on the complexity, let's assume $ i^{+} $ is a constant (i.e., all edge lengths are the same) and set it to its maximum value $ i^{+}=R\sqrt[d]{N}w^{-1} $. Then, the upper bound on collision checks $ n^{+}_{coll} $ is:
\begin{align}
\nonumber
n^{+}_{coll} &= (n_b\Delta t^{-1}wN^{-d^{-1}})
(R\sqrt[d]{N}w^{-1})
[(R\sqrt[d]{N}w^{-1})^{d}-R\sqrt[d]{N}w^{-1}]\\
&=n_b\Delta t^{-1}w^{-d}R^{d+1}N
-n_b\Delta t^{-1}w^{-1}R^{2}\sqrt[d]{N}
\end{align}
Given this upper-bound on the computation time and given uniform grid sampling strategy, the online computation time grows sub-linearly with the number of underlying FIRM nodes $ N $ in the worst case. Also, for a given dimension the online computation time is polynomial in the connection radius $ R $. By removing the dimension from the equation and extending the results to random sampling, we can write the first term of the above equation as:
\begin{align}
\nonumber
n^{+}_{coll} &= (n_b\Delta t^{-1})R V\rho
\end{align}
where $ \rho $ is the density of samples in the environment, $ V $ is the volume of the connection neighborhood, and $ R $ is the radius of the connection neighborhood.

\subsection{Enabling SLAP in changing environments}\label{sec:SLAPinChange}
In this section, we discuss the ability of the proposed planner to handle changes in the obstacle map. We focus on a challenging case, where changes in the obstacle map are persistent and can possibly eliminate a homotopy class of solutions. Doors are an important example of this class. If the robot observes a door is closed (which was expected to be open), it might have to globally change the plan to get to the goal from a different passage. This poses a challenge to the state-of-the-art methods in the belief space planning literature. 

To handle such changes in the obstacle map and replan accordingly, we propose a method for lazy evaluation of the generated feedback tree, referred to as ``lazy feedback evaluation'' algorithm, inspired by the lazy evaluation methods for PRM frameworks \cite{bohlin00lazyPRM}. The basic idea is that at every node the robot re-evaluates \textit{only} the next edge (or the next few edges up to a fixed horizon) that the robot will most likely take. By re-evaluation, we mean it needs to re-compute the collision probabilities along these edges. If there is a significant change (measured by $ \alpha $ in Algorithm \ref{alg:lazy-feedback-eval}) in the collision probabilities, the dynamic programming problem is re-solved and new cost-to-go values are computed. Otherwise, the cost-to-go values remain unchanged and the robot keeps following its rollout-policy. Such lazy evaluation (computing the collision probabilities for a single edge or a small number of edges) can be performed online. The method is detailed in Algorithm \ref{alg:lazy-feedback-eval}.

\begin{algorithm}[h!]
\caption{Lazy Feedback Evaluation (Lazy Replanning)}\label{alg:lazy-feedback-eval}
\textbf{input}  :  FIRM graph\\
\textbf{output} :  Updated cost-to-go, $J^{g}(\cdot)$ and success probabilities $ P^{success}(\cdot) $\\
{
Perceive the obstacles map; \\
\If{there is a change in map} 
{
$ \mathcal{F}\leftarrow $ Retrieve the sequence of nominal edges returned by feedback up to horizon $ l $; Set $ ctr = 0 $;\\
\ForAll{edges $ \mu \in \mathcal{F}$}
{Re-compute collision probabilities $ \mathbb{P}_{new}(B,\mu) $ \axx{from starting node $ B $ of edge $ \mu $};\\
\If{$ |\mathbb{P}_{new}(B,\mu) - \mathbb{P}(B,\mu)| > \alpha $}
{$\mathbb{P}(B,\mu)\leftarrow\mathbb{P}_{new}(B,\mu)$;\\
$ ctr\leftarrow ctr+1 $;}
} 
\If{$ ctr>0 $}
{Update edge set $ \mathbb{M} $ based on new transition probabilities;\\
$ [J^{g}(\cdot),P^{success}(\cdot)] $ = DynamicProgramming($ \mathbb{V},\mathbb{M} $);}
}
\Return $J^{g}(\cdot)$ and $P^{success}(\cdot)$;
}
\end{algorithm}

\textit{Remark}: Another challenge with these persistent changes is that they stay in the memory. Imagine a case where the robot is in a room with two doors and the goal point is outside the room. Suppose that after checking both doors, the robot realizes they are closed. To solve such cases, the door state is reset to ``open'' after a specific amount of time to persuade the robot to recheck the state of doors. We further discuss the concept of the ``forgetting time" in the experiments section.

It is important to note that it is the particular structure of the proposed planner that makes such replanning feasible online. The graph structure of the underlying FIRM allows us to \textit{locally} change the collision probabilities in the environment without affecting the collision probability of the rest of the graph \axx{(i.e., properties of different edges on the graph are independent of each other; see Fig. \ref{fig:funnel-FIRM} and \ref{fig:tree-with-funnels}).} Such a property is not present in the state-of-the-art sampling-based belief space planners (e.g., \cite{Prentice09},\cite{Berg11-IJRR}), where the collision probabilities and costs on \textit{all} edges are dependent on each other and hence need to be re-computed. 

\section{Simulation Results}\label{sec:rollout-simulation}
In this section, we demonstrate the performance of the method in simulation. The robot is tasked to go from a start location to multiple goal locations sequentially in an obstacle-laden environment with narrow passages and asymmetrically placed landmarks.

\textit{Motion Model}: 
The state of the robot at time 
$k$ is denoted by $ x_k = (\mathsf{x}_k, \mathsf{y}_k, \mathsf{\theta}_k)^T $ (2D position and the heading angle). We denote the control as $ u_k = (v_{x,k}, v_{y,k},\omega_k)^T $ and the process noise by $ w_k=(n_{v_x}, n_{v_y},n_{\omega})^T\sim\mathcal{N}(0,\mathbf{Q}_k) $.
Let $f$ denote the kinematics of the robot such that $x_{k+1} = f(x_k, u_k, w_k)=x_{k}+u_k\delta t+w_k\sqrt{\delta t}$.

\textit{Observation Model}:
We use a range-bearing landmark based observation model. Each landmark has a unique fully-observable ID. Let ${^i}\mathbf{L}$ be the location of the $i$-th landmark. The displacement vector ${^i}\mathbf{d}$ from the robot to ${^i}\mathbf{L}$ is given by ${^i}\mathbf{d}=[{^i}d_{x}, {^i}d_{y}]^T:={^i}\mathbf{L}-\mathbf{p}$, where $\mathbf{p}=[\mathsf{x},\mathsf{y}]^T$ is the position of the robot. Therefore, the observation ${^i}z$ of the $i$-th landmark can be described as ${^i}z={^i}h(x,{^i}v)=[\|{^i}\mathbf{d}\|,\text{atan2}({^i}d_{y},{^i}d_{x})-\theta]^T+{^i}v$. The observation noise is assumed to be zero-mean Gaussian  such that $ {^i}v\sim\mathcal {N}(\mathbf{0},{^i}\mathbf{R}) $ where ${^i}\mathbf{R}=\text{diag}((\eta_r\|{^i}\mathbf{d}\|+\sigma^r_b)^2,(\eta_{\theta}\|{^i}\mathbf{d}\|+\sigma^{\theta}_b)^2)$. The measurement quality decreases as the robot gets farther from the landmarks and the parameters $\eta_r$ and $\eta_{\theta}$ determine this dependency. $\sigma_b^r$ and $\sigma_b^{\theta}$ are the bias standard deviations. 

\textit{Environment and Scenario}:
The environment in Fig. \ref{fig:environment} represents a 2D office space measuring $21$m $\times$ $21$m. 
The robot is a disk with diameter $1$m. There are two narrow passages $ P1 $ and $ P2 $ with high collision probability. The narrow passages are $1.25$m wide, thus offering a very small clearance for the robot to pass through. The robot is placed at starting location `A'  and tasked to visit 4 different locations (B, C, D, and E) in 4 sequential segments: 1) $ A\rightarrow B $,  2) $ B\rightarrow C $, 3) $ C\rightarrow D $, and  4) $ D\rightarrow E $. We compare the performance of the standard FIRM with the proposed rollout-based method. 

\begin{figure}
\centering
\subfigure[]
{\includegraphics[height=1.7in]{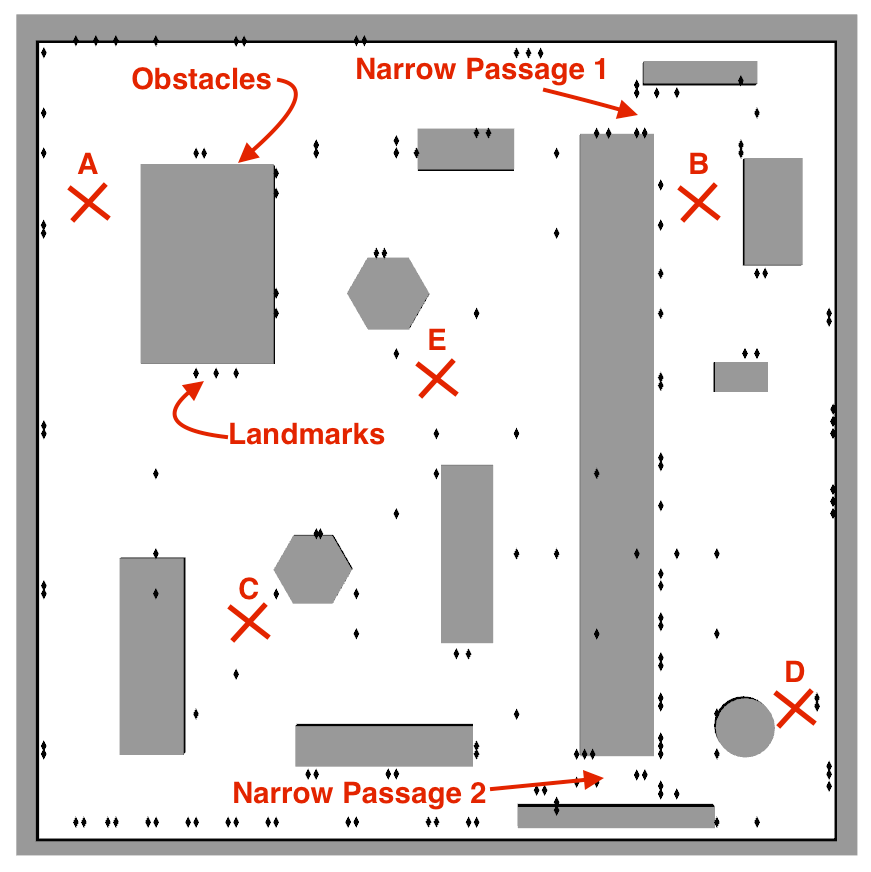}\label{fig:environment}}
\subfigure[]
{\includegraphics[height=1.7in]{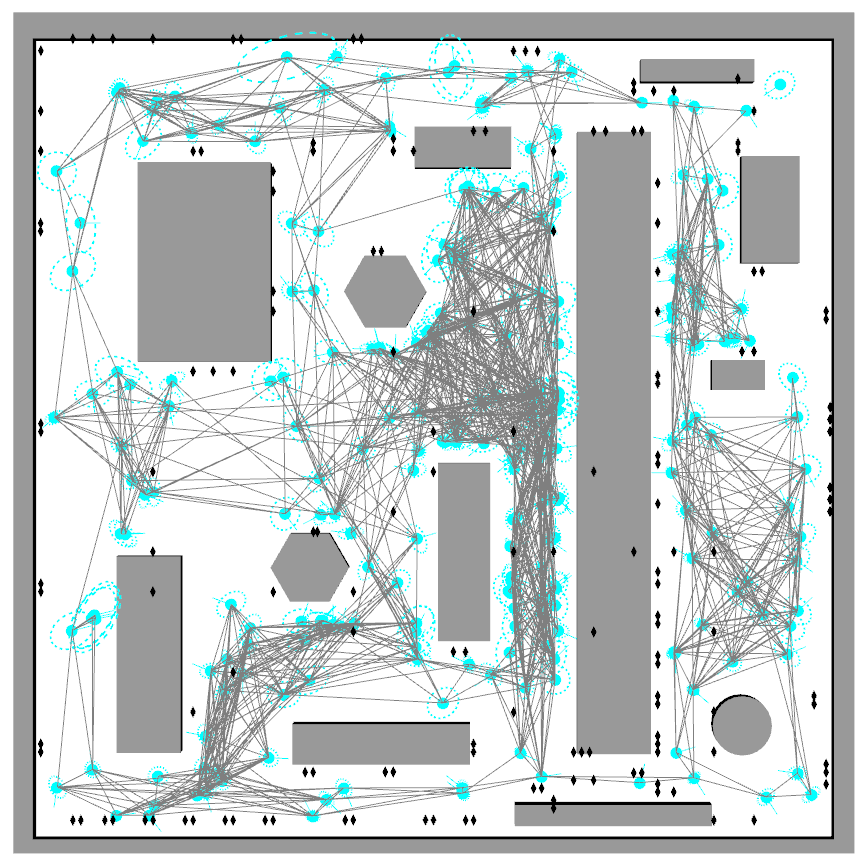}\label{fig:prm}}
\caption{\axx{(a) The simulation environment: landmarks (black diamonds) and obstacles (grey polygons). The locations of interest that the robot is tasked to visit are marked by red crosses. The two narrow passages P1 and P2 represent regions of high collision probability (risky) due to the small clearance. (b) The underlying FIRM roadmap: edges (grey lines), nodes (cyan disks), covariance of the FIRM nodes (dashed ellipses).}}
\label{fig:envAndPRM}
\end{figure}

\subsection{Planning with Standard FIRM}
First, we construct the FIRM roadmap offline (see Fig. \ref{fig:prm}).

\textit{FIRM nodes}: \axx{The roadmap is constructed by uniformly sampling configurations in the free space. Then, corresponding to each configuration node, we create a FIRM node (belief) by 
following the procedure in Section \ref{subsec:FIRM-elements}. In short, we linearize the system dynamics and sensor model around the sampled configuration point. We create the Kalman Filter corresponding to this local linear system and find its reachable belief by solving the corresponding Riccati equation. At each node, there exists a stabilizing controller which locally drives all beliefs to the belief node}. 

\textit{FIRM edges}: The edges of the FIRM roadmap are generated by first finding valid (collision free) straight line connections between neighboring nodes 
and then generating edge controllers which drive the belief from the starting belief of the edge to the vicinity of the target belief of the edge. For each edge in the graph, we run Monte Carlo simulations to compute the expected execution cost and transition probability. The constructed FIRM roadmap is stored for use in the online rollout phase.

\textit{Online-phase}: In the online phase, the planner receives a query (i.e., starting and goal configuration). These configurations are added to the existing roadmap by computing the appropriate stationary belief, stabilizer, and edge controllers. Since this construction preserves the optimal sub-structure property \axx{(i.e., edges are independent of each other; see Fig. \ref{fig:funnel-FIRM} and \ref{fig:tree-with-funnels})}, we can solve dynamic programming on the graph for the given goal location to construct the feedback tree.

\textbf{FIRM Feedback Tree}: The solution of the dynamic programming problem (i.e., $ \pi^{g} $), is visualized with a \textit{feedback tree}. Recall that $ \pi^{g} $ is a mapping (look-up table) that returns the next best edge for any given graph node. The feedback tree is rooted at the goal node. For each node, the feedback tree contains only one outgoing edge ($ \mu=\pi^{g}(B^{i}) $) that pulls the robot towards the goal. 

\textbf{Most-Likely Path (MLP)}: The most likely path is defined as a path followed by the FIRM feedback if there was no noise (i.e., it is a tree branch that connects start to goal.) Note that the actual solution (generated by FIRM) can be arbitrarily different from the MLP due to noise.

In segment 1 ($ A\rightarrow B $), the planner computes the feedback tree (Fig. \ref{fig:firm-segment-1-a}) rooted in $ B $. Fig. \ref{fig:firm-segment-1-b} shows the MLP. In this case, the noise is not high enough to change the homotopy class, and the robot is close to the MLP. To reach the goal, the robot follows edge controllers returned by the feedback policy and stabilizes to all FIRM nodes along the path. Once the robot reaches $ B $, we submit a new goal $ C $. A new feedback tree is computed online rooted in $ C $ (Fig. \ref{fig:firm-segment-2-a}), with MLP shown in Fig. \ref{fig:firm-segment-2-b}. We follow a similar procedure to accomplish segments $ C\rightarrow D $ and $ D\rightarrow E$.

\begin{figure}
\centering
\subfigure[Feedback tree (green) for goal location $ B $, robot (blue disk), and initial belief (red).]{\includegraphics[height=1.65in]{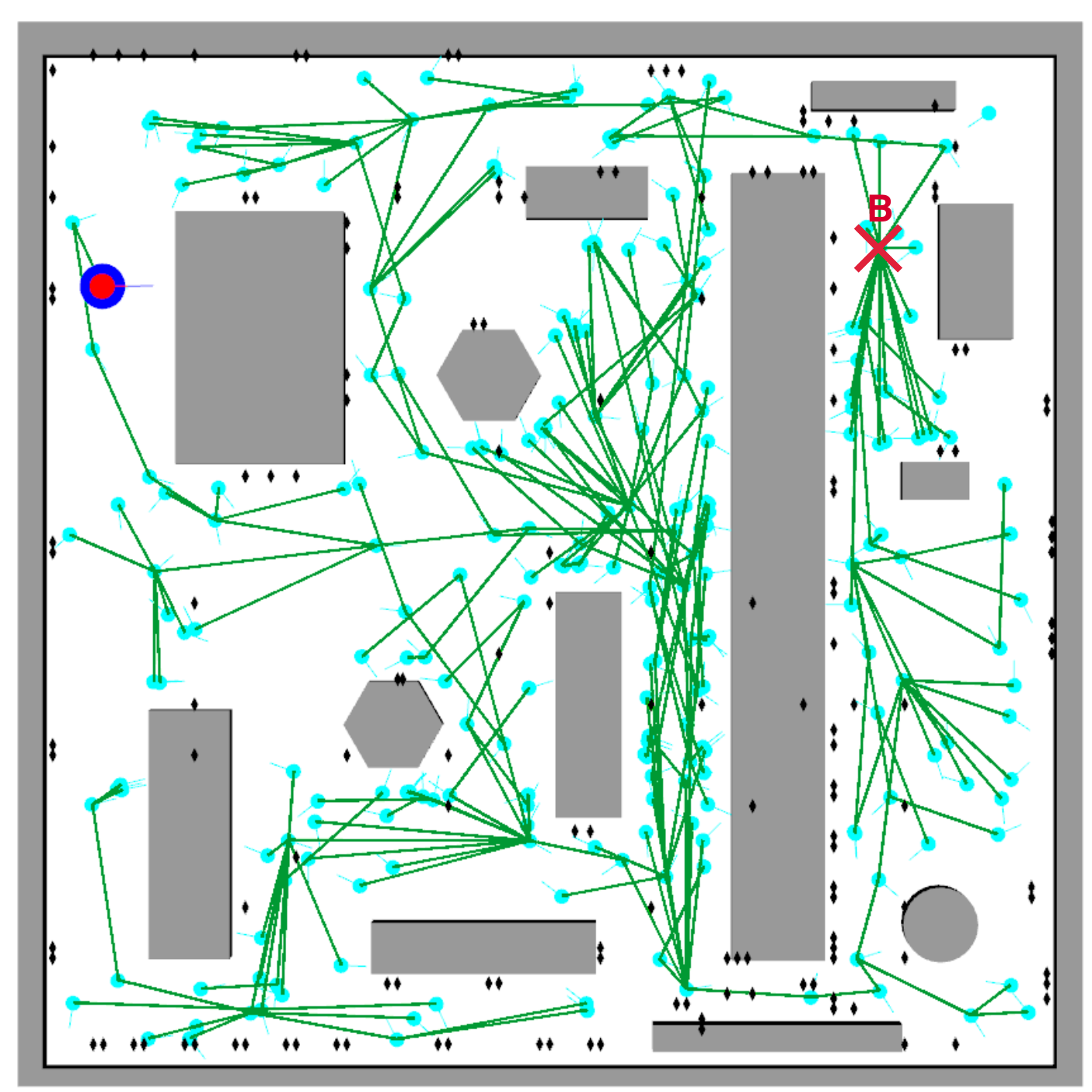}\label{fig:firm-segment-1-a}}
\hspace{0.1in}
\subfigure[The most likely path (purple) under FIRM from A to B.]{\includegraphics[height=1.65in]{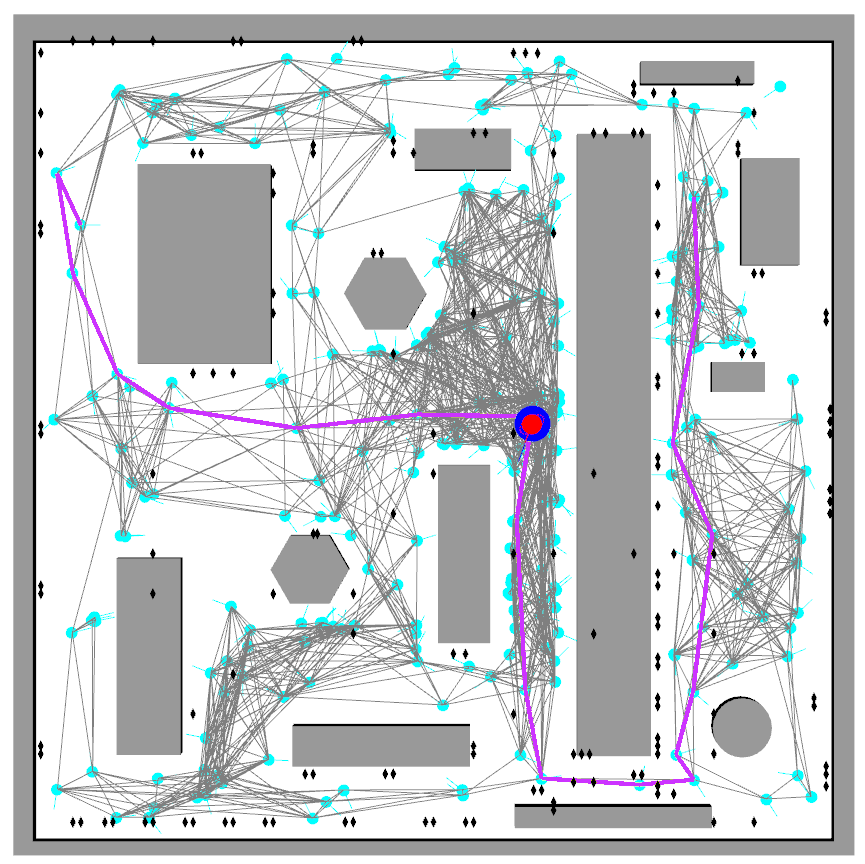}\label{fig:firm-segment-1-b}}
\caption{Segment 1 of policy execution with FIRM, starting at A and going to B. The locations A and B are marked in Fig. \ref{fig:environment}}.
\end{figure}

\begin{figure}
\centering
\subfigure[The FIRM feedback for goal location C.]{\includegraphics[height=1.65in]{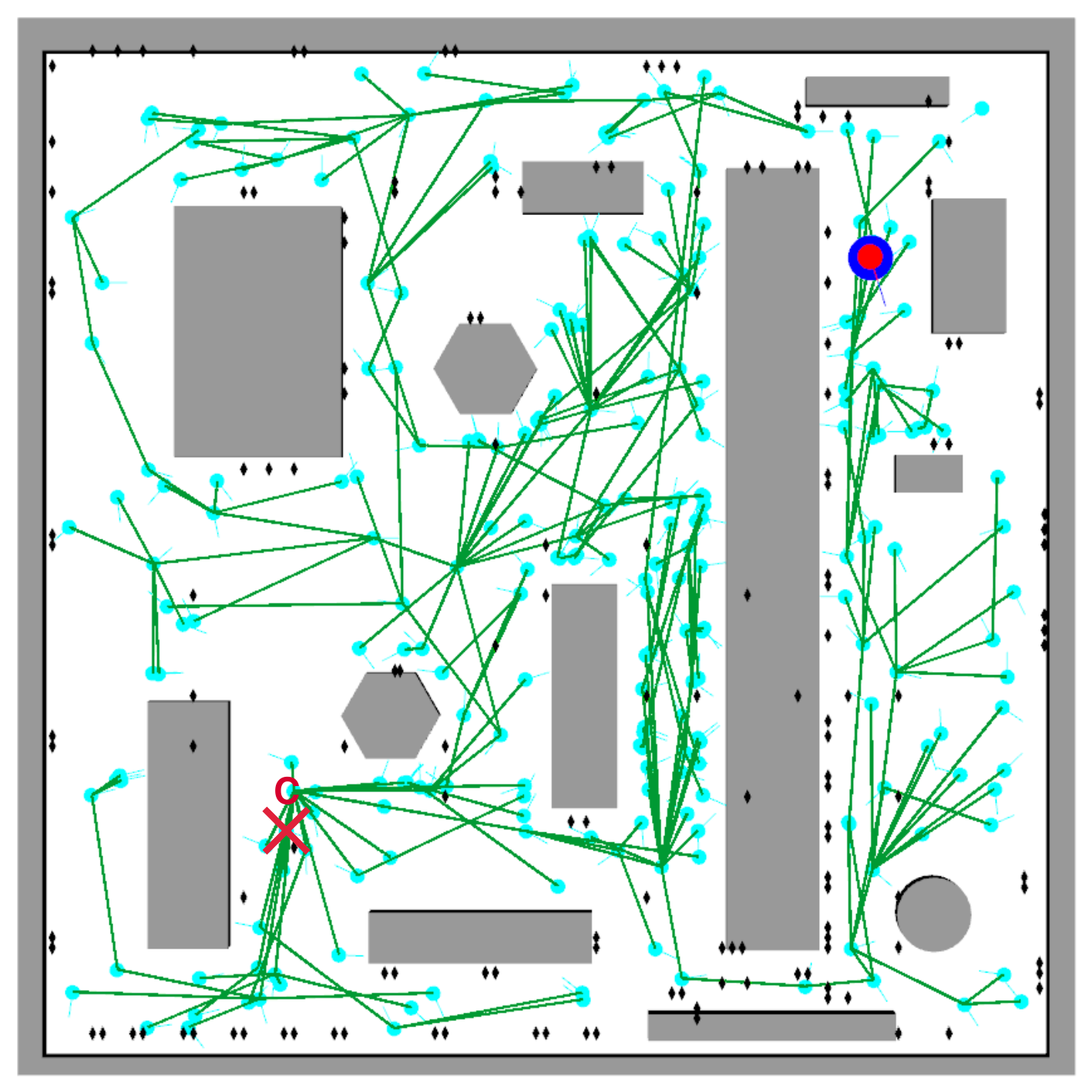}\label{fig:firm-segment-2-a}}
\hspace{0.1in}
\subfigure[The MLP (purple) under FIRM from B to C.]
{\includegraphics[height=1.65in]{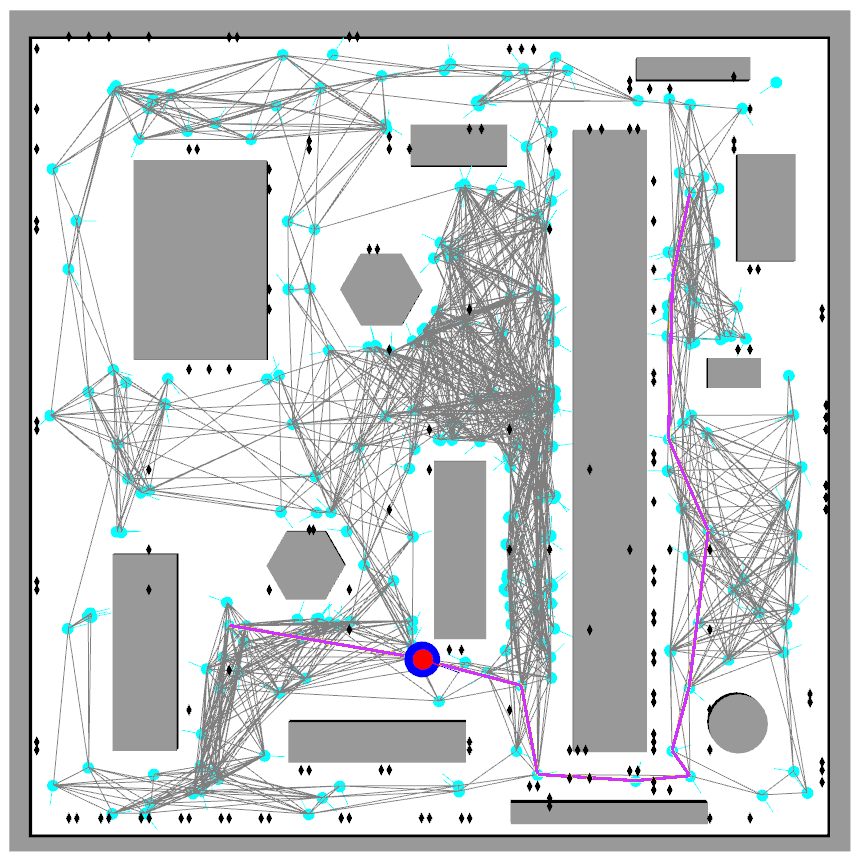}\label{fig:firm-segment-2-b}}
\caption{Segment 2 ($ B\rightarrow C $) of policy execution with FIRM.}
\end{figure}

\subsection{Planning with the proposed method}
For segment 1 ($ A\rightarrow B $), as before, we begin with the underlying FIRM roadmap constructed offline and compute the feedback tree (Fig. \ref{fig:firm-segment-1-a}). However, when rollout planner follows the feedback tree, a different behavior emerges. At each time step (or more generally every  $T_{rollout}$ seconds), the planner connects the current robot belief to neighboring FIRM nodes in radius $R$ (i.e., the planner locally generates edge controllers with their associated cost and transition probability). Then, the planner checks which connection provides the lowest sum of the edge-cost and cost-to-go from its landing node (Eq. \eqref{eq:rollout-minimization}). 
The connection with the lowest sum is chosen as the next edge to follow. 
Fig. \ref{fig:rollout-segment-1-a} shows the planner checking connections (red-edges) locally to neighboring FIRM nodes. 

An important behavior emerges in segment 1. As the robot proceeds, the rollout is able to find a shorter path through the relatively open area by skipping unnecessary stabilizations (Fig. \ref{fig:rollout-segment-1-b} and \ref{fig:rollout-segment-1-c}). As the robot traverses the narrow passage $ P2 $, the rollout realizes ``stabilizing'' to the FIRM node is the best option as it concludes it is better to reduce the uncertainty to a safe level before proceeding through the narrow passage (Fig. \ref{fig:rollout-segment-1-d}). Eventually the robot reaches location B through the path as marked in green in Fig. \ref{fig:rollout-segment-1-f}. Rollout gives the robot a distinct advantage over the nominal FIRM plan as it guides the robot through a shorter and faster route. Furthermore, it should be noted that although the last part of the two paths (after exiting the narrow passage) look similar, they differ significantly in the velocity profiles. Along the purple path, the robot stabilizes to each and every FIRM node. But, along the green path (rollout), the robot maintains a higher average velocity by skipping unnecessary stabilizations. The robot performs full or partial stabilization only when the gained information (reduced uncertainty and risk) is necessary to complete the mission.

\newcommand{\SimFigSize}{1.4in}
\begin{figure}
\centering
\subfigure[The robot checks new connections with neighbors (red). The MLP under the nominal FIRM feedback is in purple.]{\includegraphics[height=\SimFigSize]{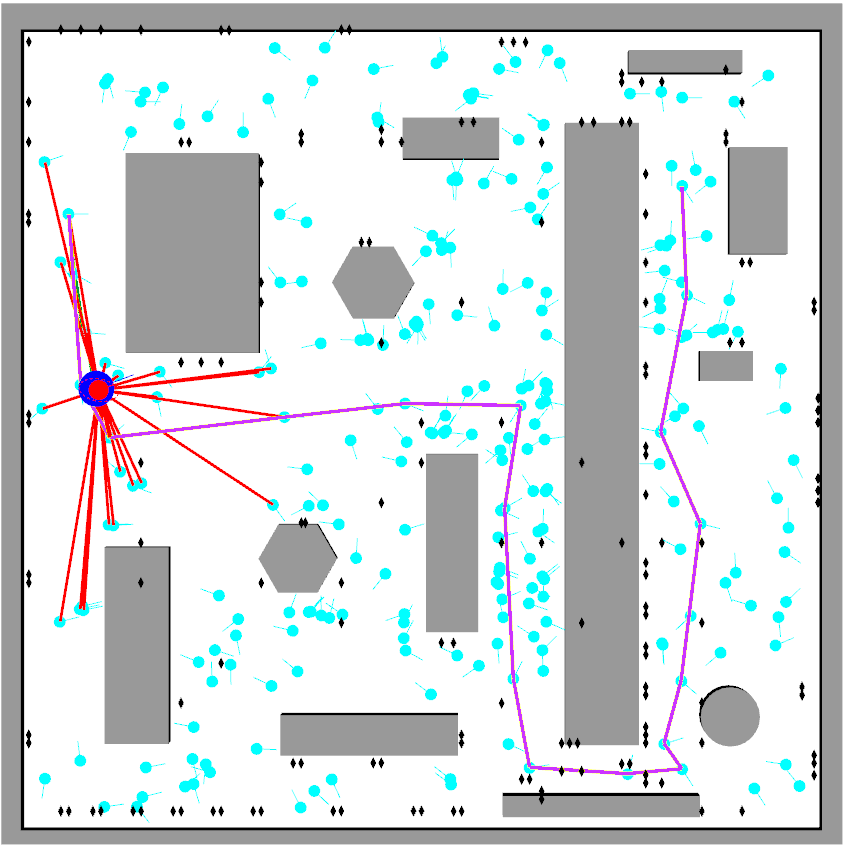}\label{fig:rollout-segment-1-a}}
\hspace{0.1in}
\subfigure[Rollout guides the robot away from the MLP to a shorter and faster path (green). The new path is in a different homotopy class.]{\includegraphics[height=\SimFigSize]{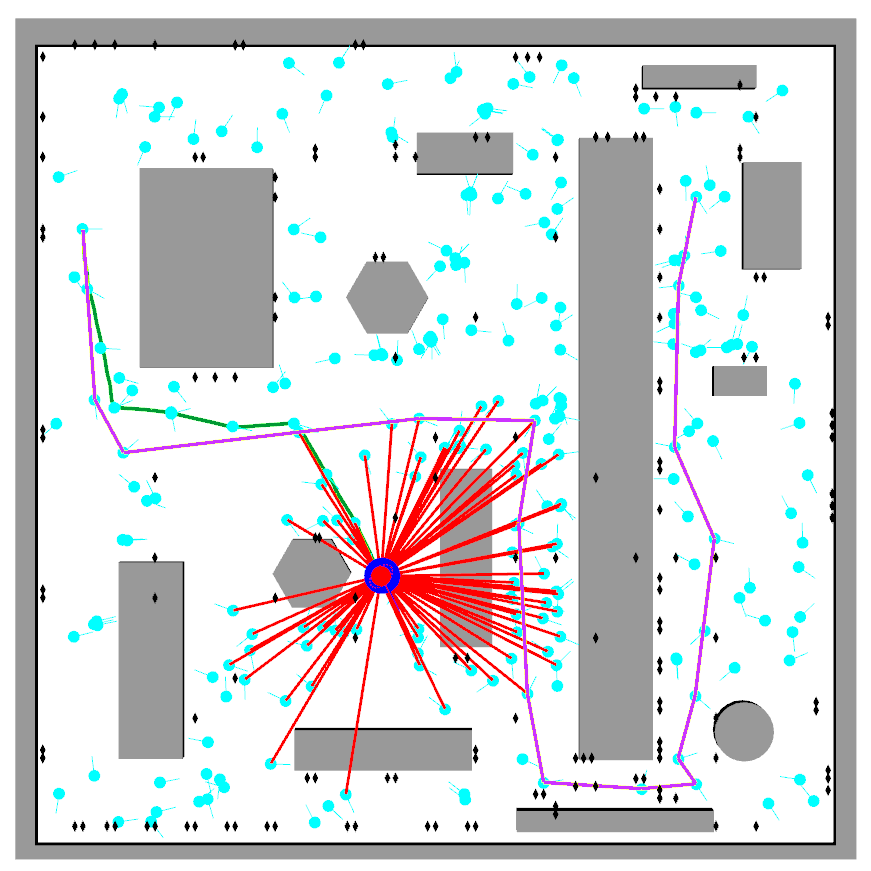}\label{fig:rollout-segment-1-b}}\quad
\subfigure[Robot approaches narrow passage $ P2 $ through a more direct path as compared to the MLP.]{\includegraphics[height=\SimFigSize]{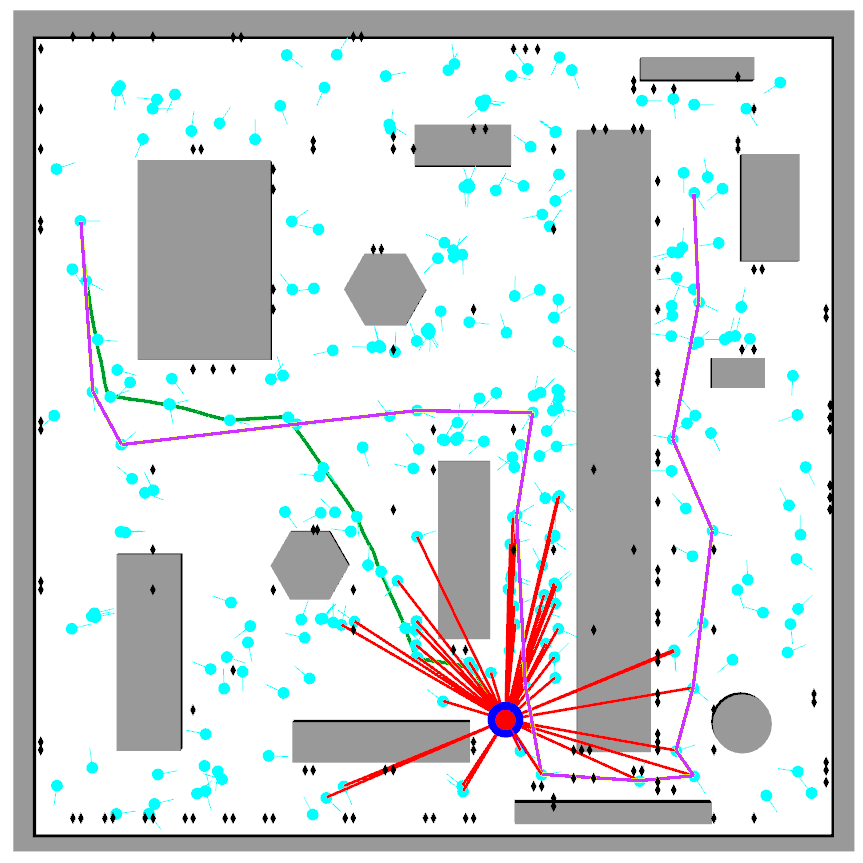}\label{fig:rollout-segment-1-c}}
\hspace{0.1in}
\subfigure[The robot stabilizes at a few FIRM nodes while passing through the narrow passage.]{\includegraphics[height=\SimFigSize]{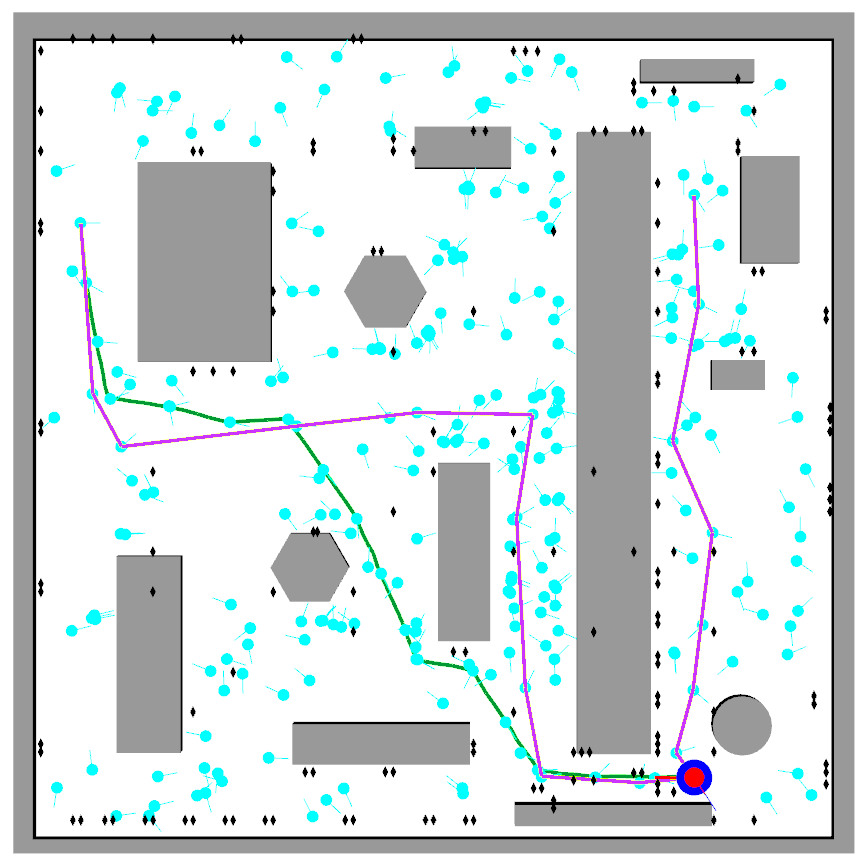}\label{fig:rollout-segment-1-d}}\quad
\subfigure[The robot approaches goal location B.]{\includegraphics[height=\SimFigSize]{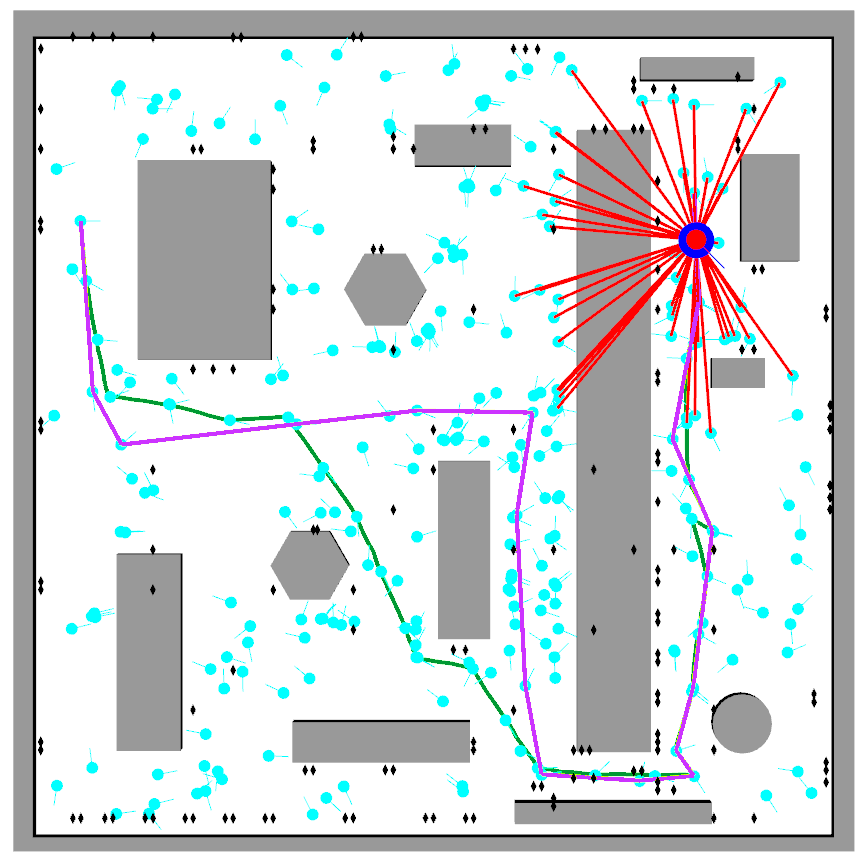}\label{fig:rollout-segment-1-e}}    
\hspace{0.1in}
\subfigure[The rollout path is shorter and faster than the FIRM's MLP. 
] {\includegraphics[height=\SimFigSize]{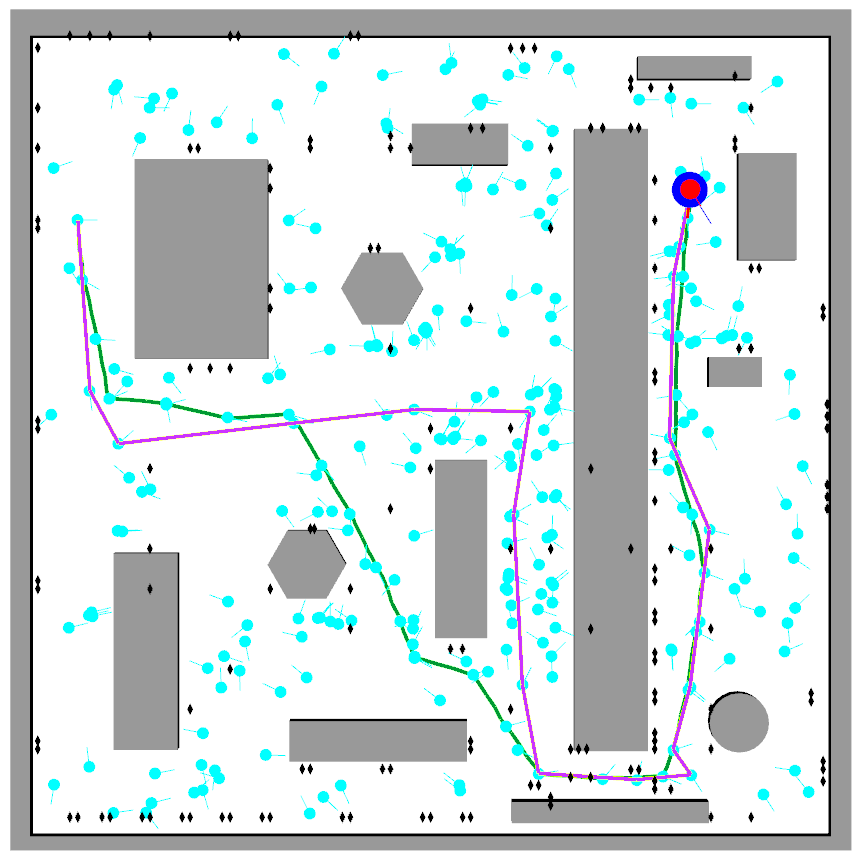}\label{fig:rollout-segment-1-f}}
\caption{ Segment 1 with rollout: Starting at A and going to B.}
\end{figure}

Similar behaviors are observed when completing segments 2 ($ B\rightarrow C $) and 3 ($ C\rightarrow D $) and 4 ($ D\rightarrow E $). Fig. \ref{fig:compare-segment-3-4} shows final paths for segments 3 and 4. We observe different levels of stabilization and different path shapes when passing through passage $ P2 $ in segment 3 and 4. This is due to asymmetric distribution of information and risk in the environment. 

\begin{figure}
	\centering
	\subfigure[Segment 3 ($ C\rightarrow D $).]
	{\includegraphics[height=\SimFigSize]{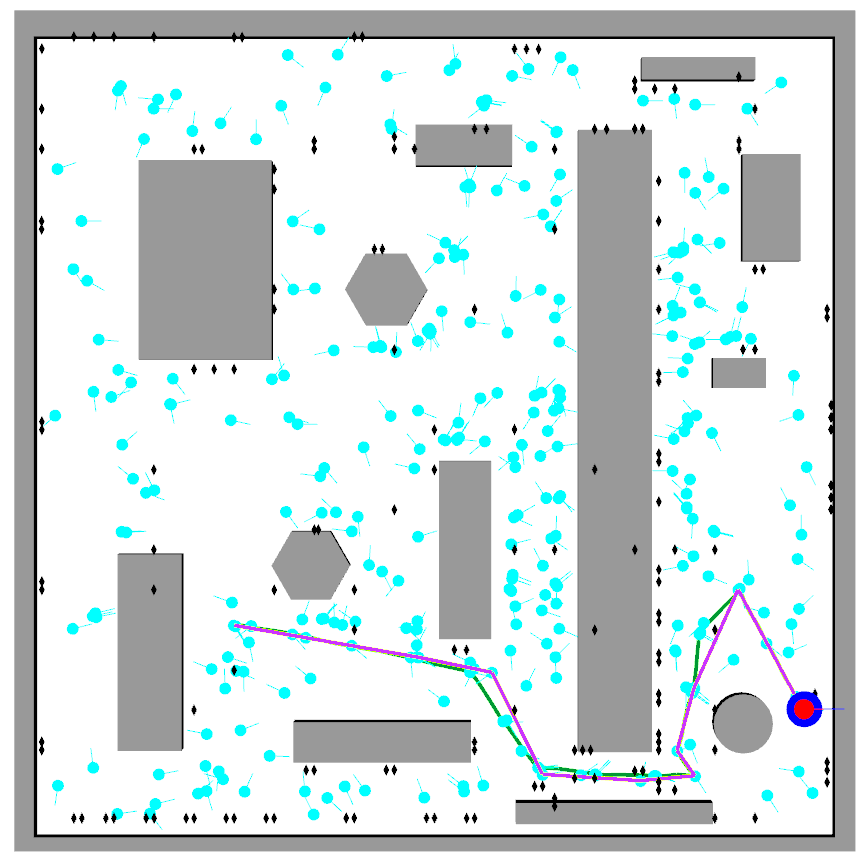}\label{fig:rollout-segment-3-b}}
	\hspace{0.1in}
		\subfigure[Segment 4 ($ D\rightarrow E $).]
		{\includegraphics[height=\SimFigSize]{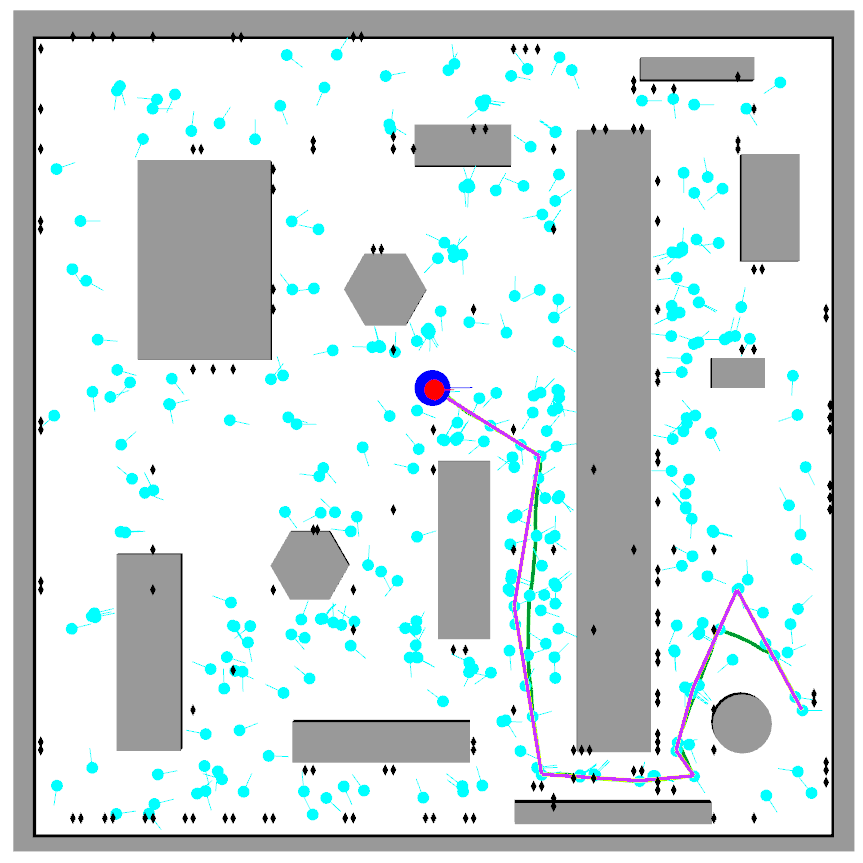}\label{fig:rollout-segment-4-b}}
	\caption{Asymmetric costs and random execution noises.}
	\label{fig:compare-segment-3-4}
\end{figure}

\subsection{Analysis of Simulation Results}\label{subsec:rollout-sim-analysis}
In this section, we discuss the statistical analysis for the presented simulation results by running the planner multiple times. The results show that the proposed method significantly increases the performance of the standard FIRM implementation while preserving its robustness and scalability.

\ph{Cost of Execution} We recorded the amount of localization uncertainty (trace of covariance) along the robot's path. Figure \ref{fig:cost-vs-time} shows the cumulative version of this cost on 50 runs for the same task under the rollout-based planner and standard FIRM. We note that the cost for the rollout based policy rises slower than the cost for FIRM, and as the planning horizon increases, rollout offers increasing returns in performance. 

\ph{Selective stabilization}
Node stabilization makes FIRM robust and scalable while maintaining the optimal sub-structure property on the graph \axx{(via edge independence; see Fig. \ref{fig:funnel-FIRM})}. Although stabilization allows FIRM to provide certain guarantees, it adds stabilization time and cost at each node to the time and cost of the mission.
The rollout-based planner brings a higher level of intelligence to the process of node stabilization. Rollout performs stabilization when required and bypasses it when possible. Bypassing the stabilization allows the robot to complete the task faster and with less cost. Fig.\ref{fig:nodes-reached} shows the number of nodes the robot has stabilized to on 50 different runs. In this example, the robot stabilizes to $ \sim\!\!45 $ nodes under FIRM compared to $ \sim\!\!10 $ nodes under the rollout-based planner ($ \sim\!\!75 $\% reduction), while the difference is growing as the task becomes longer.

\ph{Time of Task completion}
Another quantitative performance measure is the time it takes for a planner to complete the task while guaranteeing a high likelihood of success. From Fig. \ref{fig:cost-vs-time} and \ref{fig:nodes-reached}, the time taken to complete the task with rollout is around $2500$ time-steps ($250$ seconds) compared to $3000$ time-steps ($300$ seconds) for FIRM. There is $\sim\!\!15$\% reduction in the time to complete the task compared to the standard FIRM algorithm. The improvement in execution time makes the rollout-based planner a better candidate than FIRM for time-sensitive applications.

\begin{figure}
\centering
\subfigure[]{\includegraphics[height=1.35in]{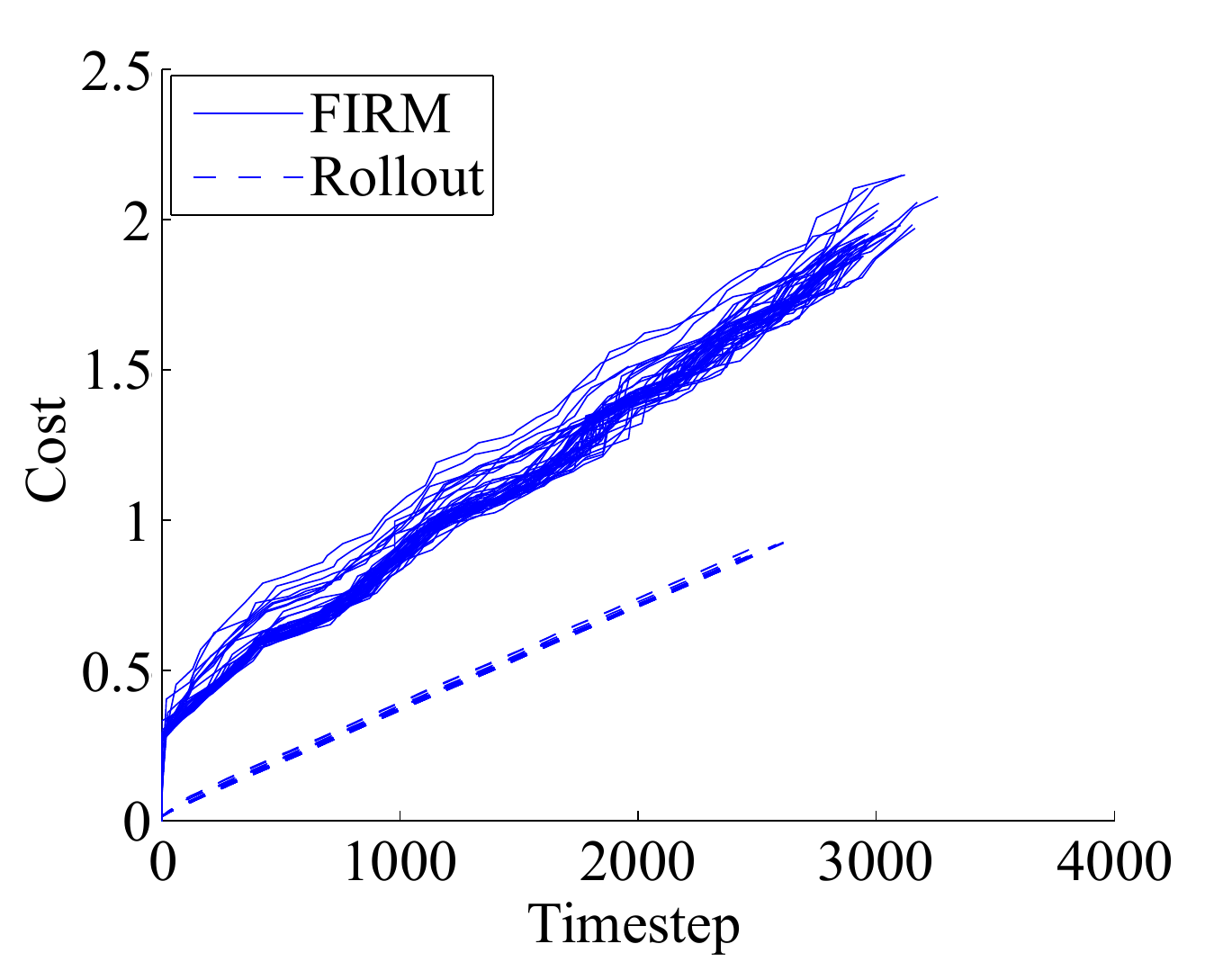}\label{fig:cost-vs-time}}
\hspace{-0.1in}
\subfigure[]{\includegraphics[height=1.35in]{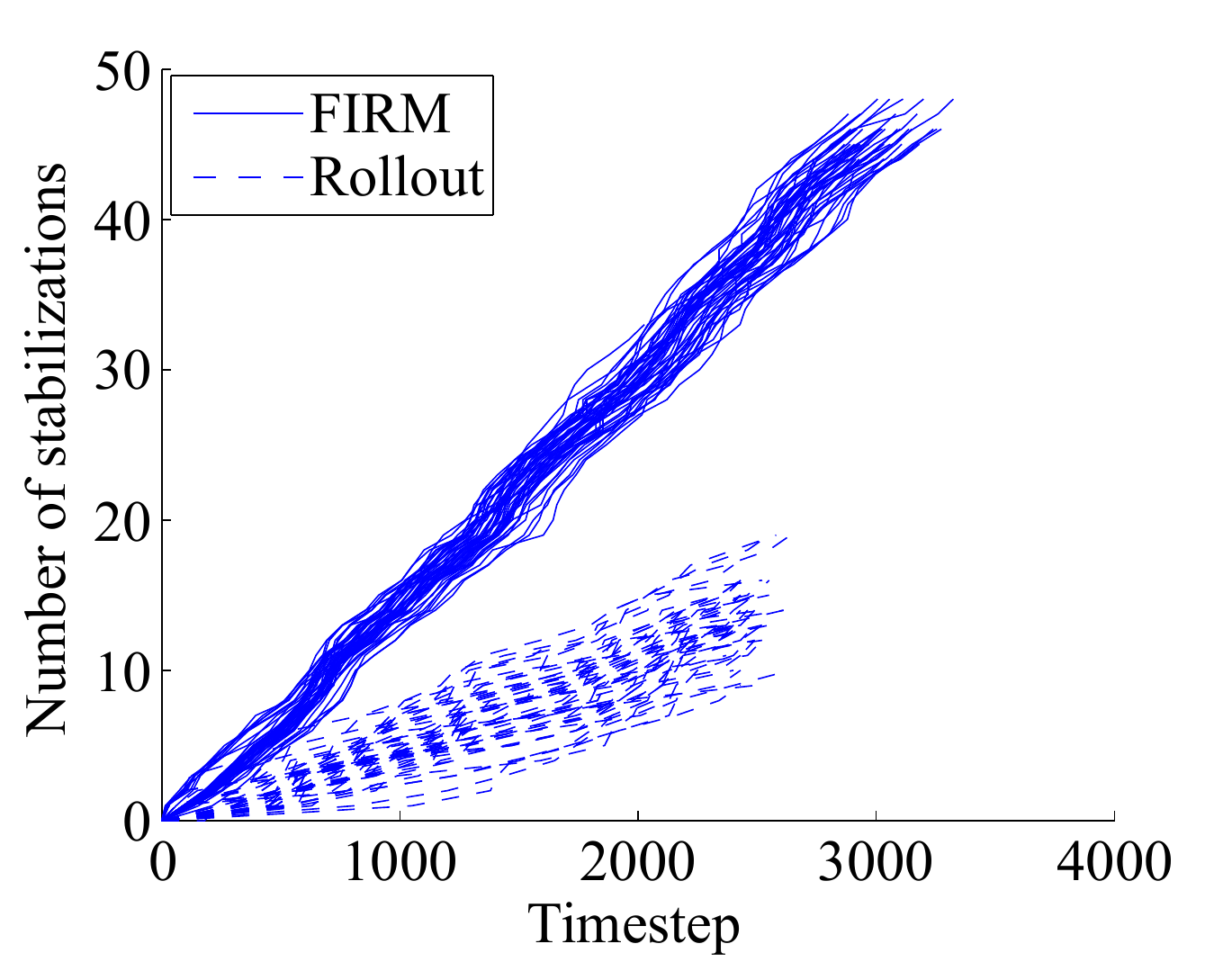}\label{fig:nodes-reached}}
\caption{Performance comparison between the original FIRM and proposed planner on 50 runs. (a) The execution cost for FIRM rises faster than the cost of the rollout-based policy. (b) The number of belief nodes that the robot stabilizes to, during plan execution, is consistently lower for the rollout.}
\end{figure}

\begin{figure*}[ht!]
	\centering
	\hspace{0in}
	\subfigure[\axx{The cost to complete the task reduces as the number of underlying graph nodes increases (due to availability of more options). Sharp dips in the graph correspond to cases where adding a new node captures a new low-cost homotopy class.}]
	{\includegraphics[height=1.7in]{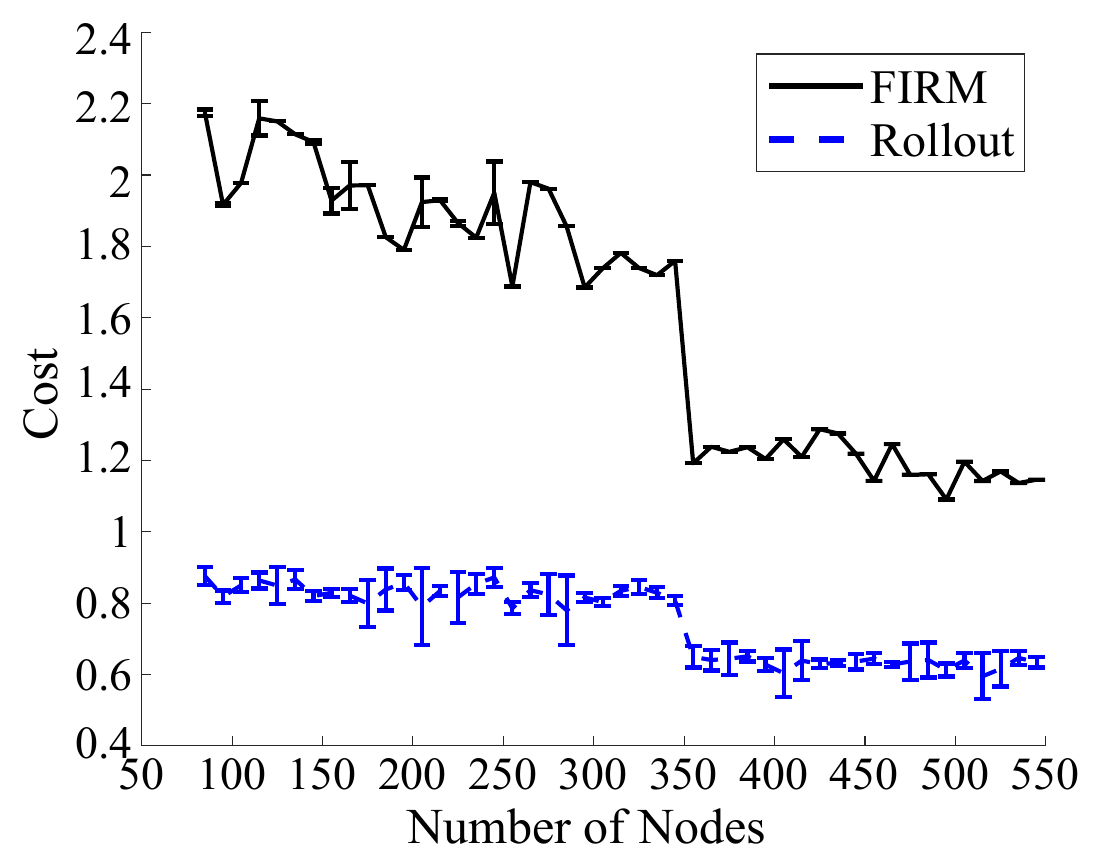}\label{fig:cost-vs-nodes}}
	\hspace{0.07in}
	\subfigure[Time taken by the robot to complete the mission. As graph density increases robot finds more nodes to connect to during the rollout phase and thus can take more shortcuts which reduces its total driving time.]{\includegraphics[height=1.7in]{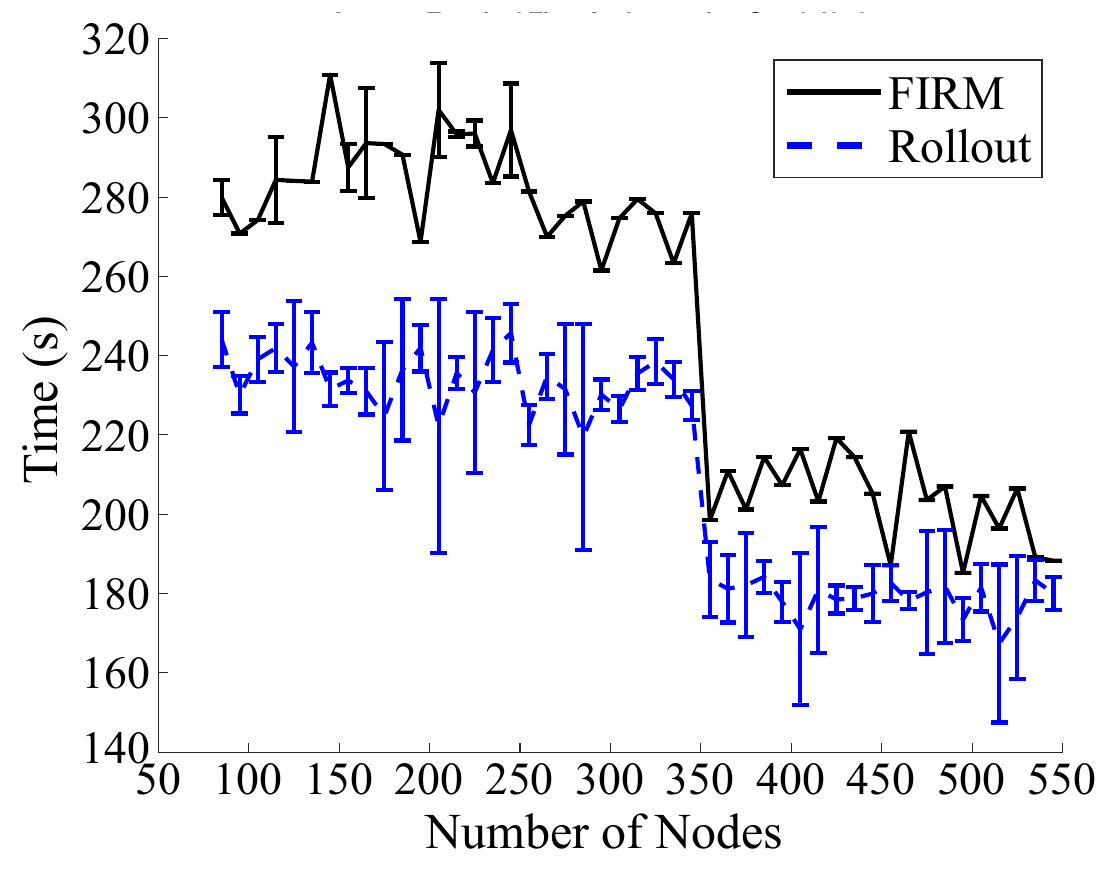}\label{fig:time-vs-nodes}}
	\hspace{0.07in}
	\subfigure[\axx{The number of visited nodes (the number of stabilizations) in rollout is significantly smaller than FIRM.}]{\includegraphics[height=1.7in]{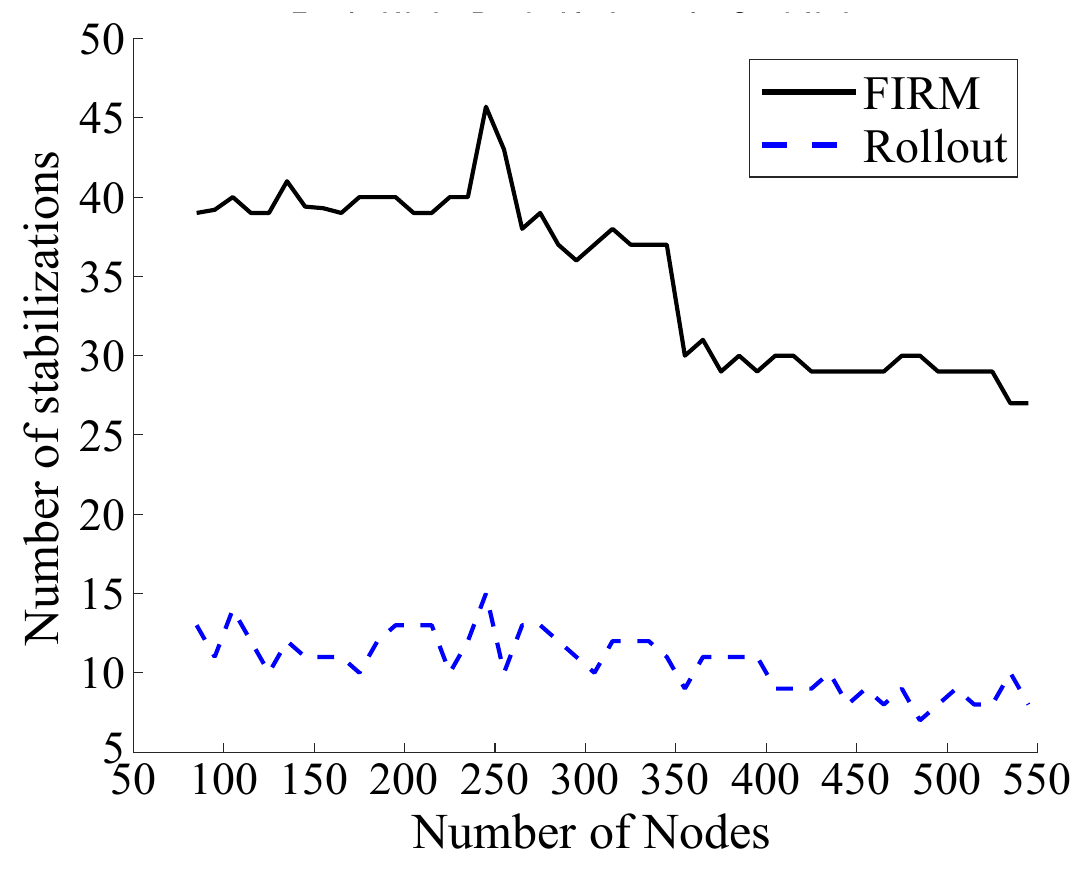}\label{fig:nodes-reached-vs-nodes}}
	\caption{The effect of increasing FIRM graph density on the rollout solution behavior. Each graph node is connected to all nodes within its radius $R=5$. As the number of nodes in the graph crosses 350, a new connection through narrow passage 2 is found, which leads to sharp changes in the graphs.}
	\label{fig:increase-node-density}
\end{figure*}

\begin{figure}
	\centering
	\includegraphics[width=1.8in]{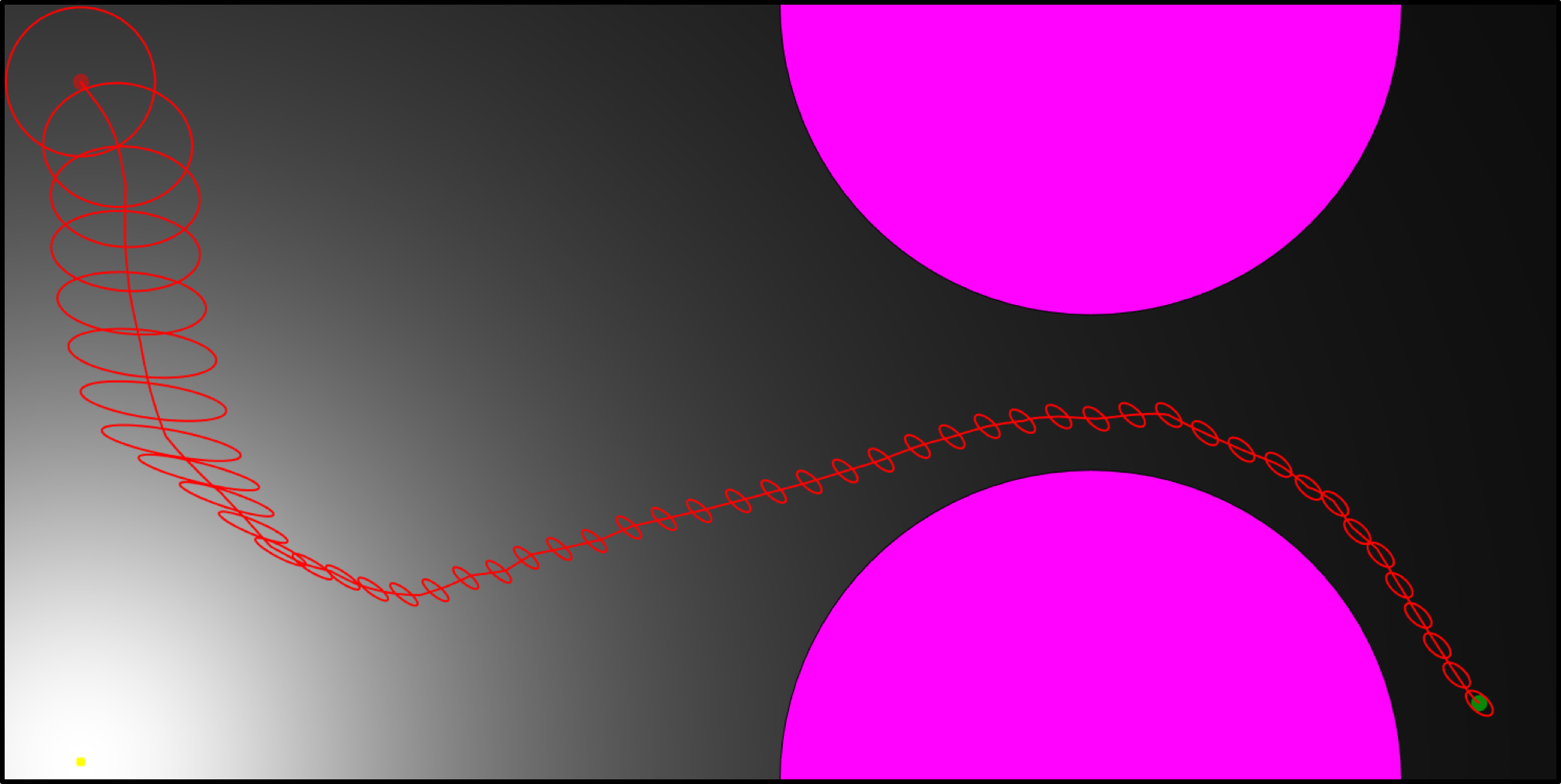}
	\caption{A simple environment with two obstacles (magenta) and one information beacon (light area in the bottom-left corner). The local optimum from start (top-left) to goal (bottom-right) is shown (the red trajectory). In this environment, there is only one local optimum (i.e., the global optimum).}
	\label{fig:simpelEnv}
\end{figure}

\begin{figure}
	\centering
	\subfigure[\axx{The initial guess (in red) for ILO-based methods using deterministic RRT planner (black tree).}]{\includegraphics[height=1.6in]{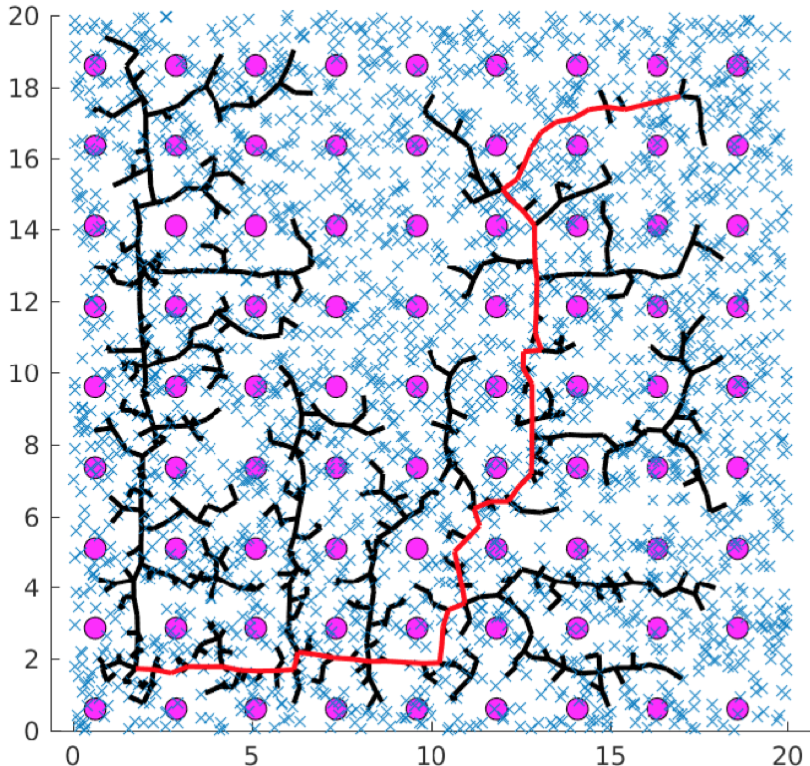}\label{fig:rrt-task3-solution}}
	\subfigure[\axx{The locally optimized solution (red) and the path under the rollout policy (green).}]{\includegraphics[height=1.5in]{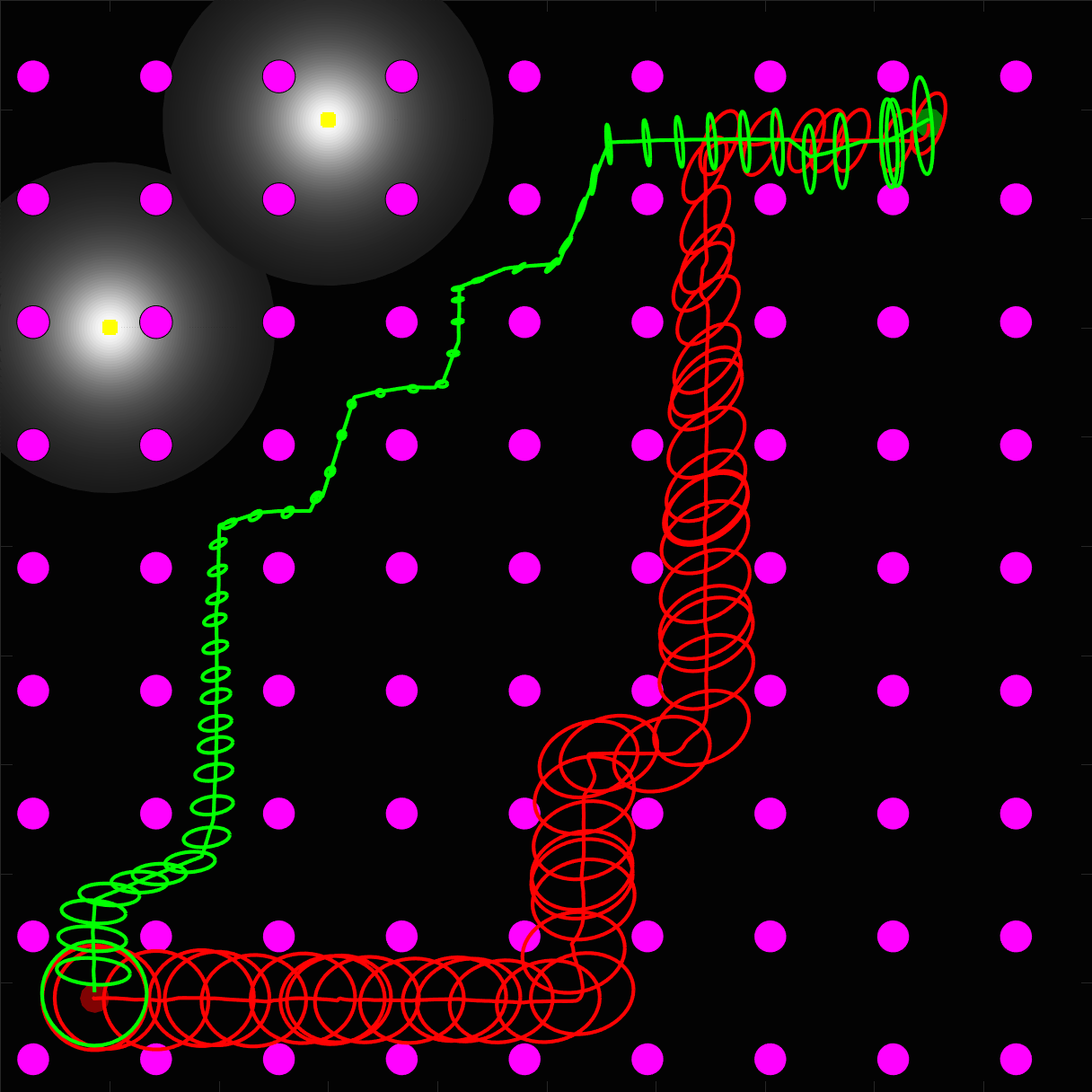}\label{fig:ilqg-rollout-task3-solution}}
	\caption{Comparison of rollout vs local optimization-based methods. Obstacles (magenta) and information beacons (yellow light sources) are shown, whose signal strength declines quadratically with the robot-beacon distance. (a) Local optimization-based methods require an initial solution. A deterministic RRT (black) generates an initial solution (red). (b) The final solution computed by belief space ILO (red) is restricted to the homotopy of the RRT solution. On the other hand, the rollout-based policy guides the robot (green) toward the global optimum by exploiting the underlying global feedback structure.}
	\label{fig:ilqg_solution_task3}
\end{figure}

\ph{Varying node density}
To further analyze the results, we study the performance of the method as a function of offline graph density. Fig. \ref{fig:increase-node-density} shows how the cost, number of stabilizations, and time to complete the task change as the density of underlying graph increases. 

\subsection{Comparison with State-of-the-Art}
We compare our proposed method with iterative local optimization-based (ILO-based) methods. Extended LQR \cite{Berg016extended} for deterministic systems, Iterative Linear Quadratic Gaussian Control \cite{todorov-acc-ilqg}, and its belief space variant \cite{vanderberg-ijrr-ilqg} (referred to as BSP-iLQG, here) are among the pioneering ILO-based methods. Due to their optimization-based nature, ILO methods perform well when the problem has a single local optimum (i.e., the global optimum). When there are multiple local minima, the performance of ILO methods are sensitive to the initial solution. 

In belief space variants of ILO methods, the problem is typically solved in two phases. First, ignoring all uncertainties, a deterministic motion planning problem is solved (e.g., using RRT in \cite{vanderberg-ijrr-ilqg}) to find an initial trajectory from start to goal.  Second, the generated trajectory is utilized as the initial solution for a local optimization process in belief space. In our simulations, we use a holonomic 2D robot and point beacon observation model similar to the one used in \cite{vanderberg-ijrr-ilqg} (Sec. 6.2.2) whose signal strength decreases quadratically with robot's distance from the beacon. We compare ILO-based approaches to the proposed method in two aspects: 1) The sensitivity of the quality of the solution to the initial guess and 2) replanning time.

Figure \ref{fig:simpelEnv} shows an environment where there exists a single optimum (i.e., the local optimum is identical to the global optimum). In this environment, ILO-based methods perform well and can converge to the optimal solution. However, environments with more obstacles and multiple homotopy classes and/or environments with more complex information distribution induce multiple local minima, which degrades the performance of the ILO-based methods. Fig.~\ref{fig:ilqg_solution_task3} shows one such environment. The initial RRT as shown in Fig. \ref{fig:rrt-task3-solution} computes a path that takes the robot towards the goal diagonally. This limits the resulting local optimum solution to a homotopy class, quite different from the global optimum. On the other hand, the proposed rollout-based method does not require any initial solution, and will be able to find the optimal homotopy spanned by the underlying graph. This key difference is shown in Fig.~\ref{fig:ilqg-rollout-task3-solution}, where the generated solution (green) optimally exploits the information distribution (beacons are in the upper left corner) in the environment. In addition, the rollout-based method can update and repair the homotopy class during the execution to compensate for potential deviations due to the noise.

In addition to the solution quality, the replanning time in ILO-based methods grows with the problem horizon. For example, in BSP-iLQG, the complexity of the optimization algorithm is in the order of $ \mathcal{O}(Init)+\mathcal{O}(N_lN_i) $, where $\mathcal{O}(Init)$ refers to the complexity of computing an initial guess (e.g., solving a deterministic motion planning algorithm such as RRT), and $ \mathcal{O}(N_lN_i) $ refers to the complexity of belief optimization. $ N_l $ is the trajectory length and $ N_i $ is the number of iterations for the optimization to converge. Thus, the computational complexity grows unboundedly with the planning horizon (path length). 

Here, we experimentally compare the computational complexity (replanning time) of the proposed method with ILO-based methods. Although computing the initial solution $ \mathcal{O}(Init) $ in ILO-based methods can take significant amount of time, here, to simplify the results, we only report the time ILO spends on belief optimization, i.e., $\mathcal{O}(N_lN_i)$, and do not include the initial solution generation time in the graphs for ILO-based methods. 

We compare the results on the forest environment (like Fig. \ref{fig:ilqg_solution_task3}). Comparison is carried out in C++ on a PC with 3.40GHz Quad-Core Intel i7-3770 CPU and 16GB RAM running Ubuntu 14.04. We grow the environment and planning horizon at each step. As the forest grows, we maintain the same obstacle density as well as the information source density. The starting point is at the bottom-left and the goal is at the top-right for all environment sizes. For each environment size, we run the planners five times to record the statistical variations. Fig.~\ref{fig:timeComp} shows how the planning time of ILO-based methods increases as the problem size grows. On the other hand, the complexity of the proposed rollout-based method is constant regardless of the planning horizon. These results confirm the analyses in Section \ref{sec:complexity}. In this forest environment, the average replanning time is 80.1ms with a variance of 35.1ms for rollout connections (edge length) of 2.25m length (on average). Such rapid replanning capability allows for dynamic replanning in belief space to enable SLAP.
\begin{figure}
	\centering
	\includegraphics[width=2.8in]{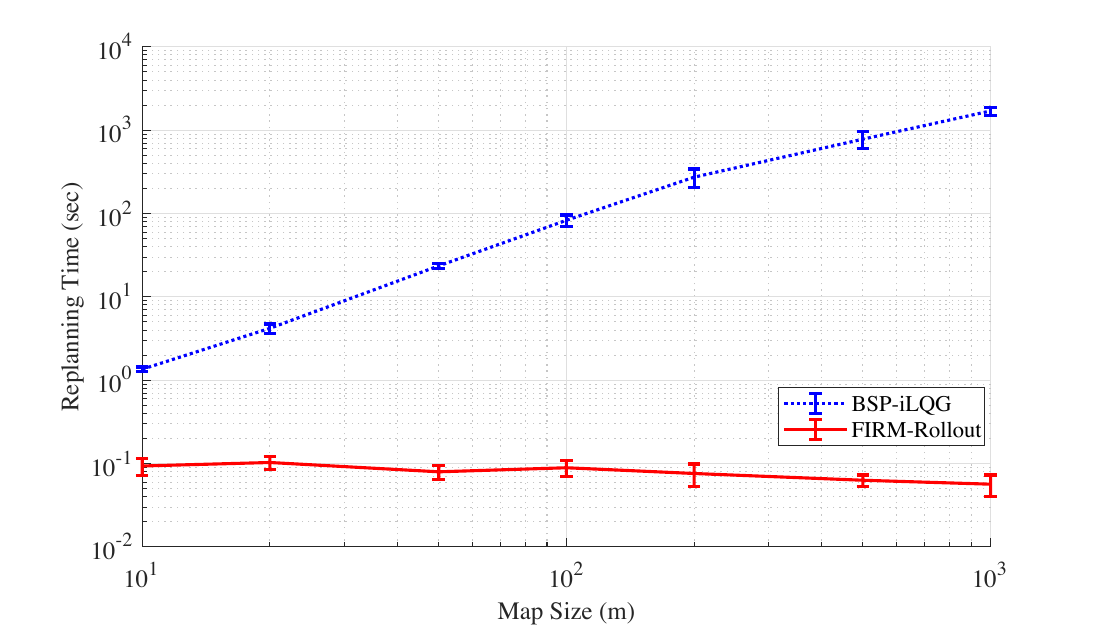}
	\caption{\axx{Planning time: ILO in belief space (blue) versus the proposed method (red). As the planning horizon (distance between start and goal) increases, local optimization-based methods take more time to plan and converge, whereas the planning time in rollout methods is not a function of planning horizon.}}
	\label{fig:timeComp}
\end{figure}

\section{Experimental Results for a Physical System} \label{sec:experiments}
In this section, we demonstrate the proposed online POMDP-based solution for the simultaneous localization and planning problem on a physical robot. We discuss the details of implementation and report the important lessons learned. A video demonstrating the experiments is available at \cite{youtube-video-RHC-FIRM}.

\vspace{\gapBeforeSection}
\subsection{Target Application and System Set Up} \label{subsec:systemSetUp}
Consider a scenario where the robot needs to operate and reach a goal in an office environment. Each time the robot reaches a goal, a new goal is submitted by a higher-level application (e.g., manually by a user or multiple users). We investigate the performance and robustness of the method to (i) changing obstacles, such as doors, and moving people, (ii) changes in the goal location, (iii) deviations due to intermittent sensory failures, and (iv) kidnap situations (significant sudden deviation in the robot's location). In all these cases, an online replanning scheme can help the robot to recover from the off-nominal situation and accomplish its goal. 
In particular, we study the ``kidnapped'' situation, where a person might move the robot to an unknown location during the plan execution, and the robot needs to recover from this catastrophic deviation. 
The main focus of the experiments in this section is to demonstrate SLAP on physical robots by enabling online belief space (re)planning.

\subsubsection{Environment} \label{subsec:environment}
Our experiments are conducted in the fourth floor of the Harvey Bum Bright (HRBB) building at the Texas A\&M University campus. The floor-plan is shown in Fig. \ref{fig:floorPlan}. The floor spans almost 40 meters of hallways whose average width is approximately 2 meters, which is distinguished in yellow and blue in Fig. \ref{fig:floorPlan}. The particular set of experiments reported in this paper is conducted in the region which is highlighted in blue in Fig. \ref{fig:floorPlan}, part of which contains a large cluttered office area (407-area). This area has interesting properties that makes the planning more challenging: (i) 407-area is obstacle-laden (chairs/desks and workstations). (ii) As is seen in Fig. \ref{fig:floorPlan}, there are several doors in this area which may be open or closed. Two of these doors (front-door and back-door) are labeled in Fig. \ref{fig:floorPlan}. (iii) There are objects such as chairs and trash-cans in this environment which usually get displaced. (iv) There are moving people, who are avoided using a reactive behavior, which may displace the robot from its planned path. 
\begin{figure}[ht!]
\centering
\includegraphics[width=2.7in]{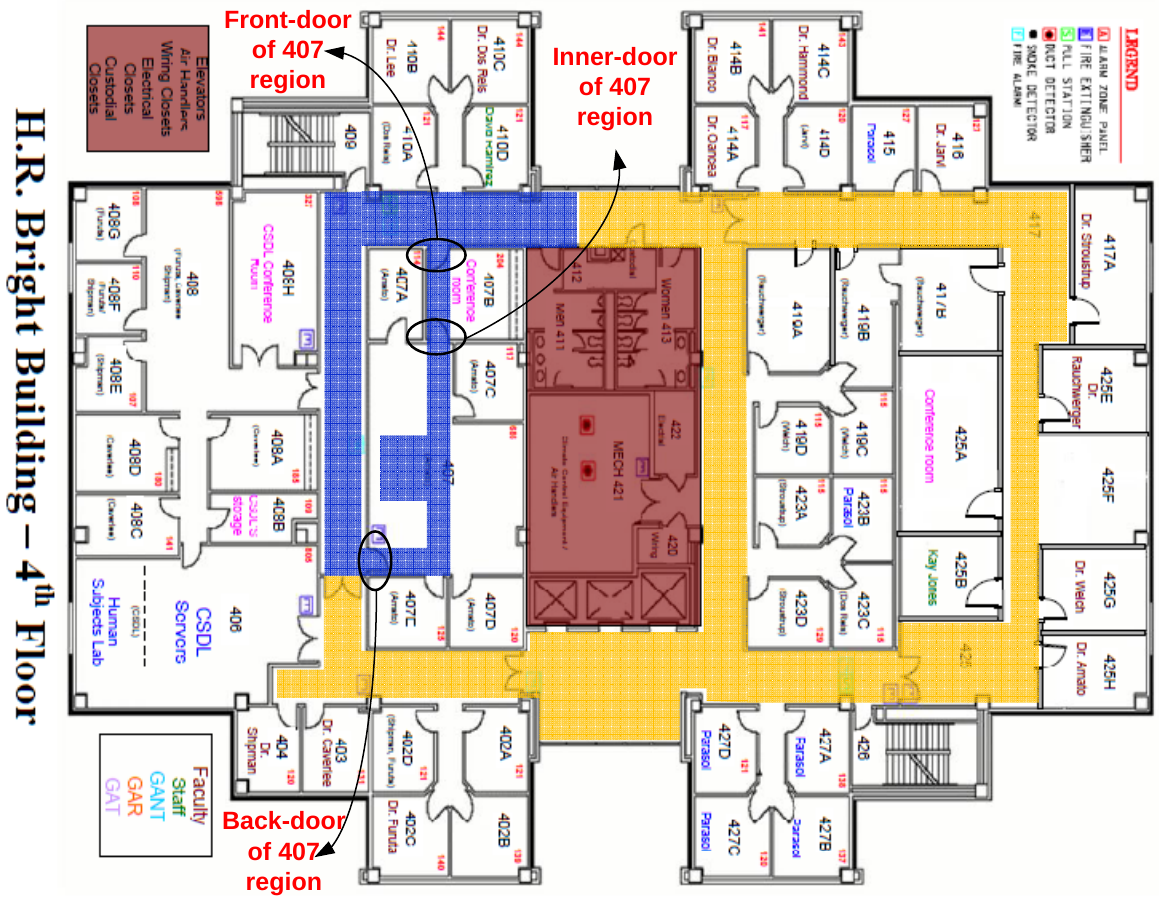}
\caption{Floor-plan of the environment, in which experiments are conducted.}
\label{fig:floorPlan}
\end{figure}

\subsubsection{Robot Model} \label{subsec:robotModel}
The physical platform utilized in our experiments is an iRobot Create mobile robot (See Fig. \ref{fig:robot-landmark}). The robot can be modeled as a unicycle with the following kinematics:
\begin{align}\label{eq:unicycle-motion-model}
\!\!\!\!x_{k+1} \!=\! f(x_k,u_k,w_k) \!=\!
\left(\!\!
\begin{array}{c}
\mathsf{x}_{k}+(V_k\delta t + n_v\sqrt{\delta t})\cos\theta_k \\
\mathsf{y}_{k}+(V_k\delta t + n_v\sqrt{\delta t})\sin\theta_k \\
\mathsf{\theta}_{k}+\omega_k\delta t + n_{\omega}\sqrt{\delta t}
\end{array}\!\right),
\end{align}
where 
$ x_k = (\mathsf{x}_k, \mathsf{y}_k, \mathsf{\theta}_k)^T $ describes the robot state at the $ k $-th time step. $(\mathsf{x}_k,\mathsf{y}_k)^T$ is the 2D position and $ \theta_k $ is the heading angle of the robot. Control commands are the linear and angular velocities $ u_k = (V_k,\omega_k)^T $. We use the Player robot interface \cite{gerkey2003player} to send the control commands to the robot. 

\ph{Motion noise} The motion noise vector is denoted by $w_k=(n_v,n_{\omega})^T\sim\mathcal{N}(0,\mathbf{Q}_k)$, which is mostly resulted from uneven tiles on the floor, wheel slippage, and inaccuracy in the duration of the applied control signals. Experimentally, we observed that in addition to the fixed uncertainty associated with the control commands, there exists a portion of the noise that is proportional to the signal strength. Thus, we model the variance of the process noise at the $ k $-th time step as $ \mathbf{Q}_{k} = \operatorname{diag}((\eta V_{k}+\sigma_b^{V})^{2}, (\eta \omega_{k}+\sigma_b^{\omega})^{2}) $, where in our implementations we have $ \eta = 0.03 $, $\sigma_b^{V} = 0.01\text{m/s}$, $ \sigma_b^{\omega} = 0.001 \text{rad} $.
\begin{figure}[ht!]
\vspace{-10pt}
\centering
\includegraphics[width=2.3in]{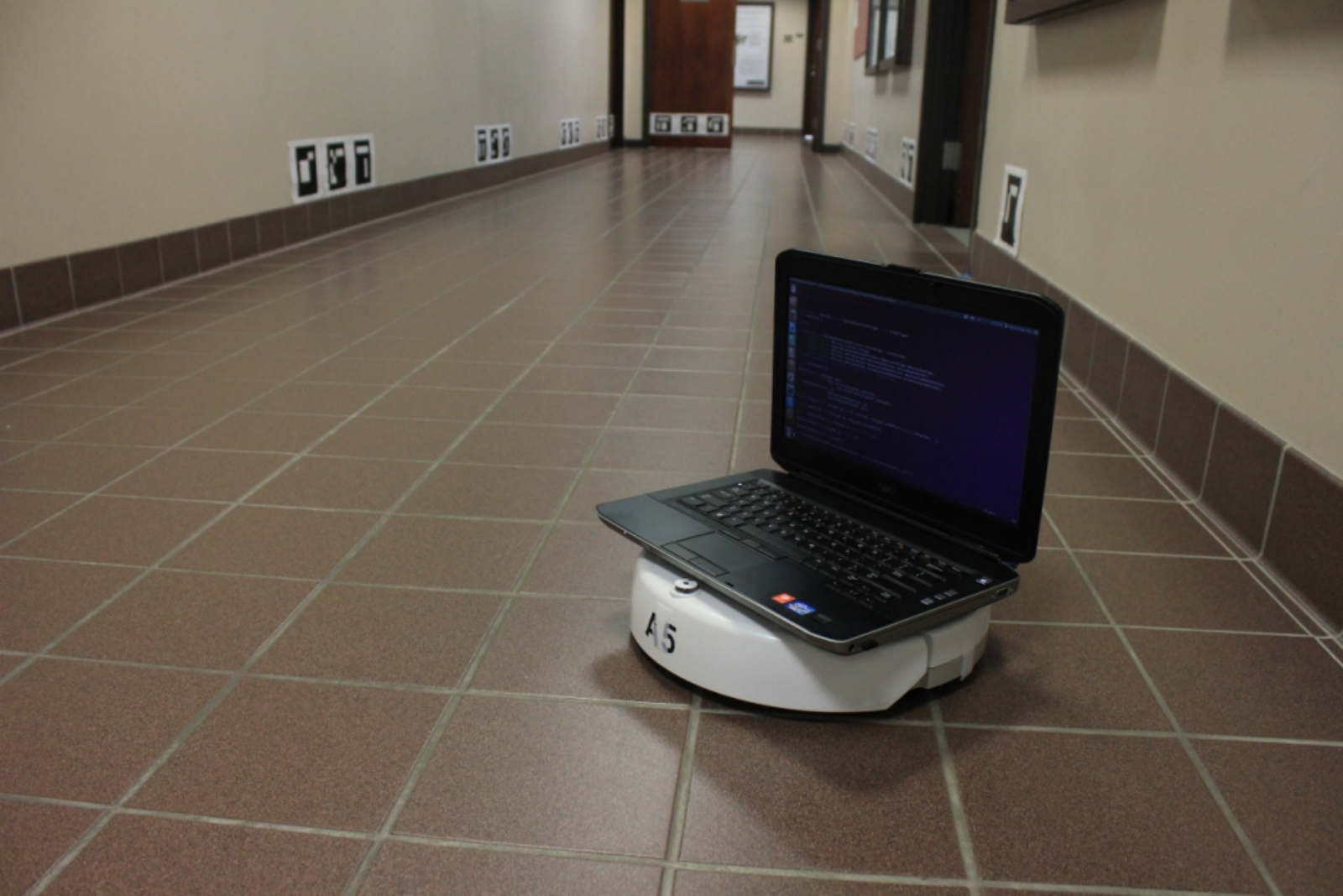}
\caption{A picture of the robot (an iRobot Create) operating in the office environment. Landmarks can be seen on the walls.}
\label{fig:robot-landmark}
\end{figure}

\subsubsection{Sensing Model} \label{subsec:sensorModel}
For sensing purposes, we use the on-board laptop webcam. We perform a vision-based landmark detection based on ArUco (a minimal library for Augmented Reality applications) \cite{Aruco}. Each landmark is a black and white pattern printed on a letter-size paper. The pattern on each landmark follows a slight modification of the Hamming code, and has a unique id, so that it can be detected robustly and uniquely. Landmarks are placed on the walls in the environment (see Fig. \ref{fig:robot-landmark}). 
The absolute position and orientation of each landmark in the environment is known. The ArUco library performs the detection process and presents the relative range and bearing to each visible landmark along with its id. Denoting the $ j $-th landmark position in the global 2D coordinate frame as $ ^{j}L $, we can model the observation as a range-bearing sensing system:
\begin{align}
\nonumber {^j}z_{k}&=[\|{^j}\mathbf{d}_{k}\|,\text{atan2}({^j}d_{2_{k}},{^j}d_{1_{k}})-\theta]^T+{}^jv,~~{^j}v\sim\mathcal {N}(\mathbf{0},{^j}\mathbf{R}),
\end{align}
where ${^j}\mathbf{d}_{k}=[{^j}d_{1_{k}},{^j}d_{2_{k}}]^T:=[\mathtt{x}_{k},\mathtt{y}_{k}]^T-L_j$. 

\ph{Measurement noise} A random vector $ {^j}v $ models the measurement noise associated with the measurement of the $ j $-th landmark. Experimentally, we observed that the intensity of the measurement noise increases by the distance from the landmark and by the incident angle. The incident angle refers to the angle between the line connecting the camera to landmark and the wall, on which the landmark is mounted. Denoting the incident angle by $ \phi\in[-\pi/2,\pi/2] $, we model the sensing noise associated with the $ j $-th landmark as a zero mean Gaussian, whose covariance is
\begin{align}
\nonumber
{^j}\mathbf{R}_{k} = \operatorname{diag}(
(\eta_{r_{d}}\|{^j}\mathbf{d}_{k}\|+\eta_{r_{\phi}}|\phi_{k}|+\sigma^r_b)^2,\\
(\eta_{\theta_{d}}\|{^j}\mathbf{d}_{k}\|+\eta_{\theta_{\phi}}|\phi_{k}|+\sigma^{\theta}_b)^2
),
\end{align}
where, in our implementations we have $ \eta_{r_{d}} = 0.1 $, $\eta_{r_{\phi}} = 0.01$, $ \sigma^r_b = 0.05\text{m} $, $ \eta_{\theta_{d}} = 0.001 $, $ \eta_{\theta_{\phi}} = 0.01 $, and $ \sigma^{\theta}_b = 2.0\text{deg} $.

\ph{Full vector of measurements} At each step, the robot observes the set of landmarks that fall into its field of view. Given that the robot can see $ r $ landmarks $ \{L_{i_1},\cdots,L_{i_r} \} $, the total measurement vector is 
$z=[{^{i_{1}}}z^T,\cdots,{^{i_{r}}}z^T]^T$. Due to the independence of measurements of different landmarks, the full observation model can be written as $z=h(x)+v$, where $ v = [{^{i_{1}}}v^T,\cdots,{^{i_{r}}}v^T]^T \sim\mathcal{N}(\mathbf{0},\mathbf{R})$ and $\mathbf{R}=\text{diag}(^{i_{1}}\mathbf{R},\cdots,{^{i_{r}}}\mathbf{R})$.

\subsection{SLAP versus Decoupled Localization and Planning} \label{subsec:closed-PRM}
In this section, we contrast the results of a traditional decoupled localization and planning with the proposed SLAP solution. Decoupled localization and planning here refers to a method, where the planner first generates a plan (ignoring the localizer) and then, in the execution phase, the localizer estimates the state to aid the controller to follow the planned trajectory. However, in the proposed SLAP solution, the planner takes the localizer into account in the planning phase and replans simultaneously as the localizer updates its estimation.

The test environment is shown in Fig. \ref{fig:prm-env}. Blue regions are obstacles and black regions are free space. Landmarks are shown by small white diamonds. The start and goal locations for the motion planning problem are marked in Fig. \ref{fig:prm-env}. The goal location is inside 407-area (see Fig. \ref{fig:floorPlan}) and the starting location is close to the front door. 

\ph{Decoupled planning and localization} To select a suitable planner, we tried a variety of traditional planners such as PRM, RRT, and their variants. We observed that following the plan generated by most of these methods leads to collisions with obstacles and cannot reach the goal point due to the high motion noise (of our low-cost robot) and due to the sparsity of the information in certain parts of the test environment. The best results are achieved using the MAPRM \axx{(Medial-Axis PRM) method \cite{wilmarth1999maprm}.} This planner is computationally more expensive than the other variants but is more powerful in dealing with collisions by sampling points on the medial axis of the free space and constructing a PRM that has the most clearance from obstacles. An MAPRM graph (in white) for this environment is shown in Fig.~\ref{fig:prm-env}.

In this environment, there are two main homotopy classes of paths between the start and goal nodes: Through the front door of 407-area and through the back door of the 407-area. From Fig.~\ref{fig:prm-env}, it is obvious that the path through the front door is shorter. Moreover, the path through the front door has a larger obstacle clearance (larger minimum distance from obstacles along the path) compared to the path through the back door (since the back door is half-open). Therefore, based on conventional metrics in deterministic settings, such as shortest path or maximum clearance, MAPRM chooses the path through the front door over the path through the back door. The feedback tree that results from solving DP in this case is shown in Fig. \ref{fig:prm-feedback}. As expected, feedback guides the robot toward the goal through the front door. To execute MAPRM's plan, we use time-varying LQG controllers to keep the robot close to the generated path. However, due to the lack of enough information along the nominal path, the success rate of this plan is low, and the robot frequently collides with obstacles. The success probability along the nominal path is computed by multiple (100) runs and is equal to 27\% (27 runs out of 100 runs were collision-free).

\ph{FIRM-based SLAP}
As can be seen in Fig. \ref{fig:prm-env}, the distribution of information is not uniform in the environment. The density of landmarks (information sources) along the path through the back door is higher than that of the path through the front door. FIRM-based SLAP can incorporate this information in the planning phase in a principled way. This leads to a better judgement of how narrow the passages are. For example, in this experiment, although the path through the front door is shorter than the path through the back door, considering the information sources, the success probability of going through the back door is much greater than going through the front door. Such knowledge about the environment is reflected in the FIRM cost-to-go and success probability. As a result, it generates a policy that suits the application, taking into account the uncertainty and available information in the environment. Solving DP on the FIRM graph generates the feedback tree shown in Fig.$~$\ref{fig:firm-feedback}, which results in 88\% success probability.
\newcommand{\comparefigWid}{2.3in}
\begin{figure*}[h!]
	\centering
	\subfigure[]
	{\includegraphics[width=\comparefigWid]
		{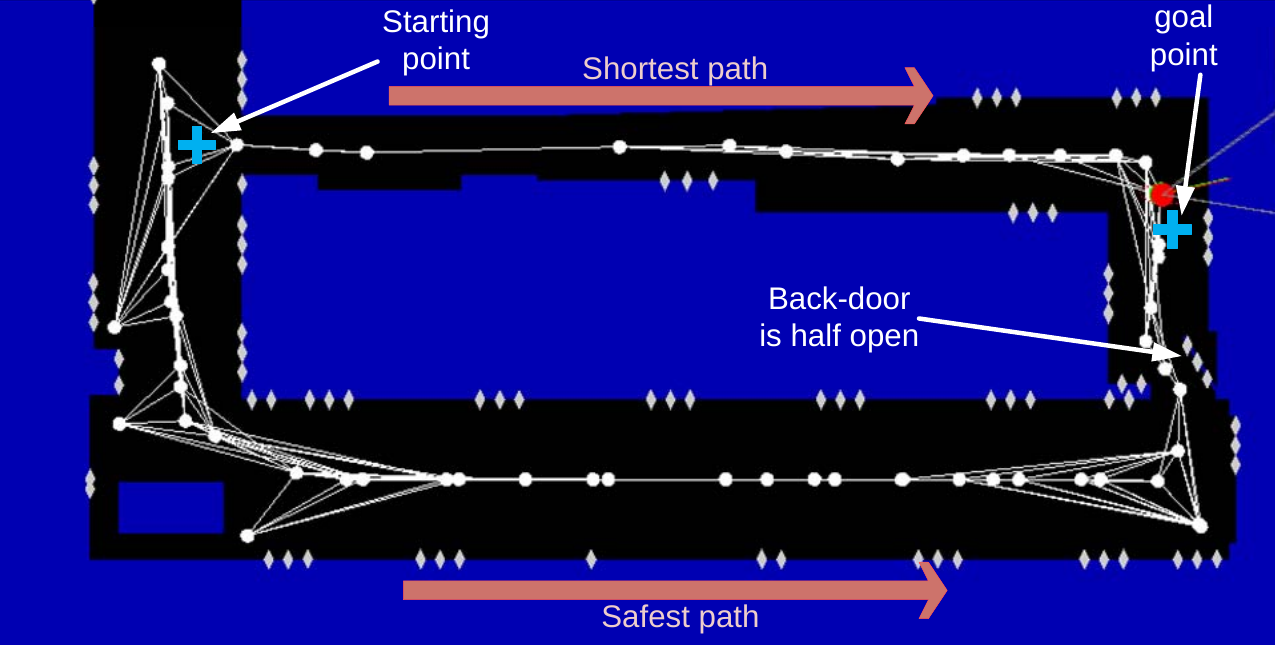}
		\label{fig:prm-env}
	}
	\subfigure[]
	{\includegraphics[width=\comparefigWid]
		{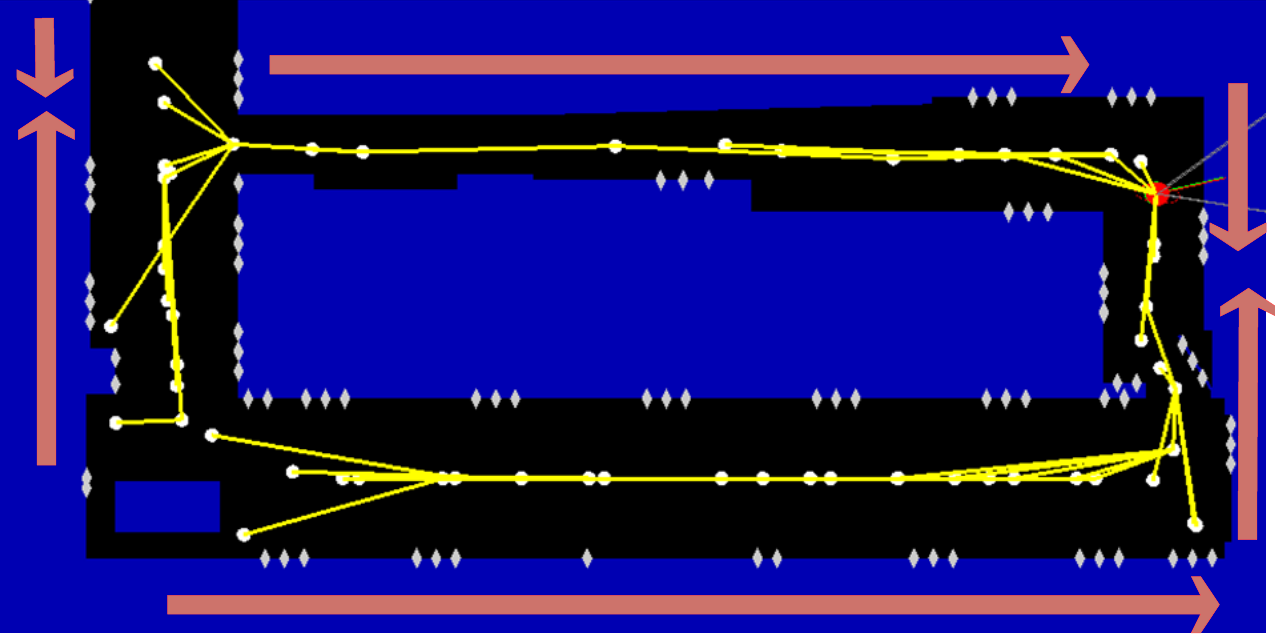}
		\label{fig:prm-feedback}
	}
	\subfigure[]
	{\includegraphics[width=\comparefigWid]
		{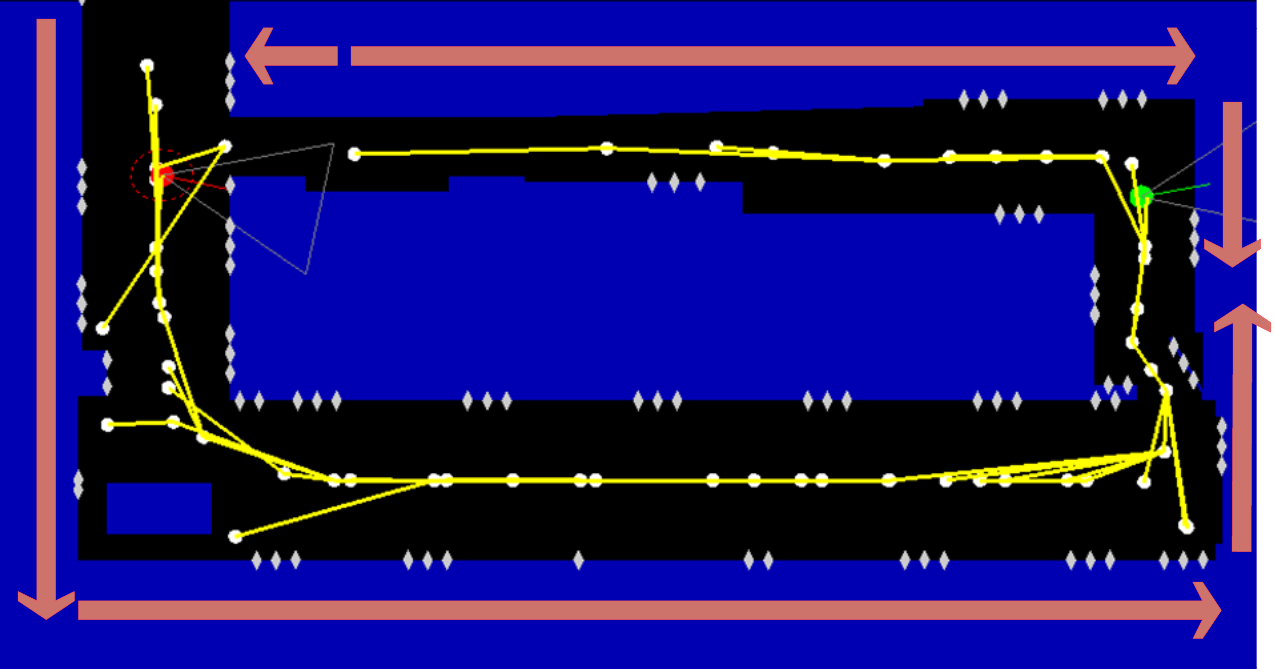}
		\label{fig:firm-feedback}
	}
	\caption{(a) Environment: obstacles (blue), free space (black), and landmarks (white diamonds). An MAPRM graph (white) approximates the connectivity of the free space. (b) The feedback tree (yellow), generated by solving DP on MAPRM. From each node there is only one outgoing edge, guiding the robot toward the goal. Arrows in pink coarsely represent the direction on which the feedback guides the robot. (c) The feedback tree (yellow), generated by solving DP on FIRM. Computed feedback tree guides the robot through the more informative regions, leading to more accurate localization and less collision probabilities. Arrows in pink coarsely represent the direction on which the feedback guides the robot.}
	\vspace{-15pt}
\end{figure*}

\subsection{Online Replanning aspect of SLAP} \label{subsec:Robustness-experiments}
In this section, we focus on the ``simultaneous" part in SLAP, which emphasizes the ability of the robot to replan after every localization update. In other words, in SLAP, the robot dynamically replans based on the new information coming from its sensors.

We study two important cases to illustrate the effect of online replanning. We first look into a challenging case where the obstacle map changes and possibly eliminates a homotopy class of solutions. This means the planner has to switch to a different homotopy class in real-time, which is a challenging situation for the state-of-the-art methods in the belief space planning literature. Second, we look into deviations from the path, where we focus on the kidnapped robot problem as the most severe case of deviation. The solution to this case can be applied to smaller deviations as well. Finally, we demonstrate the performance of the proposed method on a complex scenario that includes changes in the obstacles, deviations in the robot pose, online changes in goal location, etc.

\subsubsection{Changes in the obstacle map} \label{subsec:changingObstacles}
$ ~ $\\
Here, we show how enabling simultaneous planning and localization, online, can handle changes in the obstacle map.

In this paper, we assume no prior knowledge about the environment dynamics. As a result, we have a simple model for obstacle dynamics: All new obstacles will be added to the map with a large forgetting time of 10 minutes (i.e., almost-permanent). The only exception in this model is moving people: if a moving person is detected, a new obstacle will not be added to the map. Instead, we assume there exists a lower-level reactive behavior (e.g., stopping or dodging) in a subsumption-like architecture \cite{Brooks86} that suppresses the belief space planner in the vicinity of the moving person. Once the control is back to the SLAP layer, the robot might have deviated from its nominal plan and thus the SLAP layer has to replan to recover from such deviations.

Therefore, the method is very efficient in dealing with persistent/slow changes in the map (e.g., closed/open doors). An important aspect of the method is that it can deal with severe changes that might eliminate or create homotopy classes of solutions. Doors are an important example of this class. If the robot observes a closed door (which was expected to be open), it might have to \emph{globally} change the plan to get to the goal from a different passage. This is a very challenging problem for today's belief space planners.

As the first experiment, we consider the environment shown in Fig. \ref{fig:floorPlan}. The start and goal locations are shown in Fig. \ref{fig:openDoor}. We construct a PRM in the environment ignoring the changing obstacles (assuming all doors are open and there are no people in the environment). Then, we construct the corresponding FIRM and solve dynamic programming on it. As a result, we get the feedback tree shown in Fig. \ref{fig:openDoor} that guides the robot toward the goal through the back door of 407-area. However, the challenge is that the door may be closed when the robot reaches it, and there may be people moving in the environment. Moreover, for various reasons (such as motion blur in the image or blocked landmarks by people), the robot might misdetect landmarks temporarily during the run.\footnote{Designing perception mechanisms for obstacle detection is not a concern of this research, thus we circumvent the need for this module by sticking small markers with specific IDs on moving objects (doors or people's shoes).}
To handle such a change in the obstacle map and replan accordingly, we use the ``lazy feedback evaluation'' method outlined in Algorithm \ref{alg:lazy-feedback-eval}.

\ph{Results on physical robots}
Figure \ref{fig:closedDoor} shows a snapshot of the system during the operation when the robot detects the change signal, i.e., detects that the door is in a different state than its recorded situation in the map. As a result, the robot updates the obstacle map as can be seen in Fig. \ref{fig:closedDoor} (door is closed). Accordingly, the robot replans; Figure \ref{fig:closedDoor} shows the feedback tree resulting from the replanning. The new feedback guides the robot through the front door since it detects the back door is closed. \axx{The full video of this run provides much more details and is available in \cite{youtube-video-RHC-FIRM}}.

It is important to note that it is the particular structure of the proposed SLAP framework that makes such online replanning feasible. The graph structure of the underlying FIRM allows us to \textit{locally} change the collision probabilities in the environment without affecting the collision probability of the rest of the graph \axx{(i.e., properties of different edges on the graph are independent of each other; see Fig. \ref{fig:funnel-FIRM})}. This independence property is not present in the state-of-the-art belief planners such as BRM (Belief Roadmap Method) \cite{Prentice09} or LQG-MP \cite{Berg11-IJRR}. In those methods, collision probabilities and costs on \textit{all} edges need to be re-computed if a change in the obstacle map is detected. The general purpose planners such as ABT \cite{kurniawati-isrr13} are also not applicable to this setting due to the size of the problem and the need to recompute collision probabilities. In ABT, if the robot detects a change in the obstacle map in the vicinity of the robot, it needs to alter the uncertainty evolution in an ABT tree branch near the tree root (i.e., near the robot pose). This, in turn, will require the whole sub-tree (including collision probabilities) under the affected branch to be updated, which is not a real-time operation for long-horizon planning problems.

\begin{figure}[h!]
\centering
\subfigure[]{\includegraphics[width=3in]{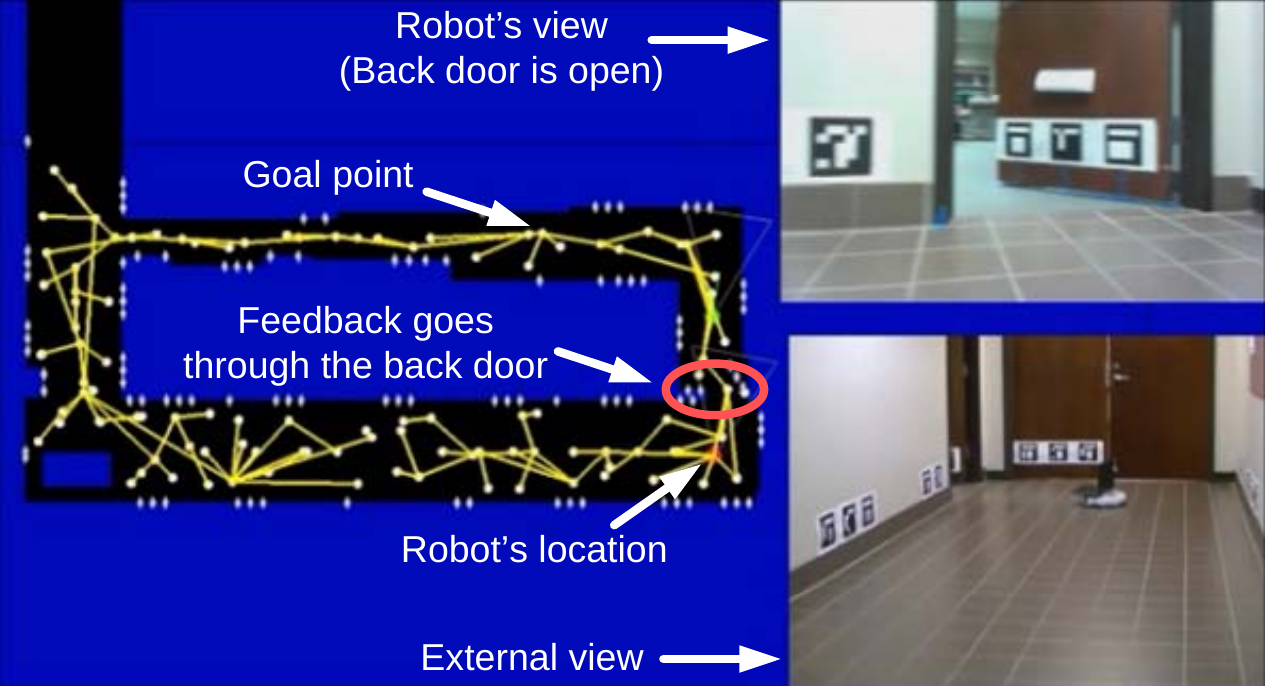}\label{fig:openDoor}}
\subfigure[]{\includegraphics[width=3in]{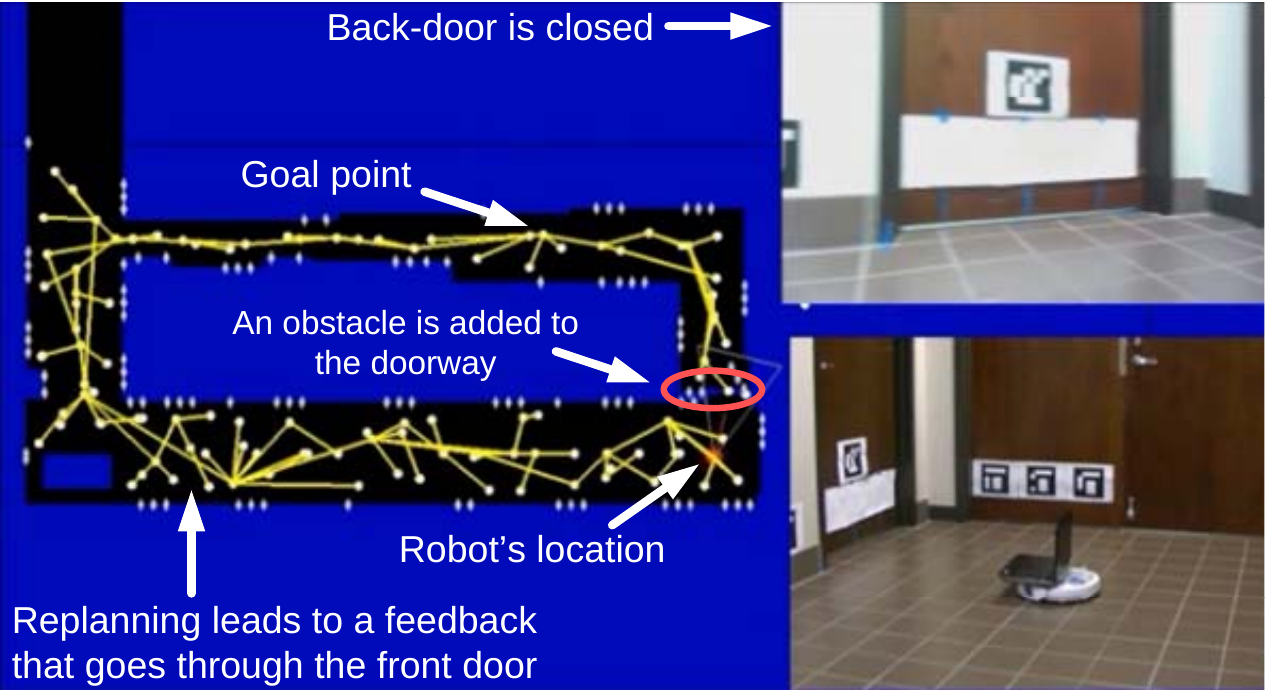}\label{fig:closedDoor}}
\caption{(a) The back door is open at this snapshot. The feedback guides the robot toward the goal through the back door. (b) The back door is closed at this snapshot. Robot detects the door is closed and updates the obstacle map (adds the door to the map). Accordingly robot replans and computes the new feedback. The new feedback guides the robot toward the goal through the front door.}
\label{fig:open-closed-door}
\end{figure}

\subsubsection{Deviations in the robot's pose} \label{subsec:KidnappedRobot}$ ~ $\\
In this subsection, we demonstrate how online replanning enables SLAP in the presence of large deviations in the robot's position. As the most severe form of this problem, we consider the \textit{kidnapped robot problem}. In the following we discuss this problem and challenges it introduces.

\ph{Kidnapped robot problem} An autonomous robot is said to be in the kidnapped situation if it is carried to an unknown location while it is in operation. The problem of recovering from this situation is referred to as the kidnapped robot problem \cite{Choset05}. This problem is commonly used to test the robot's ability to recover from catastrophic localization failures. This problem introduces different challenges such as (i) how to detect kidnapping, (ii) how to re-localize the robot, and (iii) how to control the robot to accomplish its goal. Our main focus, here, is on the third part, i.e., how to replan in belief space from the new belief resulted from the kidnapped situation. This is in particular challenging because large deviations in the robot's pose can globally change the plan and the homotopy class of the optimal solution. Therefore, the planner should be able to change the global plan online.

\ph{Detecting the kidnapped situation} To embed the kidnapped situation into the framework in a principled way, we add a boolean observation $ z^{lost} $ to the observation space. Let us denote the innovation signal as $ \widetilde{z}_{k} = z_{k}-z^{-}_{k}$ (the difference between the actual observations and predicted observation). Recall that in our implementation, the observation at time step $ k $ from the $ j $-th landmark is the relative range and bearing of the robot to the $ j $-th landmark, i.e., $ ^{j}z_{k}=({^j}r_{k},{^j}\theta_{k}) $. The predicted version of this measurement is denoted by $ ^{j}z^{-}_{k}=({^j}r^{-}_{k},{^j}\theta^{-}_{k}) $. We monitor the following measures of the innovation signal:
\begin{align}
\widetilde{r}_{k} = \max_{j}(|{^j}r_{k}-{^j}r_{k}^{-}|),~~~\widetilde{\theta}_{k} = \max_{j}(d^{\theta}({^j}\theta_{k},{^j}\theta_{k}^{-})),
\end{align}
where $ d^{\theta}(\theta,\theta') $ returns the absolute value of the smallest angle that maps $ \theta $ onto $ \theta' $. Passing these signals through a low-pass filter, we filter out the outliers (temporary failures in the sensory reading). Denoting the filtered signals by $ \overline{r}_{k} $ and $ \overline{\theta}_{k} $, if both conditions $ \overline{r}_{k}<r_{max} $ and $ \overline{\theta}_{k}<\theta_{max} $ are satisfied, then $ z^{lost}=0 $, otherwise $ z^{lost}=1 $. When $ z^{lost}=0 $, we follow the current rollout planner.
$ z^{lost}=1 $ means that the robot is constantly observing high innovations, and thus it is not in the location in which it believes to be (i.e., it is kidnapped). Once it is detected that the robot is kidnapped, we replace the estimation covariance with a large covariance (to get an approximately uniform distribution over the state space).

\ph{Replanning from the kidnapped situation}
The rollout-FIRM algorithm can inherently handle such replanning. In other words, the kidnapped situation, i.e., a deviated mean and very large covariance, will just be treated as a new initial belief and a new query. Accordingly, the FIRM rollout creates the best macro-action (i.e, graph edge or funnel) on the fly and execute it. Note that the belief deviation might change the optimal homotopy class and the plan should be updated globally, which makes it challenging for many POMDP planners. Using the proposed rollout planner, the robot just needs to go to a neighboring node from this deviated point. Since the underlying FIRM graph is spread in the belief space, the only required computation is to evaluate the cost of edges that connect the new starting point to the neighboring FIRM nodes.

To get safer plans when replanning from $ z^{lost}=1 $ situation, we update the rollout planning mechanism slightly: In addition to the new initial belief,	 we add one more belief node to the graph, as described below. 
Consider a new kidnapped initial belief $ b_{0}\equiv(\widehat{x}^{+}_{0},P_{0}) $. Let $ \delta $ denote the distance between the mean of this new belief $ \widehat{x}^{+}_{0} $ and the closest mean on the graph nodes. If $ z^{lost}=1 $ and $ \delta $ is not small, the mean belief is far from actual robot position and moving the robot $ \delta $ meters based on a wrong belief might lead to collision. To ensure that the proposed rollout-based planner can take this case into account, we add a FIRM node $ b' $ to the graph at (or very close to) the configuration point $ v = \widehat{x}^{+}_{0} $. 

In such a case starting from a deviated belief $ b_{0} $ with large covariance, the planner will take the robot to $ b' $ first, which is a belief with the same mean but smaller covariance (i.e., turning in-place or taking very small velocities). Planner will make this choice since moving to a farther node when the covariance is very large will lead to high collision probability; and this risk is reflected in the transition probabilities of the rollout edges.

\ph{Results on physical robots} Figure \ref{fig:kidnapping} shows a snapshot of a run that contains two kidnappings and illustrates the robustness of the planning algorithm to the kidnapping situation. 
The feedback tree (shown in yellow) guides the robot toward the goal through the front door. However, before reaching the goal point, the robot is kidnapped in the hallway and placed in an unknown location within 407-area (see Fig. \ref{fig:kidnapping}). In our implementations, we consider $ r_{max}=1 $ (meters) and $\theta_{max}=50 $ (degrees). The first jump in Fig. \ref{fig:innovation} shows this deviation. Once the robot recovers from being kidnapped (i.e., when both innovation signals in Fig. \ref{fig:innovation} fall below their corresponding thresholds), replanning from the new point is performed. This time, the feedback guides the robot toward the goal point from within 407-area. However, before the robot reaches the goal point, it is kidnapped again and placed in an unknown location (Fig. \ref{fig:kidnapping}). The second jump in the innovation signals in Fig. \ref{fig:innovation} corresponds to this kidnapping.
\begin{figure}[h!]
\centering
\includegraphics[width=3in]{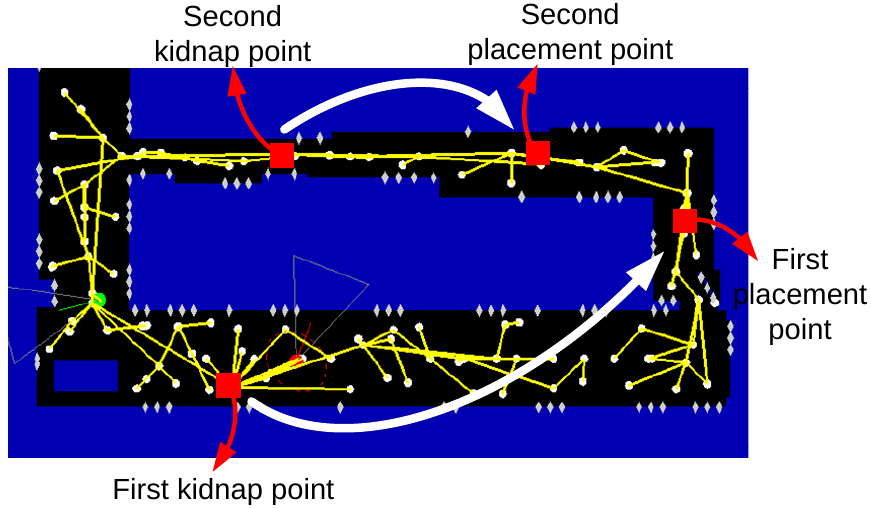}
\vspace{\gapBeforeCaption}
\caption{This figure shows the set up for the experiment containing two kidnapping.}
\label{fig:kidnapping}
\end{figure}
\begin{figure}[h!]
\centering
\includegraphics[width=3.5in]{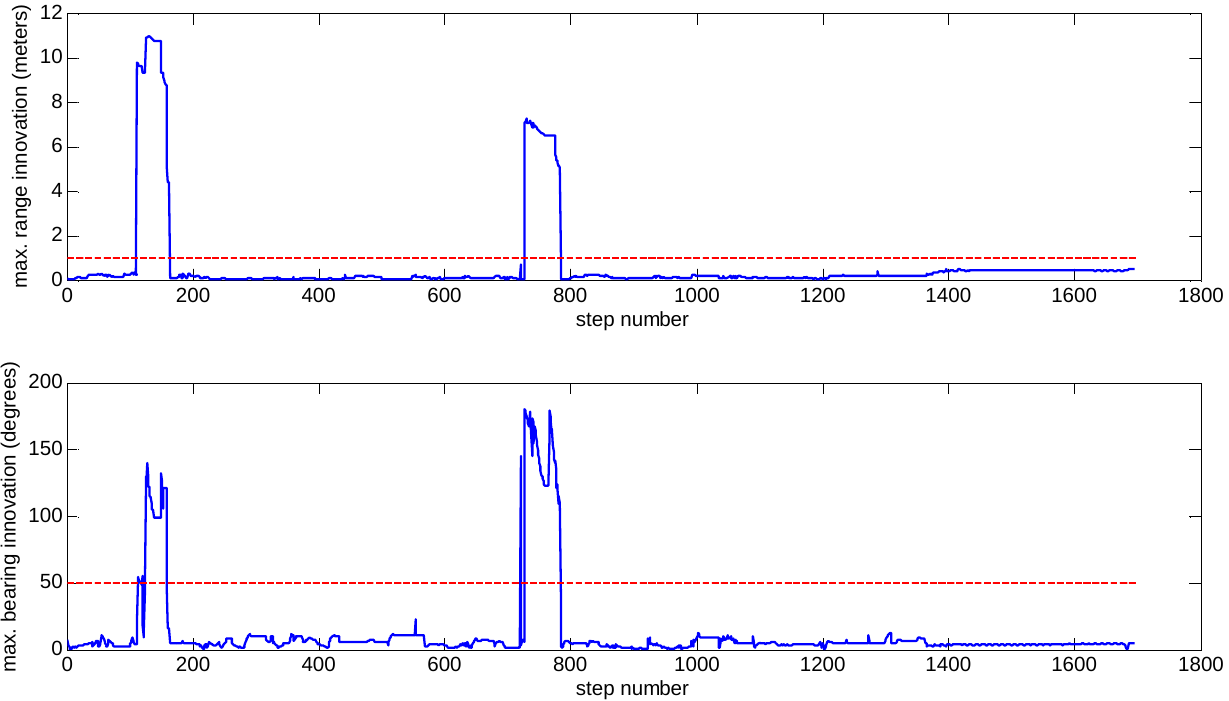}
\vspace{\gapBeforeCaption}
\caption{This figure shows the innovation signals $ \bar{r}_{k} $ and $ \bar{\theta}_{k} $, along with the thresholds $ r_{max} $ and $ \theta_{max} $ (dashed red lines). Large jumps correspond to the kidnapping events.}
\label{fig:innovation}
\end{figure}

\vspace{\gapBeforeSection}
\subsection{Longer and more complex experiments: Robustness to changing goals, obstacles, and to large deviations} \label{subsec:longer experiment}
In this section, we emphasize the ability of the system to perform long-term SLAP that consists of visiting several goals. The user(s) submit a new goal for the robot every time it reaches its current goal. While the robot needs to change the plan each time a new goal is submitted, it frequently encounters changes in the obstacle map (open/closed doors and moving people) as well as intermittent sensor failures and kidnapping situations. Thus, the ability to simultaneously replan online while localizing is necessary to cope with these changes. The video in \cite{youtube-video-RHC-FIRM} shows the robot's performance in this long and complex scenario.
\vspace{\gapBeforeFigure}
\begin{figure}[h!]
\centering
\includegraphics[width=3.4in]{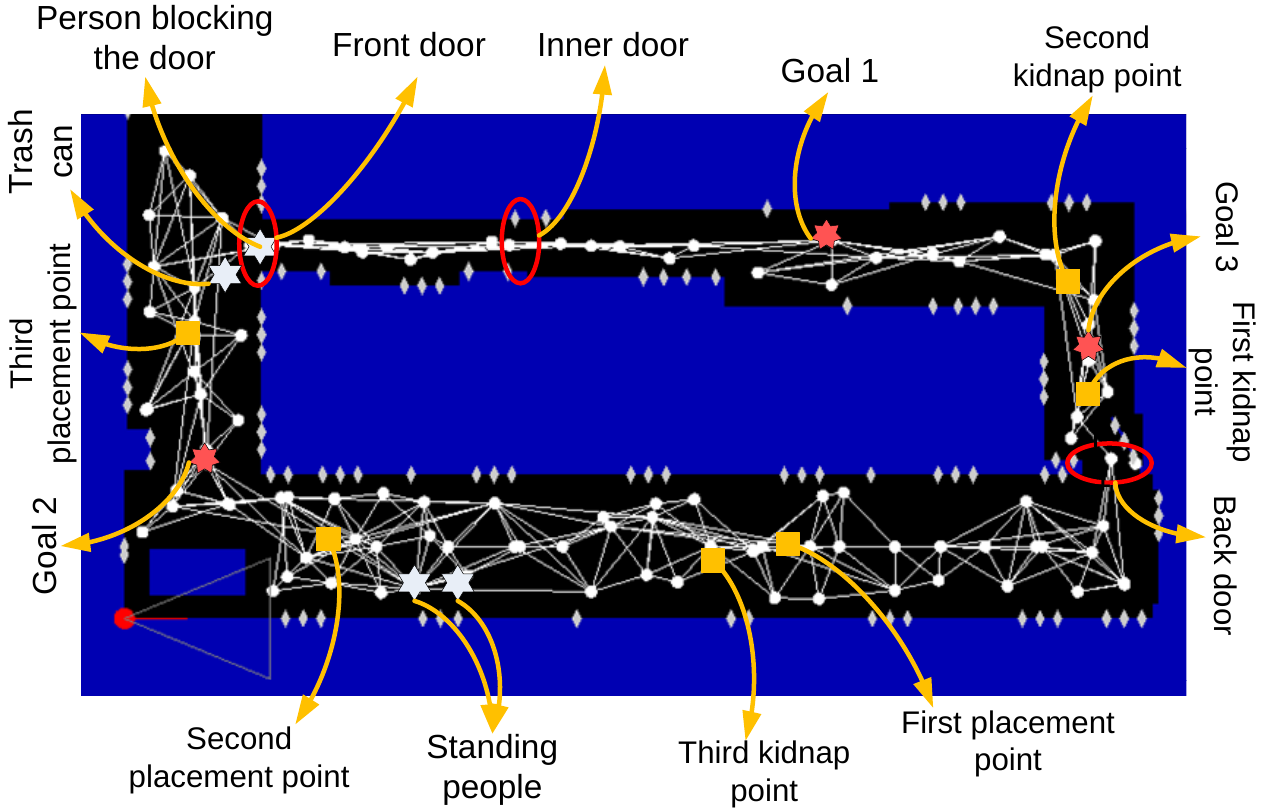}
\caption{The scenario for the long-term autonomous operation with a sequence of goals as well as various intermediate events and changes in the environment map.}
\label{fig:complex-map}
\end{figure}

In the following, we provide an itemized description of the specific steps involved in this run based on Fig. \ref{fig:complex-map} and accompanying video \cite{youtube-video-RHC-FIRM}: \textbf{1)} The robot begins at the starting point in the bottom corridor in Fig. \ref{fig:complex-map} and aims to reach Goal 1 shown in Fig. \ref{fig:complex-map}. 
FIRM returns a feedback tree that guides the robot through the back door of 407-area.
\textbf{2)} The robot passes through the narrow passage created by the half-open back door. However, before reaching the goal, it is kidnapped and is placed in an unknown location (shown in Fig. \ref{fig:complex-map} by the \textit{first placement point}.)
\textbf{3)} Observing new landmarks, the robot detects that it has been kidnapped. Accordingly, it adds a new node to the graph and replans online. The generated feedback guides the robot toward Goal 1 through the back door again.
\textbf{4)} However, in the meantime the back door has been closed. When the robot reaches the vicinity of the back door, it detects the closed door. Therefore, it updates its map by adding an obstacle at the doorway. Note that the robot will open the door (remove the obstacle) in its map after the forgetting time of 10 minutes. Accordingly, the robot replans a feedback tree that guides it through the front door to Goal 1.
\textbf{5)} Along the way, moving people are reactively avoided and the standing people and static obstacles such as a trash-can (see Fig. \ref{fig:complex-map}) are temporarily added to the map as obstacles. Replanning several times to cope with such changes, the robot goes through the front and inner doors and reaches Goal 1.
\textbf{6)} After reaching Goal 1, Goal 2 is submitted. Replanning leads to a new feedback tree that guides the robot through the inner door, and front door, to reach Goal 2.
\textbf{7)} However, as the robot reaches the vicinity of the inner door, it detects the door has been closed. Therefore, it updates its map and replans. The new feedback guides the robot to Goal 2 through the back door.
\textbf{8)} Before reaching the goal, the robot gets kidnapped at the ``second kidnap point''. The robot is placed at a really far-off point (the ``second placement point'' in Fig. \ref{fig:complex-map}). Once the robot detects it is kidnapped, it replans and moves very slowly to gather information. Detecting landmarks, it reduces its uncertainty and continues going toward the goal point.
\textbf{9)} After reaching Goal 2, the next goal is submitted and replanning leads to a feedback through the front door.
However, detecting a person standing in the doorway of the front door leads to a new plan going through the back door.
\textbf{10)} On the way to the back door, it is again displaced at the ``third kidnap point'' and placed at the ``third placement point''.
\textbf{11)} This time, due to the forgetting time, the replanning leads to a path through the front door (the person is not there any more).
\textbf{12)} Again, the robot follows the feedback and achieves Goal 3.

This long and complicated scenario demonstrates how simultaneous planning and localization can lead to robust behaviors in the presence of intermittent model discrepancies, changes in the environment, and large deviations in the robot's location. It is worth noting that online replanning in belief space is a challenge for the state-of-the-art methods in belief space as they mainly rely on structures that depend on the system's initial belief. Hence, when the system's localization pdf encounters a significant deviation, replanning from the new localization pdf requires the structure to be re-built, which is not a feasible operation online. However, constructing a query-independent graph (the underlying FIRM) allows us to embed it in a replanning scheme such as the proposed rollout policy technique and perform online replanning to enable SLAP.

\section{Method limitations and Future work} \label{sec:futureWork}
In this section, we recap the method assumptions and limitations mentioned in previous sections.

\ph{Restricted class of POMDPs}
As discussed in the previous sections, it is worth noting that the proposed method is not a general-purpose POMDP solver. It provides a solution for a class of POMDP problems (including SLAP) where one can design closed-loop controllers with a funneling behavior in belief space. In the proposed instantiation of FIRM in this paper, designing funnels requires knowledge about the closed-form dynamics and sensor model. Also, the system needs to be locally linearizable at belief nodes, and the noise is assumed to be Gaussian. Further, designing a funnel/controller in belief space requires the uncertainty to be over the part of the state space that is controllable (e.g., the ego-vehicle). For example, the proposed SLAP solution is not applicable to two-player games, where there is no direct control on the opponents motion or sensing. 

\ph{Combining FIRM with general-purpose online solvers}
Most of the general-purpose tree-based POMDP solvers can be combined with FIRM, where an online tree-based planner creates and searches the tree and use FIRM as the approximate policy (and cost-to-go) beyond the tree horizon. In particular, when the problem in hand does not satisfy the above-mentioned assumptions, one can approximate the original problem with a problem that does satisfy the above assumptions, create a FIRM graph, and use it as the base policy. Leveraging this base policy, one can use general-purpose online POMDP solvers in the vicinity of the current belief, such as Despot \cite{somani2013despot}, ABT \cite{kurniawati-isrr13}, POMCP \cite{silver2010monte}, AEMS \cite{Ross_2007_AEMS} that act on the original exact problem.

\ph{Dealing with dynamic environments}
In this paper, we assume no prior knowledge about the environment dynamics. As a result, the simple model we use for new obstacles is: they either enter the map with a large forgetting time of 10 minutes (e.g., doors) or avoided reactively (e.g., moving people). A more sophisticated and efficient solution can be obtained by learning and modeling changes over time \cite{Marthi_RSS12_PR2} or using some prior on the motion of moving objects. Incorporating such knowledge in the proposed planning framework is a subject of future work. 

\section{Conclusions} \label{sec:conclusions}
In this paper, we proposed a rollout-policy based algorithm for online replanning in belief space to enable SLAP. The proposed algorithm is able to switch between different homotopy classes of trajectories in real-time. It also bypasses the belief stabilization process of the state-of-the-art FIRM framework. A detailed analysis was presented, which shows that the method can recover the performance and success probability that was traded-off in the stabilization phase of FIRM. Further, by re-using the costs and transition probabilities computed in the offline construction phase, the method is able to enable SLAP, via online replanning, in the presence of changes in the environment and large deviations in the robot's pose. Via extensive simulations we demonstrated performance gains when using the rollout-based belief planner. As a key focus of the work, we also demonstrated the results of the proposed belief space planner on a physical robot in a real-world indoor scenario. Our experiments show that the proposed method is able to perform SLAP and guide the robot to its target locations while dynamically replanning in real-time and reacting to changes in the obstacles and deviations in the robot's state. Such replanning is an important ability for physical systems where stochasticity in the system's dynamics and measurements can often result in failure. Hence, we believe that the proposed SLAP solution takes an important step towards bringing belief space planners to physical applications and advances long-term safe autonomy for mobile robots.

\bibliographystyle{IEEEtran}
\bibliography{AliAgha}

\end{document}